\documentclass[nohyperref]{article}

\usepackage{microtype}
\usepackage{graphicx}
\usepackage{dsfont}
\usepackage{booktabs} 

\usepackage{hyperref}

\usepackage[accepted]{icml2022}

\usepackage{amsmath}
\usepackage{bbm}
\usepackage{amssymb}
\usepackage{mathtools}
\usepackage{amsthm}

\newcommand{\R}{\mathbb R}

\newcommand{\Var}{\mathrm{Var}}
\renewcommand{\l}{\left}
\renewcommand{\r}{\right}
\newcommand{\E}{\mathbb{E}}
\newcommand{\dx}{\textrm{d}x}
\newcommand{\dr}{\textrm{d}r}
\newcommand{\dtheta}{\textrm{d}\theta_0}

\newcommand{\pmix}{p_{\textrm{mix}}}
\newcommand{\pgold}{p_{\textrm{gold}}}
\newcommand{\x}{\vec{x}}
\newcommand{\vr}{\vec{r}}
\newcommand{\vtheta}{\vec{\theta}}
\newcommand{\vpsi}{\vec{\psi}}
\renewcommand{\v}{\vec}
\newcommand{\TV}{\mathrm{TV}}
\newcommand{\vol}[1]{\mathrm{vol}(B^{#1})}
\newcommand{\relu}{\mathrm{ReLU}}
\newcommand{\Ismallr}{I_{r_{\mathrm{small}}}}
\newcommand{\Ilarger}{I_{r_{\mathrm{large}}}}

\newcommand{\KL}{\mathrm{KL}}
\renewcommand{\l}{\left}
\renewcommand{\r}{\right}
\newcommand{\y}{\vec{y}}
\renewcommand{\v}{\vec}
\newcommand{\Lm}{\hat{L}}

\newcommand{\eps}{\epsilon}

\newcommand{\indwx}{1_{w_i \cdot x > 0}}
\newcommand{\Qxy}{Q_{x,y}}
\newcommand{\Mxyhat}{M_{\hat{x} \leftrightarrow \hat{y}}}
\newcommand{\xhat}{\hat{x}}
\newcommand{\yhat}{\hat{y}}
\newcommand{\indwy}{1_{w_i \cdot y > 0}}
\DeclareMathOperator{\Span}{span}

\definecolor{customgreen}{HTML}{82b08e}

\usepackage{adjustbox}
\usepackage{subcaption}

\usepackage{pifont}
\newcommand{\xmark}{\ding{55}}%

\usepackage[capitalize,noabbrev]{cleveref}

\usepackage{glossaries}
\makeglossaries

\theoremstyle{plain}
\newtheorem{theorem}{Theorem}[section]
\newtheorem{proposition}[theorem]{Proposition}
\newtheorem{lemma}[theorem]{Lemma}

\theoremstyle{definition}
\newtheorem{definition}[theorem]{Definition}

\theoremstyle{remark}

\usepackage{mathtools}
\DeclarePairedDelimiterX{\infdivx}[2]{(}{)}{%
  #1\;\delimsize\|\;#2%
}
\newcommand{\yuval}[1]{\textbf{\textcolor{blue}{[yuval: #1]}}}

\usepackage[textsize=tiny]{todonotes}

\icmltitlerunning{Score Guided Intermediate Layer Optimization}

\begin{document}

\twocolumn[
\icmltitle{Score-Guided Intermediate Layer Optimization: \\ Fast Langevin Mixing for Inverse Problems}

\icmlsetsymbol{equal}{*}

\begin{icmlauthorlist}
\icmlauthor{Giannis Daras}{equal,austincs}
\icmlauthor{Yuval Dagan}{equal,mit}
\icmlauthor{Alexandros G. Dimakis}{austinece}
\icmlauthor{Constantinos Daskalakis}{mit}
\end{icmlauthorlist}

\icmlaffiliation{austincs}{Department of Computer Science \\ University of Texas at Austin}
\icmlaffiliation{austinece}{Department of Electrical and Computer Engineering \\ University of Texas at Austin}
\icmlaffiliation{mit}{MIT CSAIL}

\icmlcorrespondingauthor{Giannis Daras}{giannisdaras@utexas.edu}
\icmlcorrespondingauthor{Yuval Dagan}{dagan@mit.edu}
\icmlcorrespondingauthor{Alexandros G. Dimakis}{dimakis@austin.utexas.edu}
\icmlcorrespondingauthor{Constantinos Daskalakis}{costis@csail.mit.edu}

\icmlkeywords{Machine Learning, ICML}

\vskip 0.3in
]



\printAffiliationsAndNotice{\icmlEqualContribution} 

\begin{abstract}
We prove fast mixing and characterize the stationary distribution of the Langevin Algorithm for inverting random weighted DNN generators. This result extends the work of Hand and Voroninski from efficient inversion to efficient posterior sampling. In practice, to allow for increased expressivity, we propose to do posterior sampling in the latent space of a pre-trained generative model. To achieve that, we train a score-based model in the latent space of a StyleGAN-2 and we use it to solve inverse problems.
Our framework, {\em Score-Guided Intermediate Layer Optimization (SGILO)}, extends prior work by replacing the sparsity regularization with a generative prior in the intermediate layer. Experimentally, we obtain significant improvements over the previous state-of-the-art, especially in the low measurement regime. 
\end{abstract}

\section{Introduction}
\label{sec:introduction}
We are interested in solving inverse problems with generative priors, a family of unsupervised imaging algorithms initiated by Compressed Sensing with Generative Models (CSGM)~\cite{bora2017compressed}. This framework has been successfully applied to numerous inverse problems including non-linear phase retrieval~\citep{hand2018phase}, improved MR imaging~\citep{kelkar2021prior, darestani2021measuring} and 3-D geometry reconstruction from a single image~\citep{pigan, lin20223d, daras2021solving}, etc.
CSGM methods can leverage any generative model including GANs and VAEs as originally proposed~\citep{bora2017compressed}, but also invertible flows~\citep{asim2019invertible} or even untrained generators~\citep{heckel2018deep}.  

 One limitation of GAN priors when used for solving inverse problems is that the low-dimensionality of their latent space impedes the reconstruction of signals that lie outside their generation manifold. To mitigate this issue, sparse deviations were initially proposed in the pixel space~\cite{dhar2018modeling} and subsequently generalized to intermediate layers with Intermediate Layer Optimization (ILO)~\cite{daras2021intermediate}. ILO extends the set of signals that can be reconstructed by allowing sparse deviations from the range of an intermediate layer of the generator. 
 Regularizing intermediate layers is crucial when solving inverse problems to avoid overfitting to the measurements. In this work, we show that the sparsity prior is insufficient to prevent artifacts in challenging settings (e.g. inpainting with very few measurements, see Figure \ref{inpainting}).

 Recently, two new classes of probabilistic generative models, ~\textit{Score-Based} networks~\cite{score_first} and \textit{Denoising Diffussion Probabilistic Models (DDPM)}~\cite{ho2020denoising} have also been successfully used to solve inverse problems~\cite{nichol2021glide, jalal2021robust, song2021solving, meng2021sdedit, whang2021deblurring}.
 Score-Based networks and DDPMs both gradually corrupt training data with noise and then learn to reverse that process, i.e. they learn to create data from noise. A unified framework has been proposed in the recent \citet{song_sde} paper and the broader family of such models is widely known as \textit{Diffusion Models}.  Diffusion models have shown excellent performance for conditional and unconditional image generation~\citep{ho2020denoising, ddpm_beats_gans, song_sde, karras2022elucidating, dalle2, imagen}, many times outpeforming GANs in image synthesis~\citep{stylegan, stylegan2, biggan, ylg}. 

Unlike MAP methods, such as CSGM and ILO, solving inverse problems with Score-Based networks and DDPMs corresponds (assuming mixing) to sampling from the posterior. Recent work showed that posterior sampling has several advantages including diversity, optimal measurement scaling~\cite{jalal_cond_resampling,nguyen2021provable} and reducing bias~\cite{jalal2021fairness}. The main weakness of this approach is that, in principle, mixing to the posterior distribution can take exponentially many steps in the dimension $n$. In practice, Score-Based models usually require thousands of steps for a single reconstruction~\cite{jolicoeurmartineau2021gotta, xiao2021tackling, watson2021learning}.

\textit{We show that (under the random weights assumption), CSGM with Stochastic Gradient Langevin Dynamics has polynomial (in the dimension) mixing to the stationary distribution}. This result extends the seminal work of ~\citet{paul_hand_v1, paul_hand_v2} from MAP to posterior sampling. Specifically, \citet{paul_hand_v1, paul_hand_v2} established polynomial-time point convergence of Gradient Descent (with sign flips) for CSGM optimization for random weight ReLU Generators. We prove that, even without the sign flips, Langevin Dynamics will mix fast. Our result is important since prior work assumed mixing of the Markov Chain sampler to establish theoretical guarantees (e.g. see \citet{jalal_cond_resampling}). 

Finally, we show how to solve inverse problems with posterior sampling in the latent space of a pretrained generator. Effectively, we combine ILO and Score-Based models into a single framework for inverse problems. We call our new method Score-Guided Intermediate Layer Optimization (SGILO). 
\textit{The central idea is to create generative models that come endowed with a score-based model as a prior for one internal intermediate layer in their architecture}. This replaces the sparsity prior used by ILO with a learned intermediate layer regularizer.

We start with a StyleGAN2~\cite{stylegan, stylegan2} and train a score-based model to learn the distribution of the outputs of an intermediate layer. To solve an inverse problem, we optimize over an intermediate layer as in ILO~\cite{daras2021intermediate}, but instead of constraining the solutions to sparse deviations near the range, we use the learned score as a regularization. Specifically, we are using Stochastic Gradient Langevin Dynamics (SGLD) to sample from the posterior distribution of the latents where the gradient of the log-density is provided by our score-based model. 



\textbf{Our Contributions:}
\begin{enumerate}
    \item We propose a novel framework, Score-Guided Intermediate Layer Optimization (SGILO), for solving general inverse problems. Our method replaces the sparsity prior of ILO~\cite{daras2021intermediate} with a learned score-based prior. 
    \item To learn this prior we train a score-based model on an intermediate latent space of StyleGAN using inversions of real images from FFHQ~\cite{stylegan} obtained with ILO~\cite{daras2021intermediate}. Our score-based models use a Vision Transformer (ViT)~\cite{ViT} variant as the backbone architecture, demonstrating design flexibility when training score models for intermediate representations.
    \item Given some measurements (e.g. inpainted image), we use the learned prior and the Langevin algorithm to do posterior sampling. Experimentally we show that our approach yields significant improvements over ILO~\cite{daras2021intermediate} and other prior work. Further, we show that our Langevin algorithm is much faster to train and to sample from, compared to standard score-based generators, since we work in the much lower dimension of the intermediate layer.  
    \item Theoretically we prove that the Langevin algorithm converges to stationarity in polynomial time. Our result extends prior work~\cite{paul_hand_v1, paul_hand_v2} which analyzed MAP optimization to Langevin dynamics. Like prior work, our theory requires that the generator has random independent weights and an expansive architecture. 
    \item We open-source all our \href{https://github.com/giannisdaras/sgilo}{code} and pre-trained models to facilitate further research on this area.
\end{enumerate}

    \begin{table*}[!ht]
        \centering
        \scalebox{1.0}{{\begin{minipage}{\textwidth}
       \begin{tabular}{c|c|c|c|c}
            Algorithm & Expressive & Sampling & Fast & Provable Convergence \\
            \hline
            Gradient Descent in $\R^k$ (CSGM~\cite{bora2017compressed}) &   \xmark & \xmark & \checkmark & \checkmark\\
            Projected Gradient Descent in $\R^p$ (ILO~\cite{daras2021intermediate}) & \checkmark & \xmark & \checkmark & \checkmark\\
            Langevin Dynamics in $\R^n$ \cite{jalal2021instanceoptimal} & \checkmark & \checkmark & \xmark & \xmark \\
            Langevin Dynamics in $\R^p$ (\textbf{SGILO}) & \checkmark & \checkmark & \checkmark & \checkmark (under assumptions) \\
        \end{tabular}
        \caption{Summary of different reconstruction algorithms for solving inverse problems with deep generative priors. For the GAN based methods (Rows 1, 2), we think of a generator as a composition over two transformations $G_1:\R^k\to\R^p$ and $G_2:\R^p \to \R^n$, where $k < p < n$. Gradient Descent in the intermediate space, as in the ILO paper, can be expressive (increased expressivity due to ILO) and fast (GAN-based methods) but does not offer diverse sampling. On the other hand, Stochastic Gradient Langevin Dynamics in the pixel space is slow as it is usually done with high-dimensional score-based models. SGILO (Row 4) combines the best of the two worlds.}
        \label{tab:reconstruction_algorithms_summary}
        \end{minipage}}}
    \end{table*}

\begin{figure*}[!ht]
\captionsetup{justification=centering}
\begin{center}
\begin{adjustbox}{width=0.9\textwidth, center}
\begin{tabular}{ccccc} 
\begin{minipage}{0.2\linewidth}
\includegraphics[width=\linewidth]{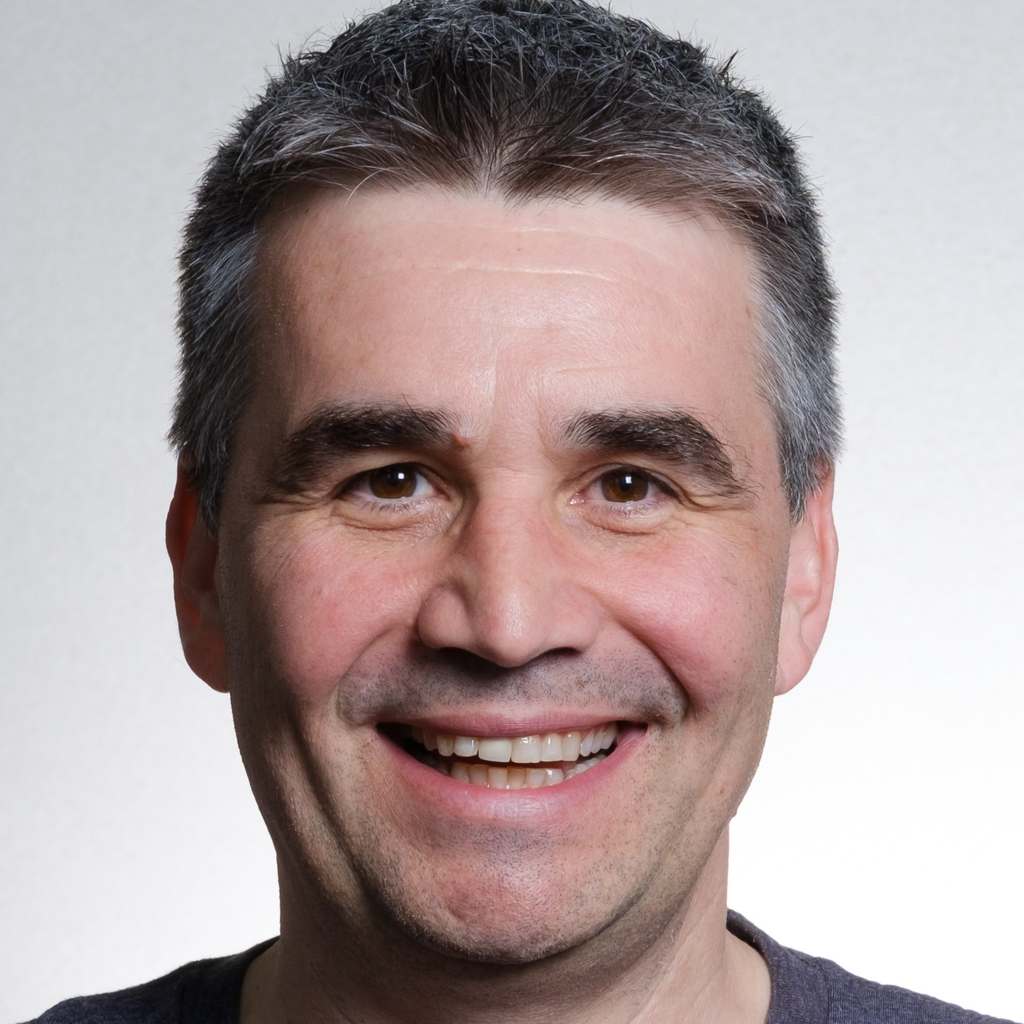} 
\end{minipage} &
\begin{minipage}{0.2\textwidth}
\includegraphics[width=\linewidth]{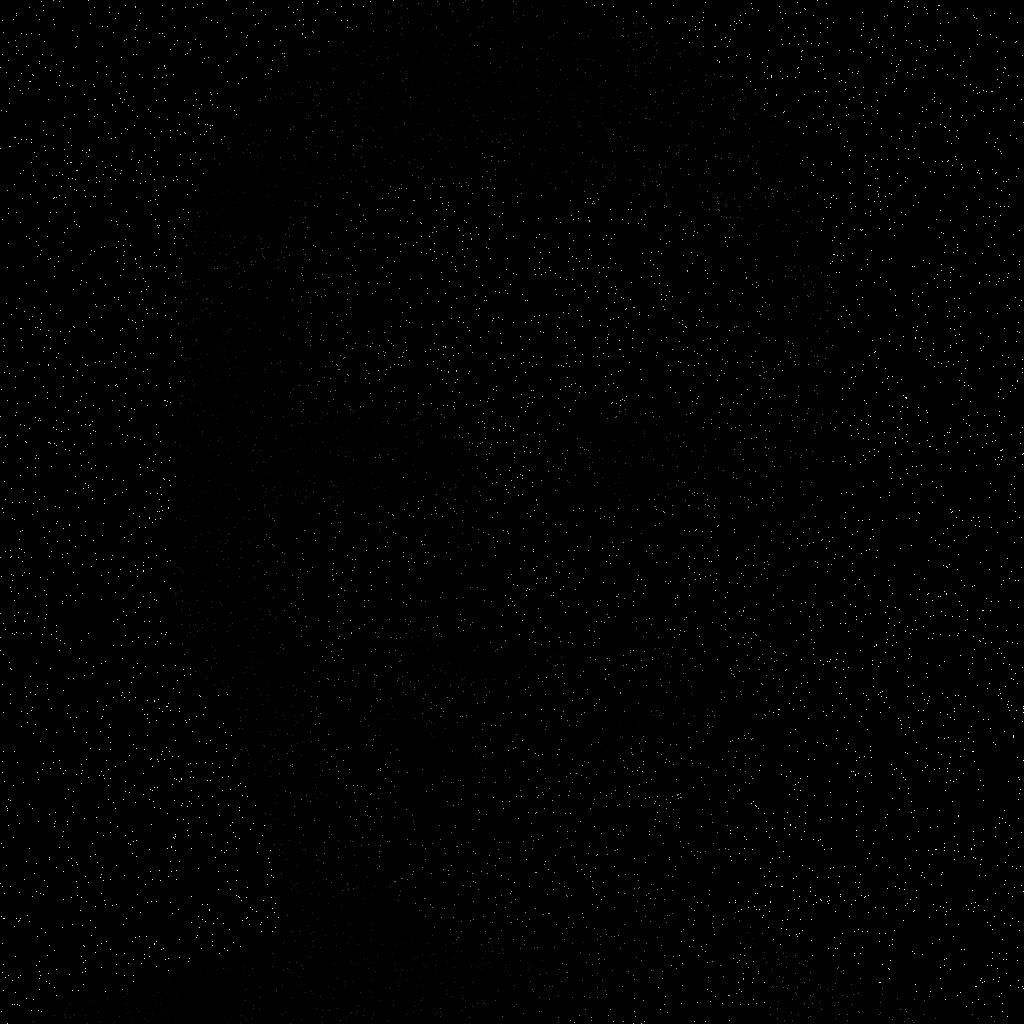} 
\end{minipage} &
\begin{minipage}{0.2\textwidth}
\includegraphics[width=\linewidth]{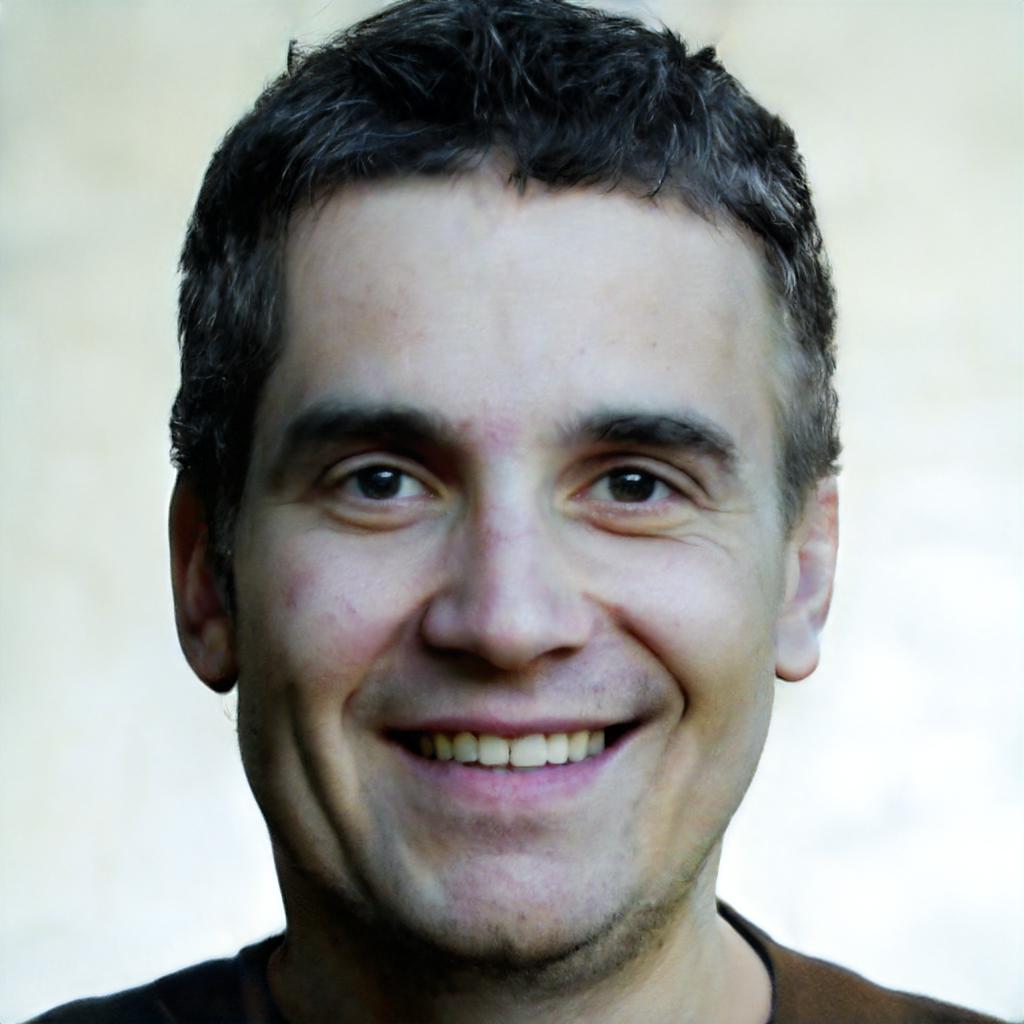} 
\end{minipage} & 
\begin{minipage}{0.2\textwidth}
\includegraphics[width=\linewidth]{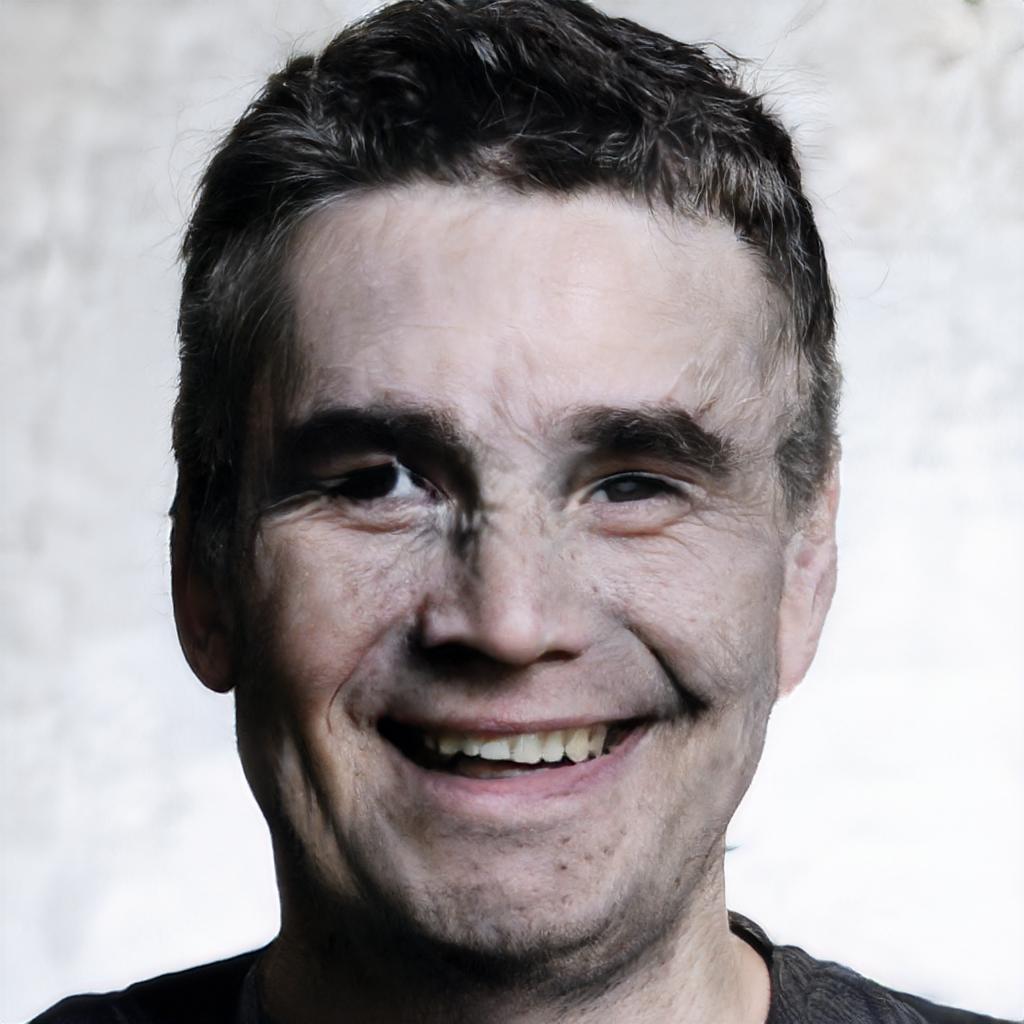} 
\end{minipage} &
\begin{minipage}{0.2\textwidth}
\includegraphics[width=\linewidth]{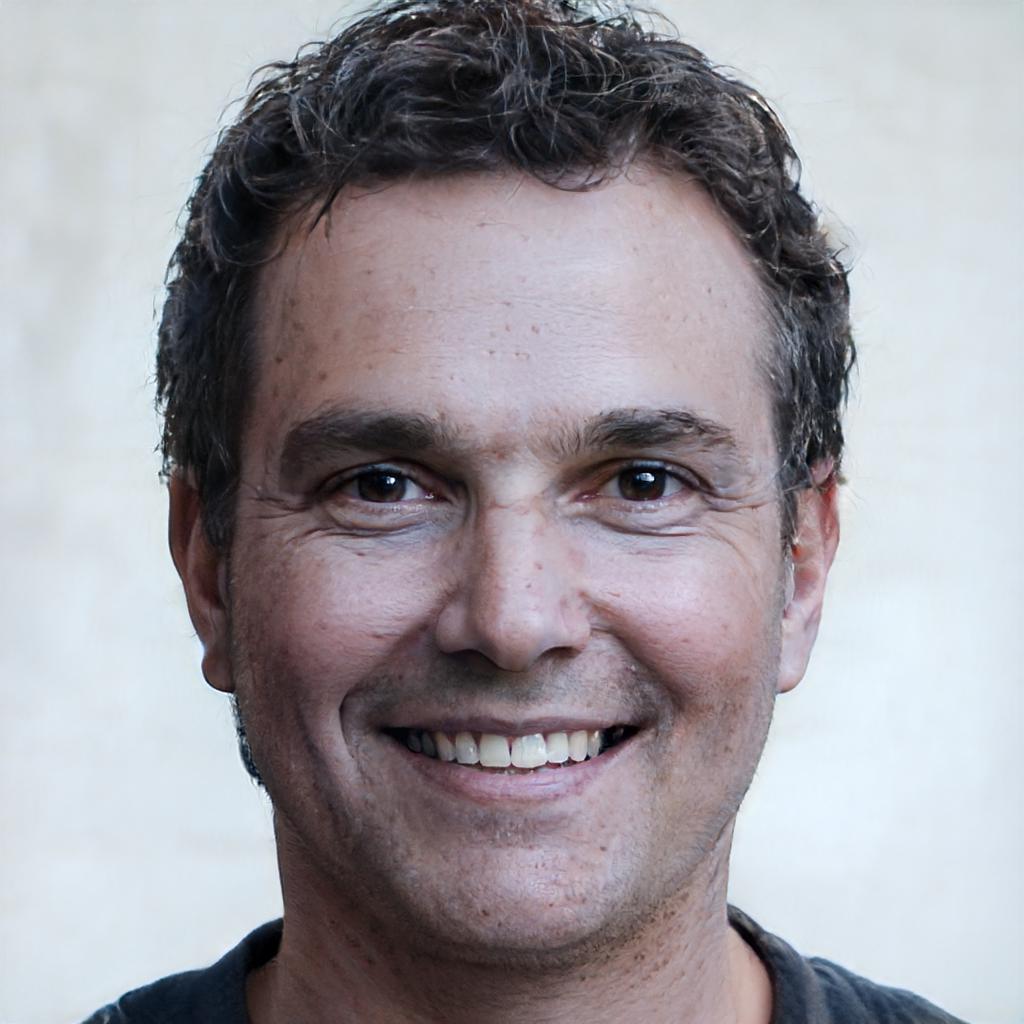} 
\end{minipage} \\ 
\begin{minipage}{0.2\textwidth}
\includegraphics[width=\linewidth]{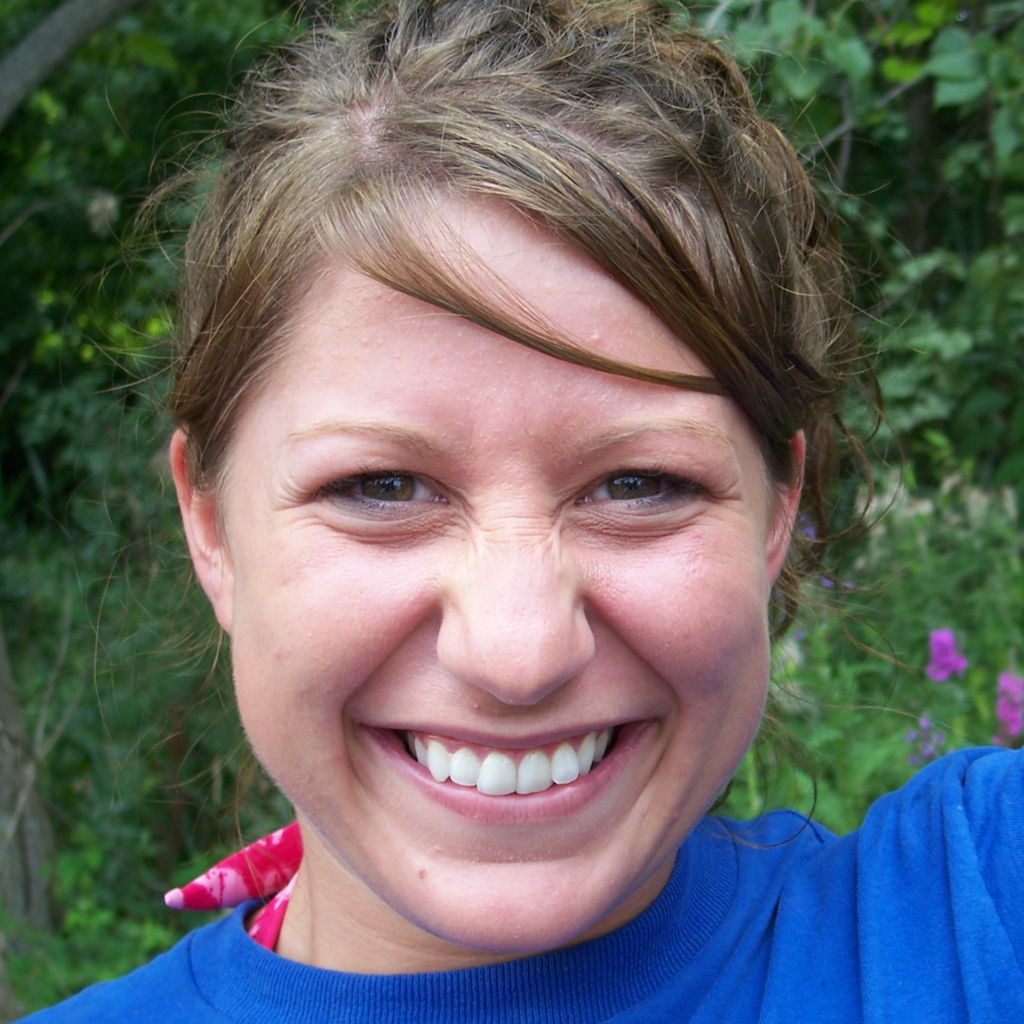} 
\end{minipage} &
\begin{minipage}{0.2\textwidth}
\includegraphics[width=\linewidth]{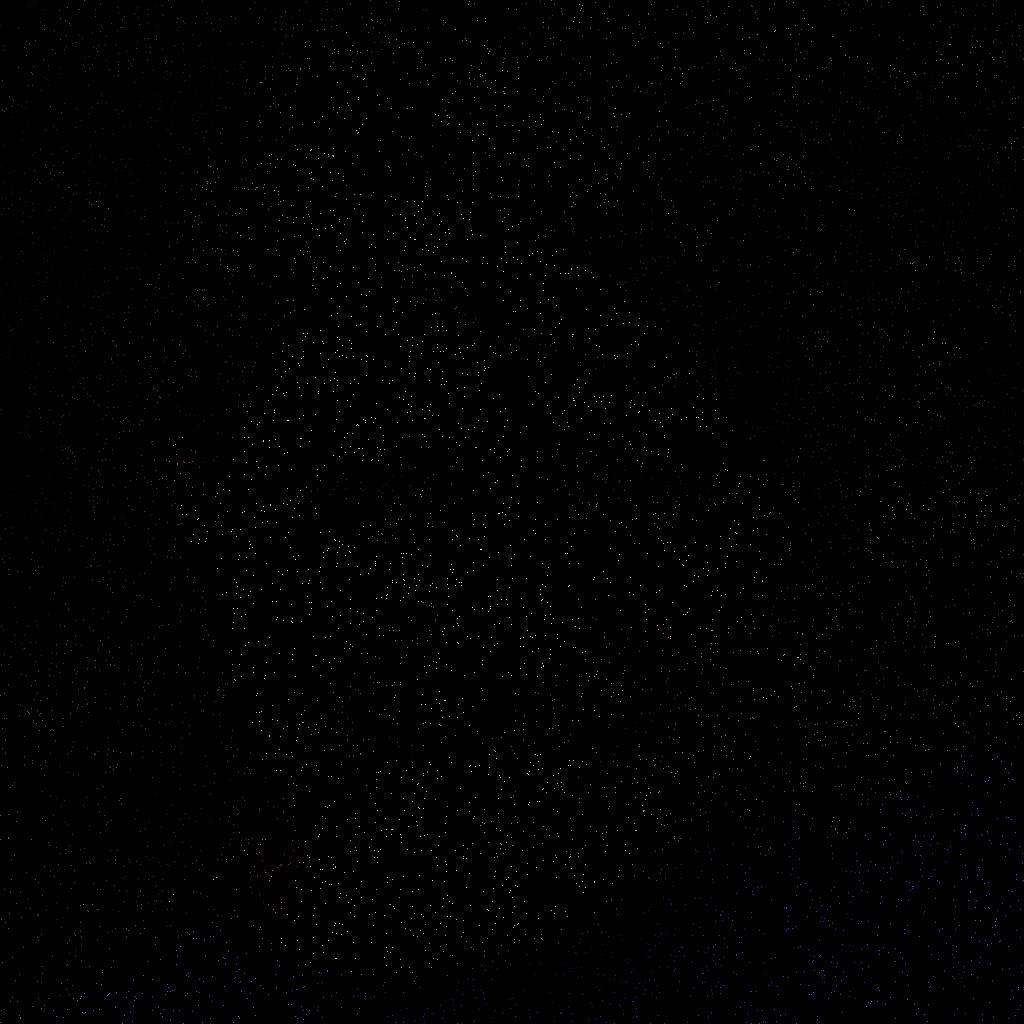} 
\end{minipage} &
\begin{minipage}{0.2\textwidth}
\includegraphics[width=\linewidth]{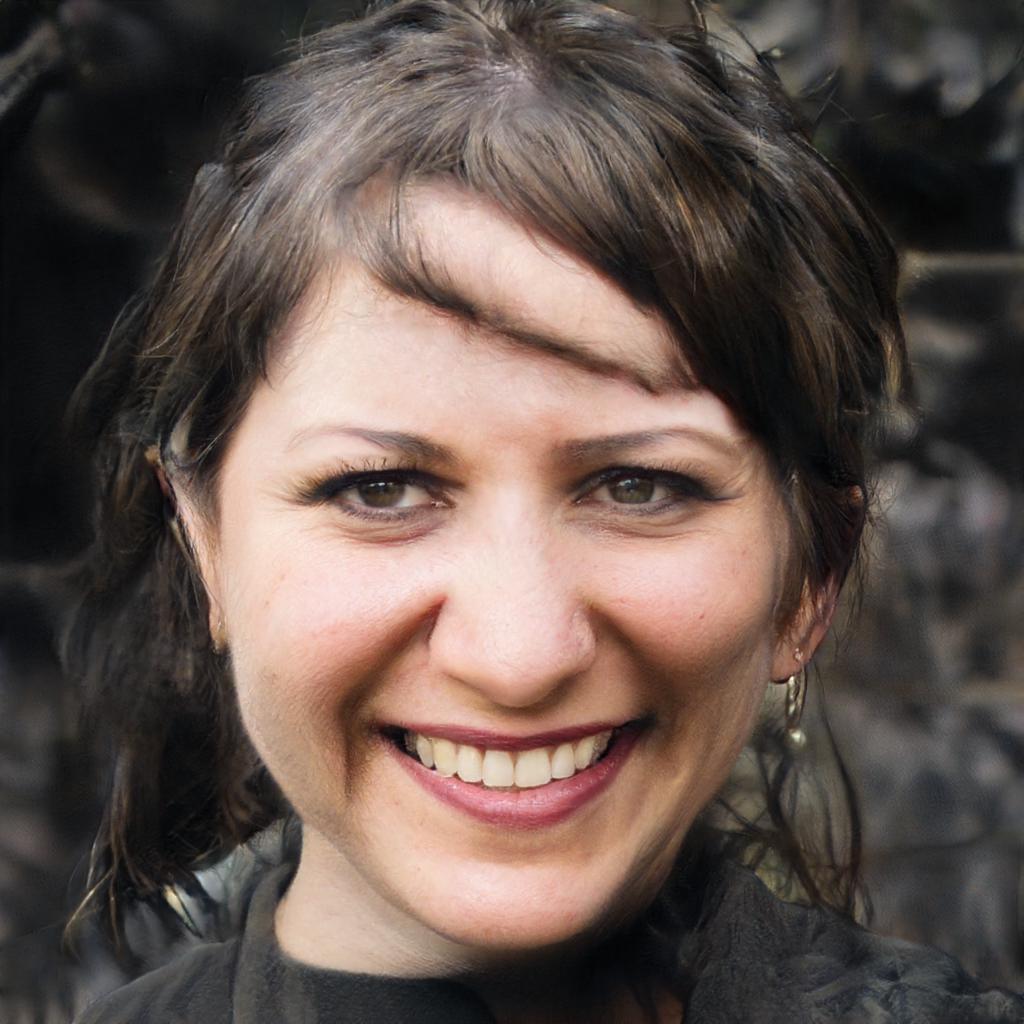} 
\end{minipage} & 
\begin{minipage}{0.2\textwidth}
\includegraphics[width=\linewidth]{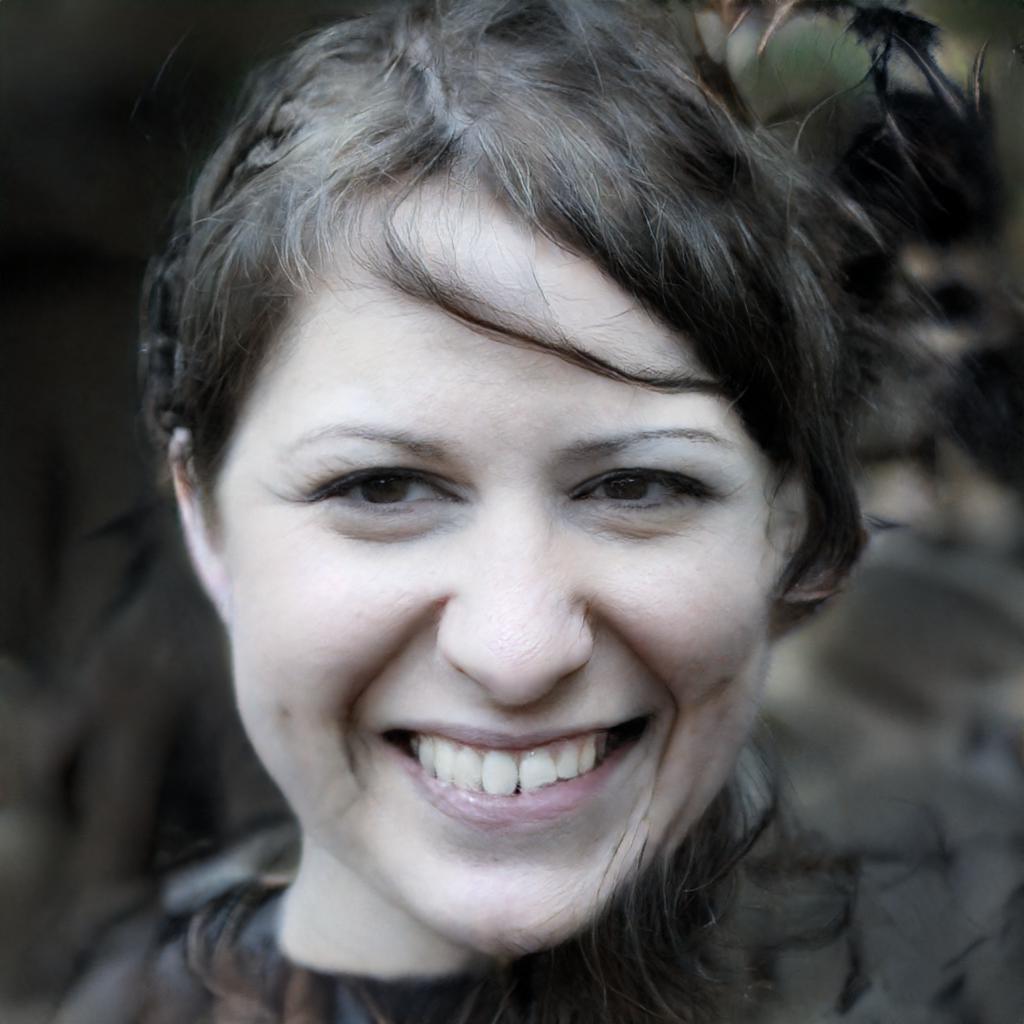} 
\end{minipage} &
\begin{minipage}{0.2\textwidth}
\includegraphics[width=\linewidth]{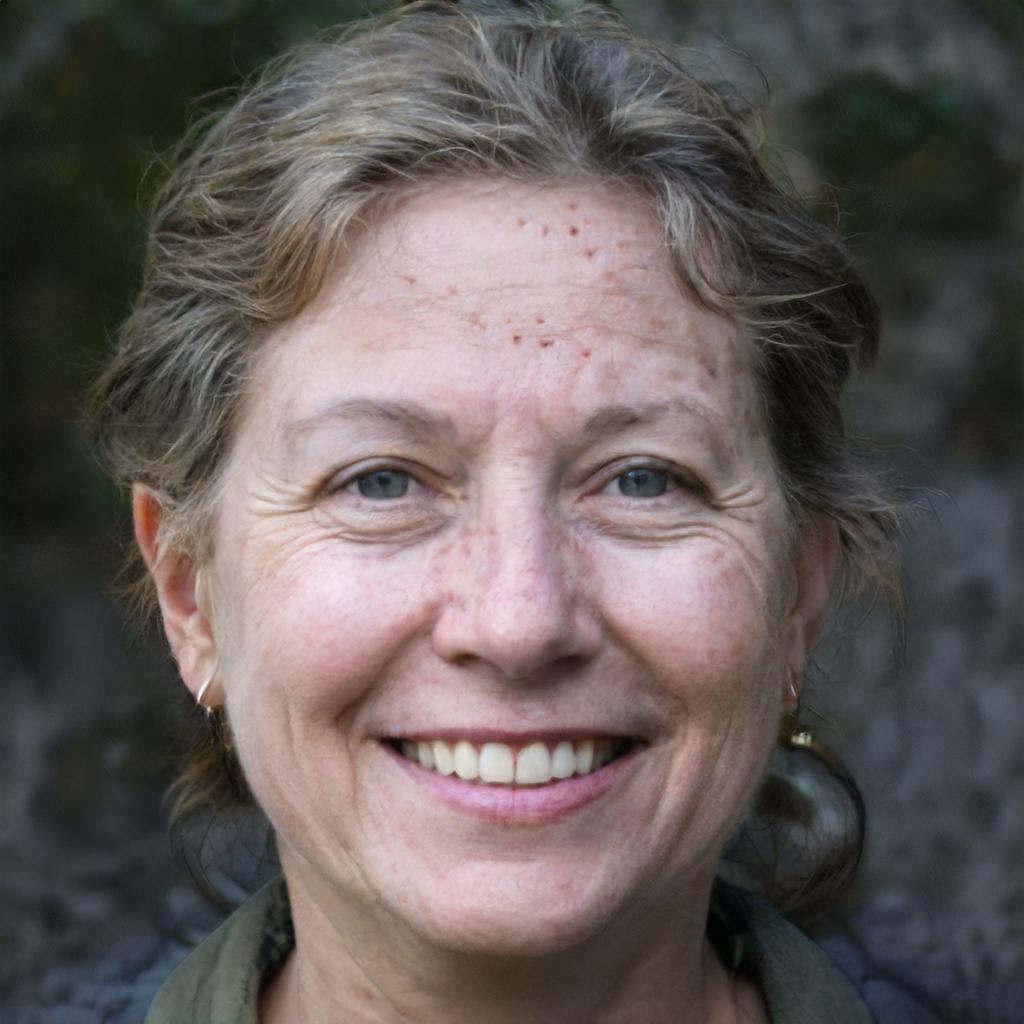} 
\end{minipage} \\ 
\begin{minipage}{0.2\textwidth}
\includegraphics[width=\linewidth]{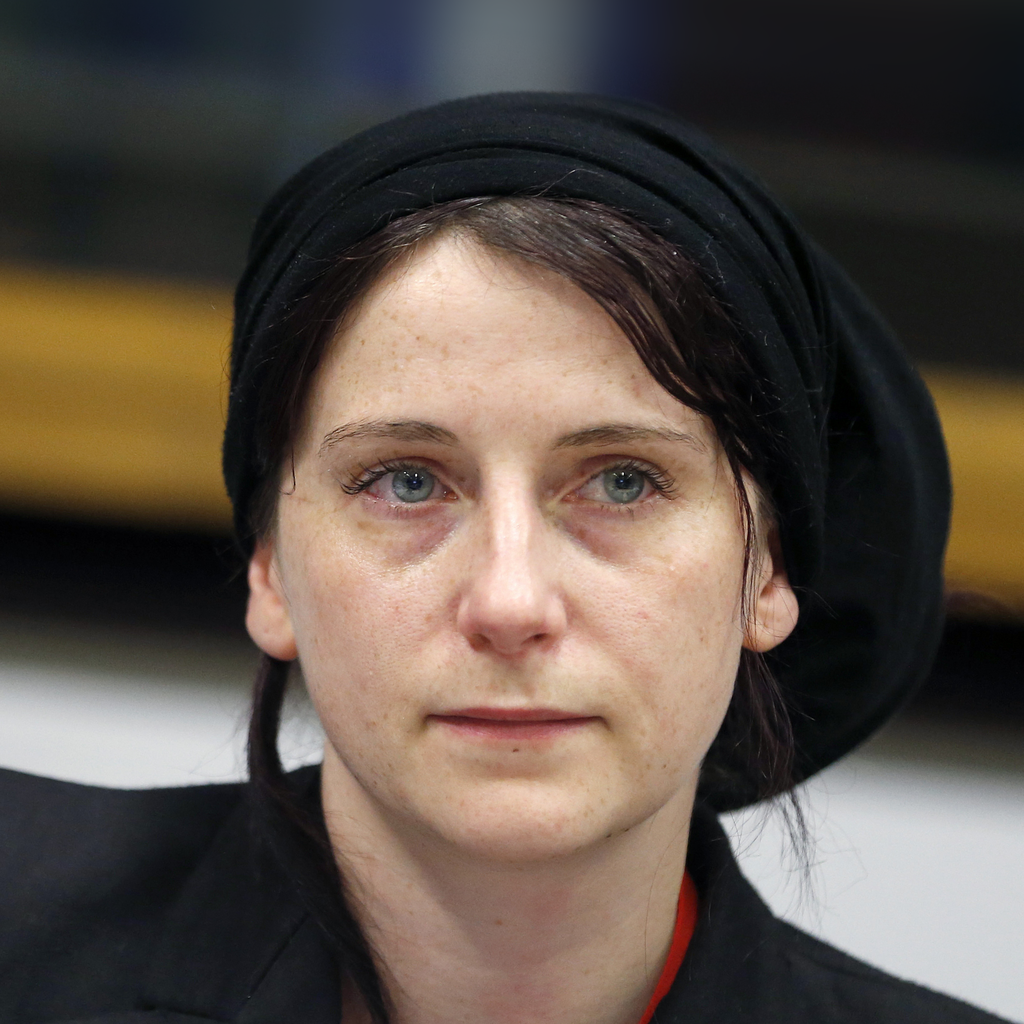} 
\caption*{Reference}
\end{minipage} &
\begin{minipage}{0.2\textwidth}
\includegraphics[width=\linewidth]{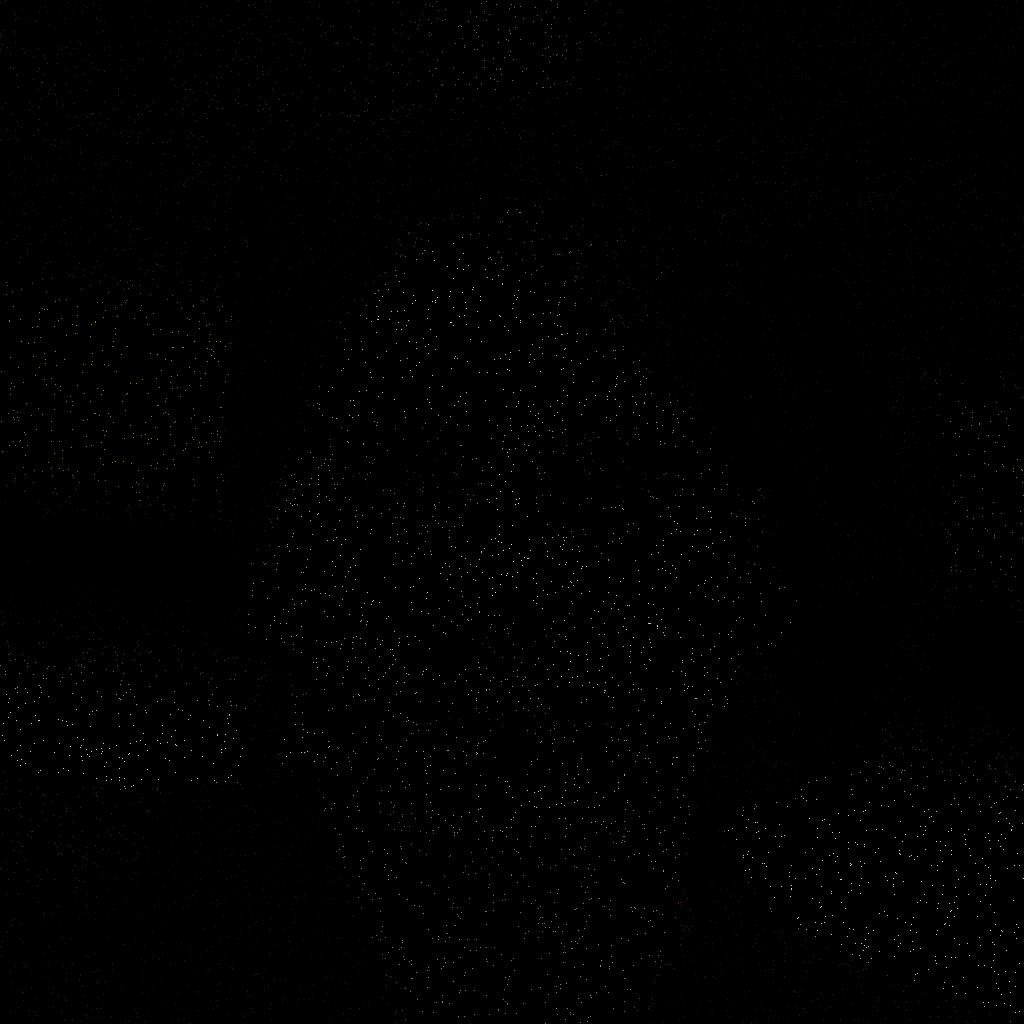} 
\caption*{Input}
\end{minipage} &
\begin{minipage}{0.2\textwidth}
\includegraphics[width=\linewidth]{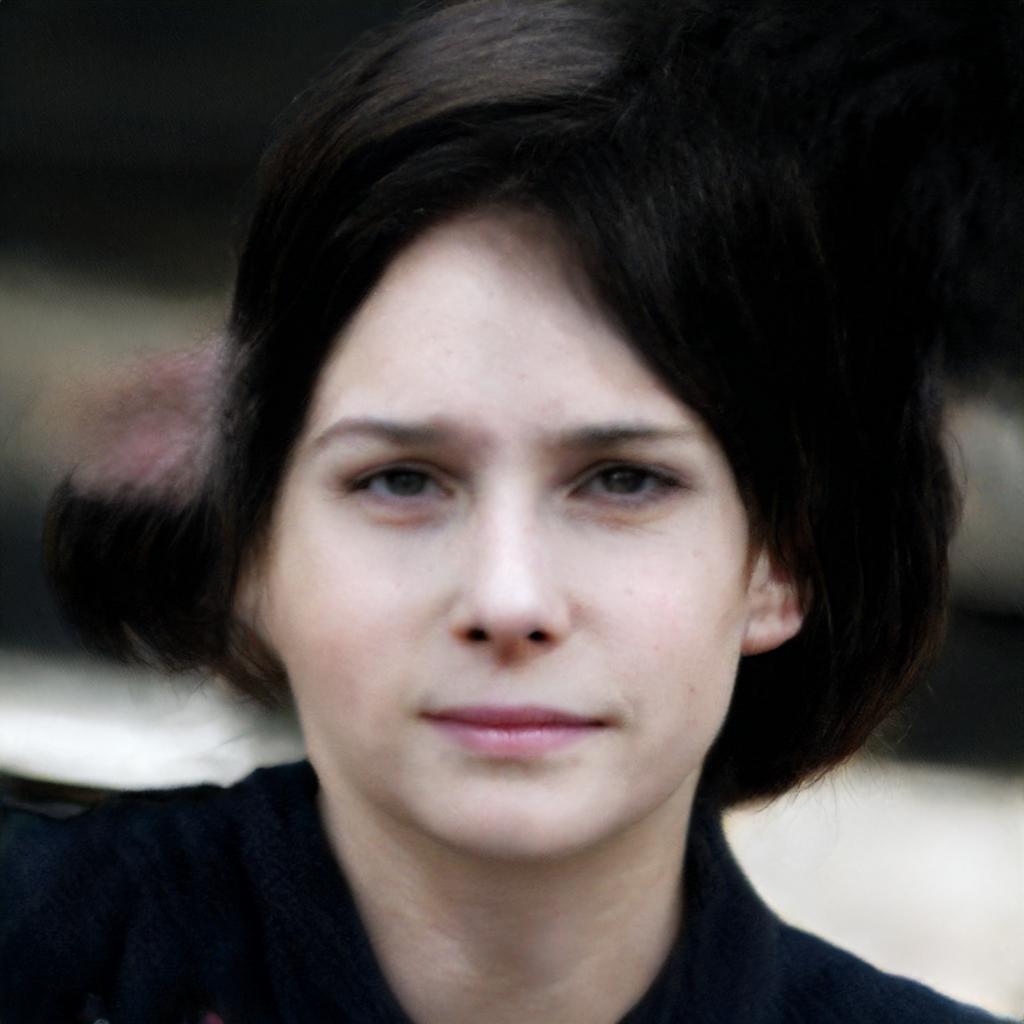} 
\caption*{SGILO (Ours)}
\end{minipage} & 
\begin{minipage}{0.2\textwidth}
\includegraphics[width=\linewidth]{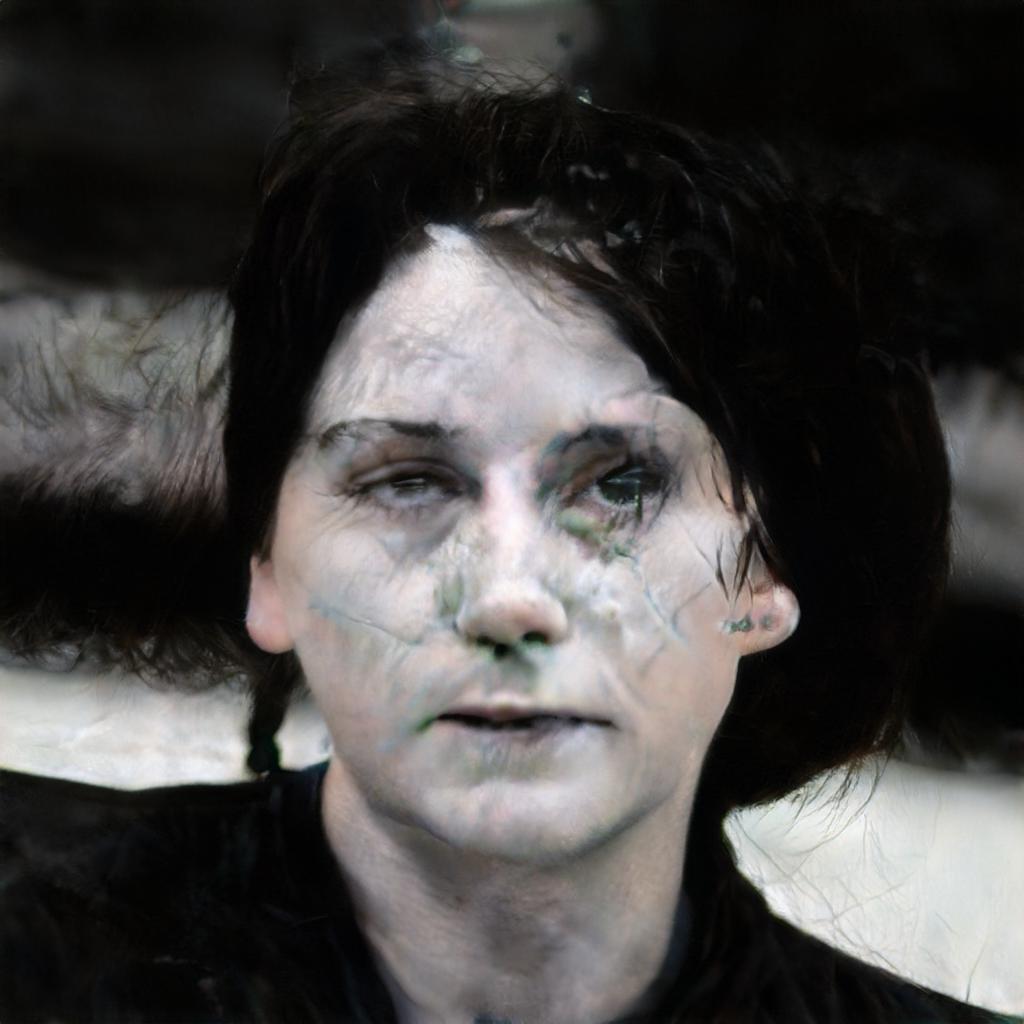} 
\caption*{ILO}
\end{minipage} &
\begin{minipage}{0.2\textwidth}
\includegraphics[width=\linewidth]{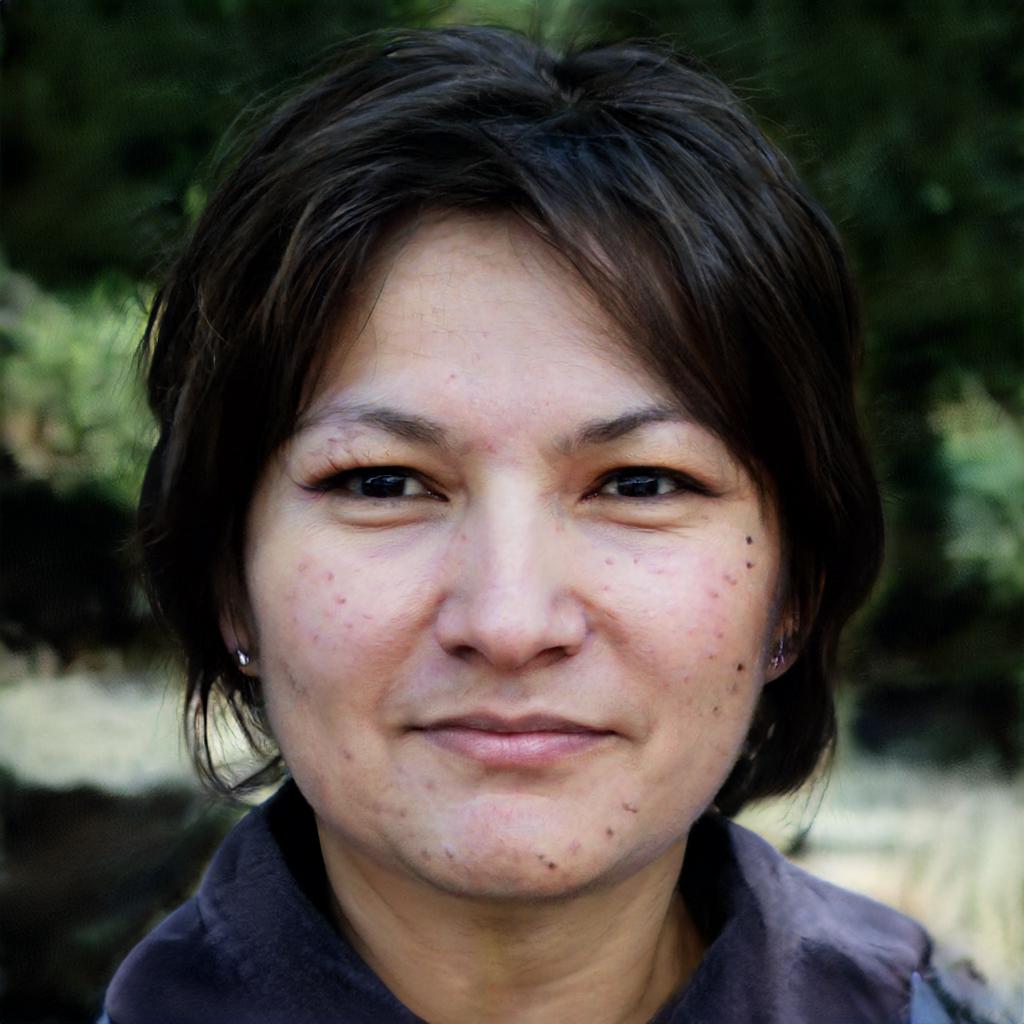} 
\caption*{CSGM}
\end{minipage}
\end{tabular}
\end{adjustbox}
\caption{\small{Results on randomized inpainting in the very challenging regime of only $\mathbf{0.75\%}$ observed pixels (with random sampling). 
    The input seems completely black unless zoomed in. The proposed SGILO benefits from the intermediate layer score-based model to remove artifacts and unnatural colors that appear in ILO~\cite{daras2021intermediate}. CSGM~\cite{bora2017compressed} is constrained to be on the range of StyleGAN2 and hence produces high quality images that, however, do not resemble much the (unobserved) reference. We emphasize that these are real reference images that have not been used in training, for any of the models. }}
\label{inpainting}
\end{center}
\end{figure*}

\section{Score Guided Intermediate Layer Optimization}
\label{sec:method}
\paragraph{Setting} Let $x^* \in \R^n$ be an unknown vector that is assumed to lie in the range of a pre-trained generator $G(z):\R^k\to\R^n$, i.e. we assume that there is a $z^*\in \R^k$ such that: $x^* = G(z^*)$. We observe some noisy measurements of $x^*$, i.e. the vector $y = \mathcal A(x^*) + \xi \in \R^m$, where $A: \mathbb R^n \to \mathbb R^m$ is a known, differentiable forward operator and $\xi \sim \mathcal N(0, \sigma^2I)$.

\paragraph{Posterior Sampling in the Latent Space}  We first want to characterize the posterior density $p(z|y)$. Applying Bayes rule, we get that: $p(z|y) = \frac{p(z, y)}{p(y)} \propto p(y|z) p(z)$. The noise is assumed Gaussian, so $p(y|z) = \mathcal N(y; \mu = \mathcal A(G(z)), \Sigma= \sigma^2I)$. Hence, 

\begin{gather}
    \log p(z|y) \propto \underbrace{\frac{1}{2\sigma^2}||\mathcal A(G(z)) - y||_2^2 -\log p(z)}_{L(z)}\enspace.
    \label{loss_function}
\end{gather}

To derive this posterior, we assumed that $x^*$ is in the range of the generator $G$. This assumption is somewhat unrealistic for the Gaussian latent space of state-of-the-art GANs, such as StyleGAN~\cite{stylegan, stylegan2} which motivates optimization over an intermediate space, as done in ILO \cite{daras2021intermediate}. 

ILO has two weaknesses: i) it is a MAP method while there is increasing research showing the benefits of posterior sampling~\cite{jalal_cond_resampling, jalal2021fairness, nguyen2021provable}, ii) it is assuming a handcrafted prior which is uniform in an $l_1$ dilation of the range of the previous layers and $0$ elsewhere.

Instead, we propose a new framework, Score-Guided Intermediate Layer Optimization (SGILO), that trains a score-based model in the latent space of some intermediate layer and then uses it with Stochastic Gradient Langevin Dynamics to sample from $e^{-b L(z)}$ for some temperature parameter. 

Figure \ref{fig:fig2} illustrates the central idea of SGILO. As shown, ILO optimizes in the intermediate layer $\mathbb{R}^p$ assuming a uniform prior over the expanded manifold (that is colored green). In this paper, we learn a distribution in the intermediate layer using a score based model. This learned distribution is shown by orange geodesics and can expand outside the $\ell_1$-ball dilated manifold. 

Table \ref{tab:reconstruction_algorithms_summary} summarizes the strengths and weaknesses of the following reconstruction algorithms: i) Gradient Descent (GD) in the latent space of the first layer of a pre-trained generator as in the CSGM~\cite{bora2017compressed} framework, ii) (Projected) GD in the latent space of an intermediate layer, as in ILO~\cite{daras2021intermediate}, iii) Stochastic Gradient Langevin Dynamics (SGLD) in the pixel space, as done by ~\citet{jalal_cond_resampling, song2021solving} and others with Score-Based Networks and iv) SGLD in the intermediate space of some generator, as we propose in this work. Notation wise, for the GAN based methods (Rows 1, 2), we think of a generator as a composition over two transformations $G_1:\R^k\to\R^p$ and $G_2:\R^p \to \R^n$, where $k < p < n$.

Gradient Descent in the intermediate space, as in the ILO paper, can be expressive (increased expressivity due to ILO) and fast (GAN-based methods) but does not offer diverse sampling. On the other hand, Stochastic Gradient Langevin Dynamics in the pixel space is slow as it is usually done with high-dimensional score-based models. SGILO (Row 4) combines the best of the worlds of GAN-based inversion and posterior sampling with Score-Based Networks. Specifically, it is expressive (optimization over an intermediate layer), it offers diverse sampling (posterior sampling method) and it is fast (dimensionality $p < n$). Experimental evidence that supports these claims is given in Section \ref{sec:experiments}.

\begin{figure*}[!ht]
    \centering
    \includegraphics[width=0.6\textwidth]{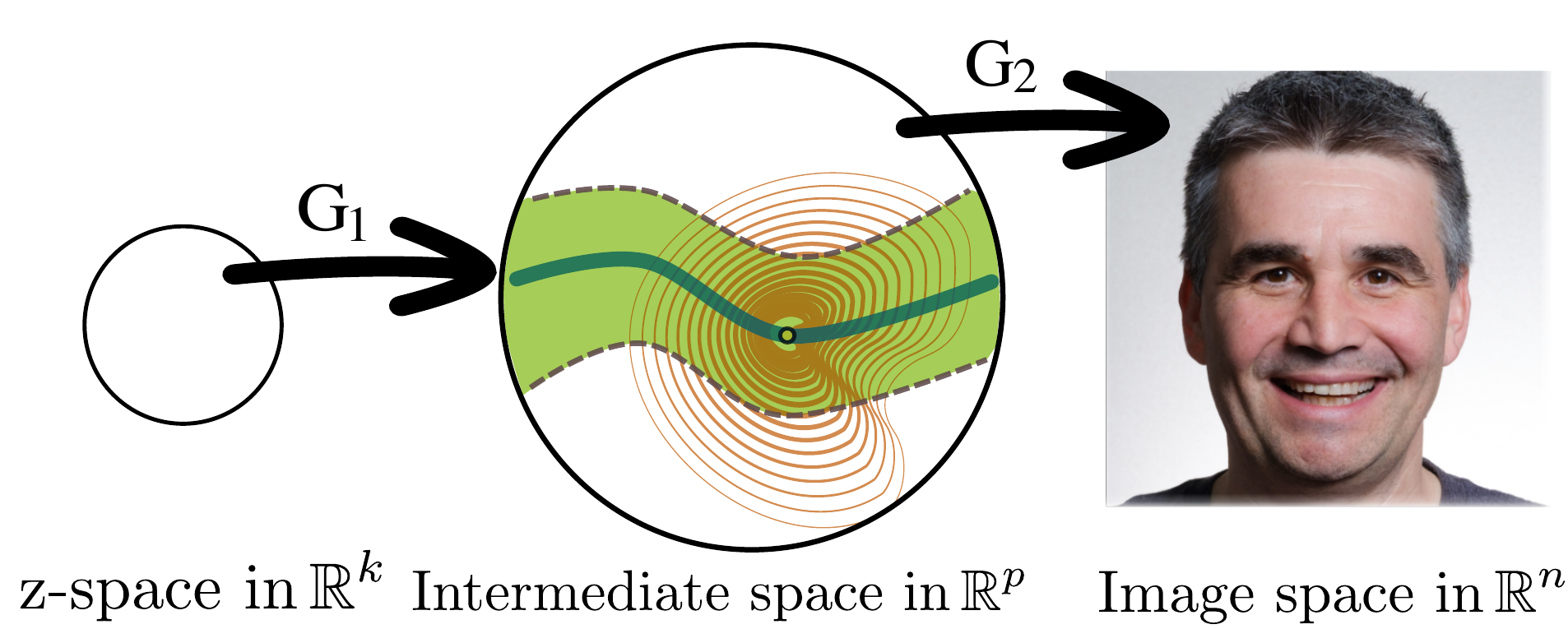}
    \caption{\small An illustration of SGILO. 
    In previous work (ILO~\cite{daras2021intermediate}) a generator is considered as the composition of two transformations $G_1$ from the latent space to an intermediate space $\mathbb{R}^p$ and a second transformation $G_2$ from the intermediate layer to the image space. 
    The range of the generator $G_1$ is a $k$ dimensional manifold in $\mathbb{R}^p$ shown with a blue line in the figure. ILO expands this by taking the Minkowski sum of the manifold with the $\ell_1$ ball. ILO
    optimizes in the intermediate layer $\mathbb{R}^p$ assuming a uniform prior over the expanded manifold, shown as the green set. In this paper we learn a distribution in the intermediate layer using a score based model. This learned distribution is shown by orange geodesics and can expand outside the $\ell_1$-ball dilated manifold. }
    \label{fig:fig2}
\end{figure*}

\begin{figure*}[!ht]
\captionsetup[subfigure]{labelformat=empty}
\captionsetup{justification=centering}
\begin{center}
\begin{adjustbox}{width=0.8\textwidth, center}
\begin{tabular}{cc} 
\subfloat{\includegraphics[width=\textwidth]{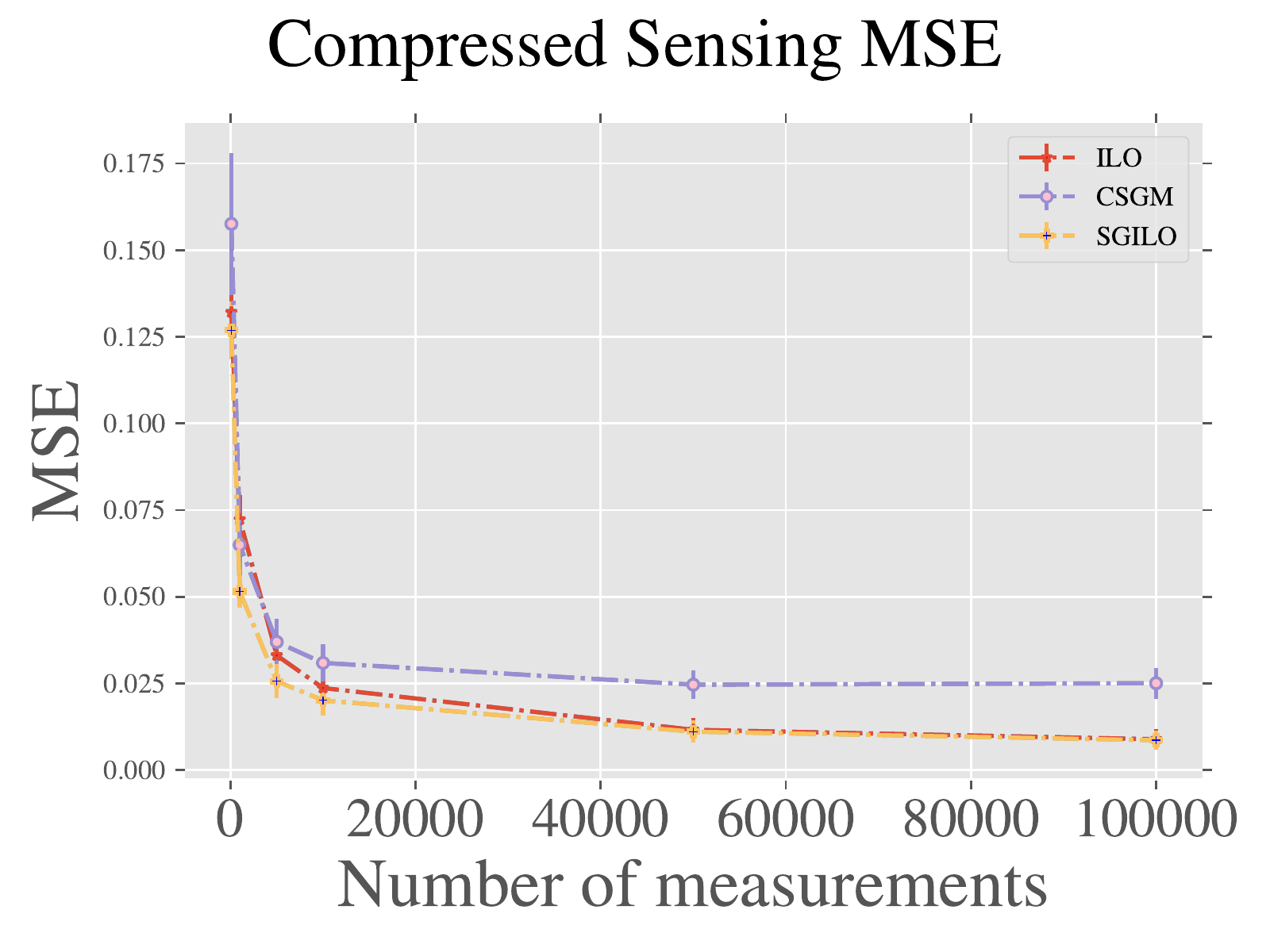}} & 
\subfloat{\includegraphics[width=\textwidth]{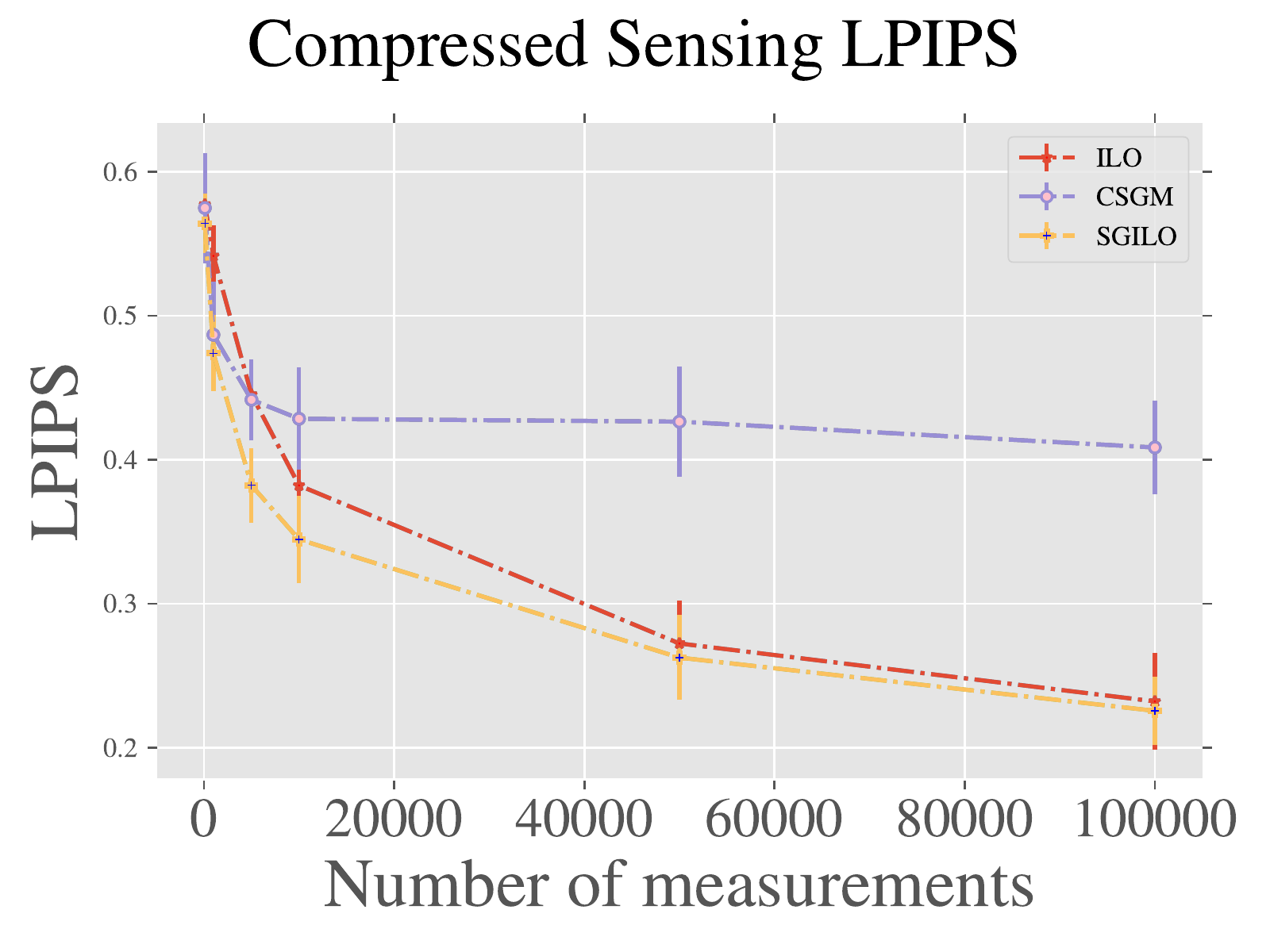}} \\
\subfloat{\includegraphics[width=\textwidth]{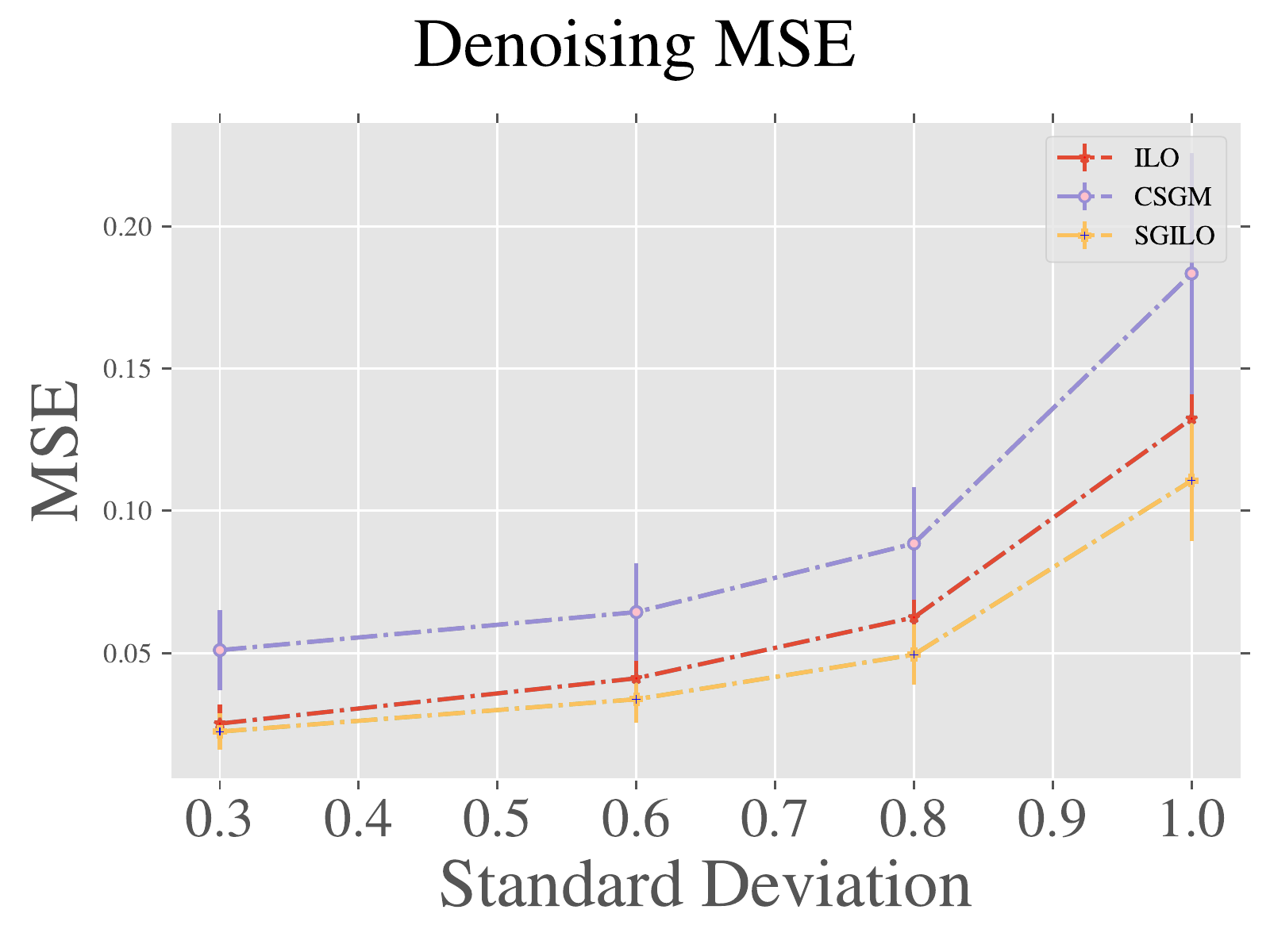}} & 
\subfloat{\includegraphics[width=\textwidth]{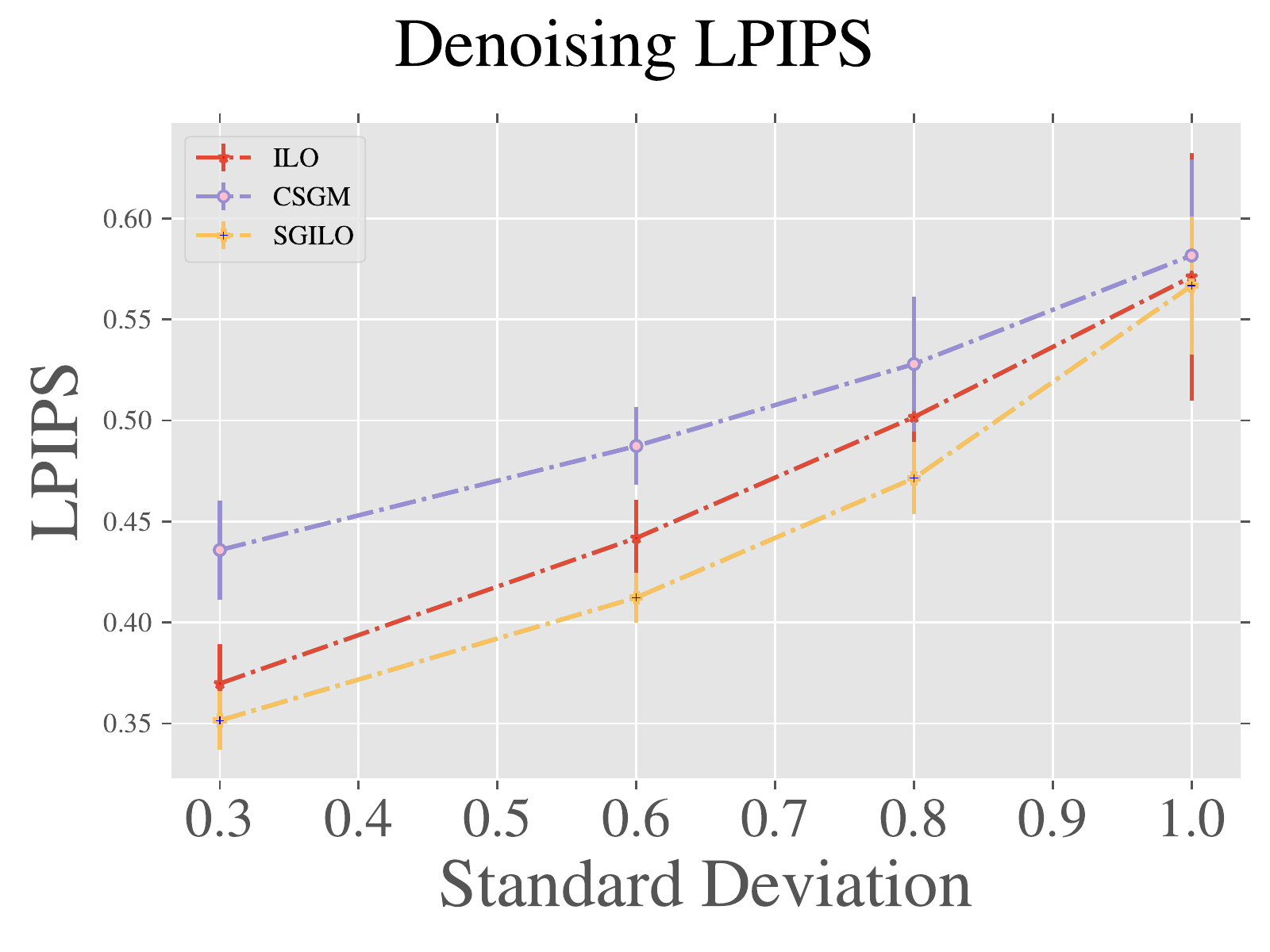}} \\
\subfloat{\includegraphics[width=\textwidth]{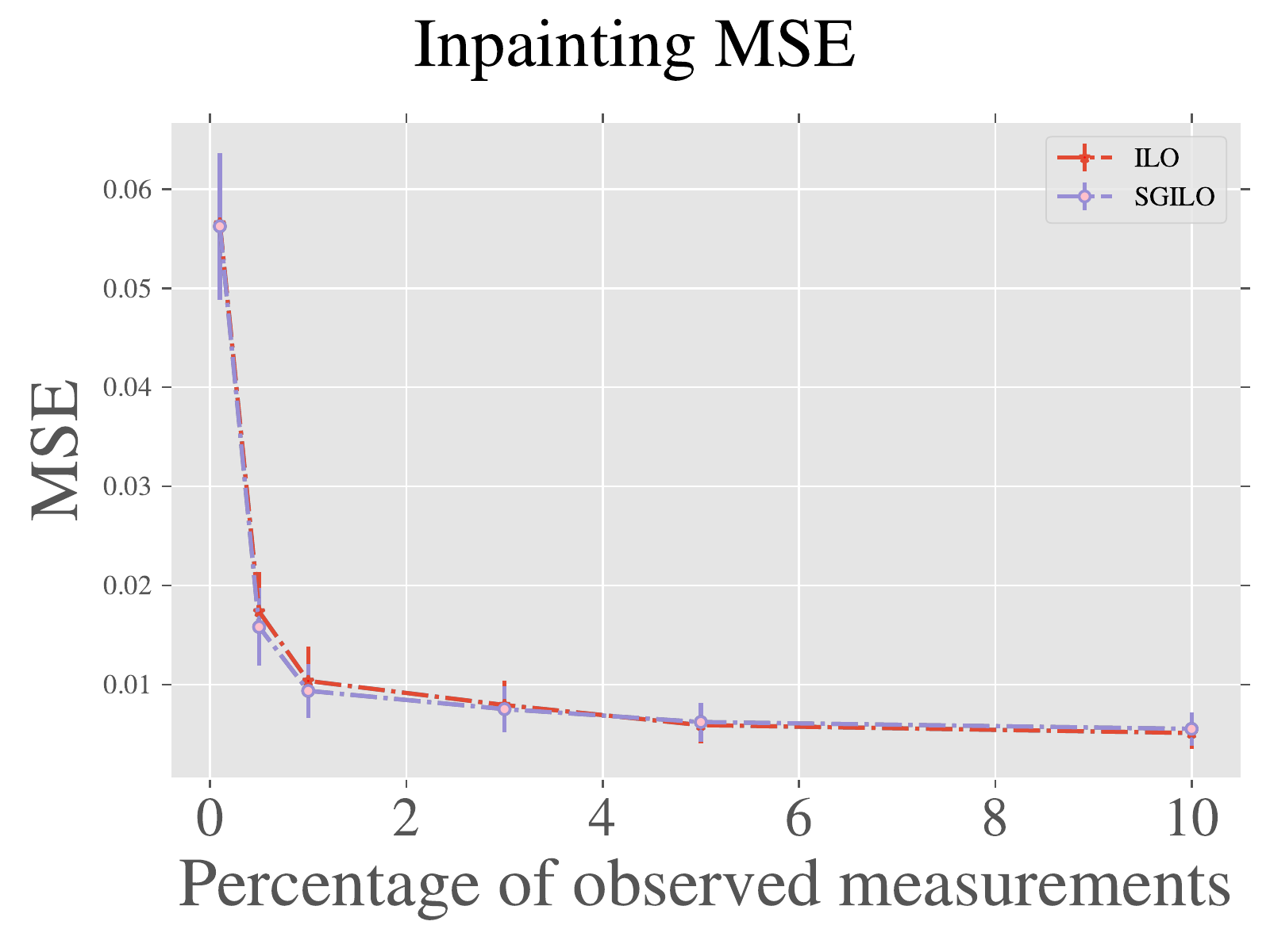}} & 
\subfloat{\includegraphics[width=\textwidth]{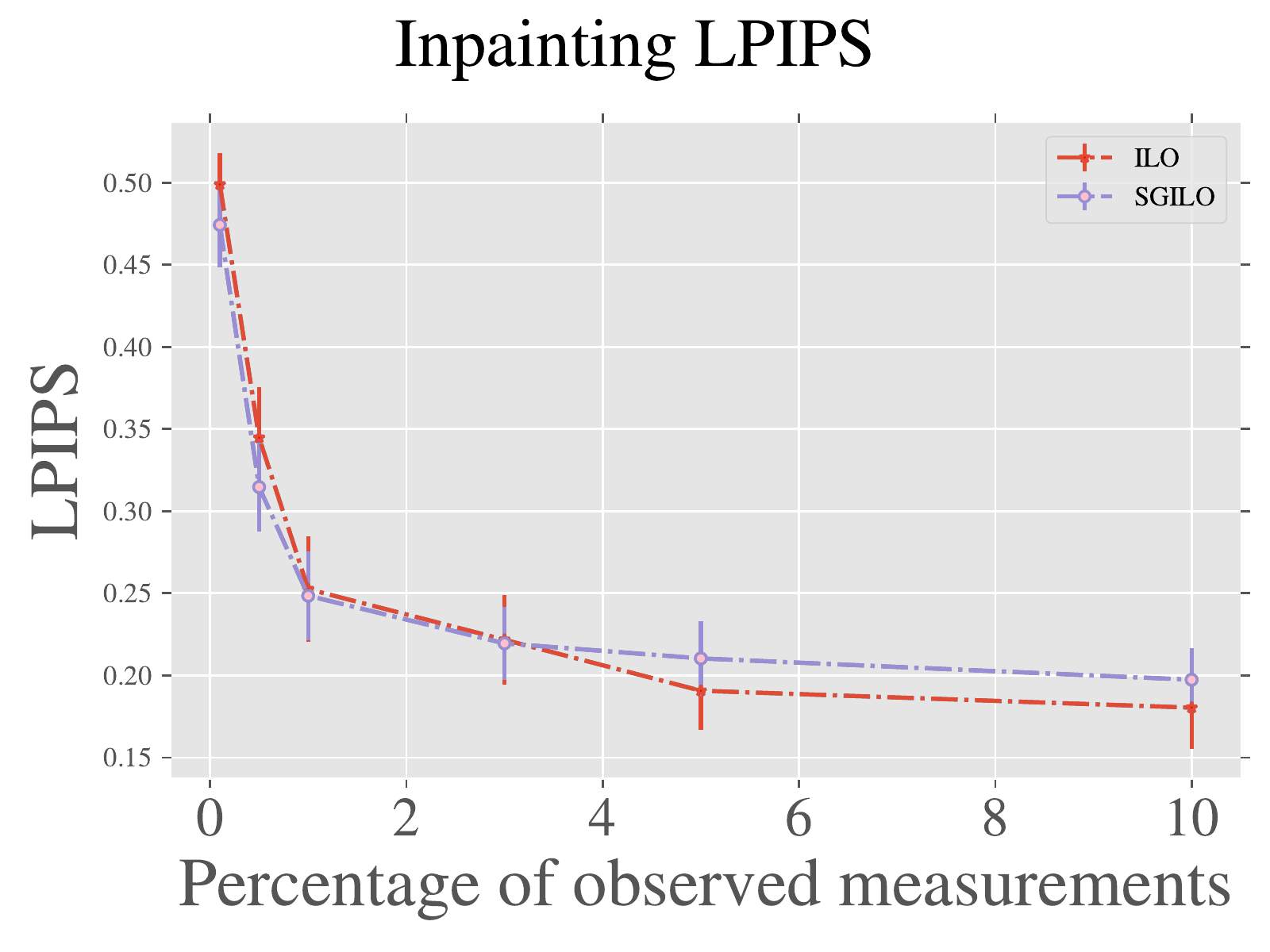}} 
\end{tabular}
\end{adjustbox}
\caption{\small{Quantitative results on the task of compressed sensing, denoising and inpainting. Our method, SGILO, significantly outperforms the state-of-the-art unsupervised method ILO when the measurements are scarce or the level of the noise is high. In this challenging regime, the prior from the score-based model is a much better regularizer than the sparse deviations constraint of ILO, yielding significant performance boosts.}}
\label{plots}
\end{center}
\end{figure*}

The last column of Table \ref{tab:reconstruction_algorithms_summary} characterizes the different algorithms with respect to what we know about their convergence. A recent line of work~\cite{paul_hand_v1, paul_hand_v2, daskalakis_constant_expansion} has been able to prove that despite the non-convexity, for neural networks with random Gaussian weights, a signal in the range of the network can be approximately recovered using Gradient Descent (with sign flips) in polynomial time under an expansion assumption in the dimension of the layers of the generator. This motivates the question of whether we can prove under the same setting, that a Langevin Sampling algorithm would converge fast to a stationary measure. The next section, answers affirmatively this question while even removing the need for sign flipping. The theoretical results hold for the CSGM setting, but can apply to the optimization in an intermediate layer with a uniform prior over the latents. Unfortunately, assuming uniformity in the intermediate layer is not a realistic assumption. Proving distributional convergence of SGILO under more realistic assumptions is left for future work.

\section{Theoretical Results}
We are now ready to state the main Theorem of our paper.

\label{sec_theory}
\begin{theorem}[Informal]
Consider the Markov Chain defined by the following Langevin Dynamics:
\begin{gather}
    z_{t+1} = 
    z_t - \eta \nabla f(z_t) + \sqrt{2\eta \beta^{-1}}u
\end{gather}
where $u$ is a zero-mean, unit variance Gaussian vector, i.e. $u_{ij} \sim \mathcal N(0, \sigma^2=1)$, $G(z)$ is a fully-connected $d$-layer ReLU neural network,
\begin{gather}
    G(z) = \relu\left(W^{(d)} \left(\cdots\relu\left(W^{(1)}z \right) \cdots\right) \right)\notag
\end{gather}
and $f(z)$ is the loss function:
\[
    f(z) = \beta\|AG(z) - y\|_2^2
\]
where $A \in \R^{m\times k}$, and $y=AG(z^*)$, for some unknown vector $z^* \in \mathbb{R}^n$.

Define $\mu(z) \propto e^{-f(z)}$ and $z_t\sim Z_t$, then for any $\epsilon >0$ and for $t \ge \Omega(\log(1/\epsilon)/\epsilon^2)$,
\begin{align*}
\mathcal W(Z_t, \mu) 
:= \inf_{\text{$Q\in\{$couplings of $Z_t,\mu\}$}} \mathbb{E}_{(z_t,z)\sim Q}\|z_t-z\|\\
\le (\epsilon + e^{-\Omega(n)}) \|z^*\|,
\end{align*}
provided that $\eta = \Theta(\epsilon^2)$, that $\beta = Cn$ (for some sufficiently large constant $C$), that $\|z_0\|\le O(\|z^*\|)$, that $W^{(i)}$ and $A$ satisfy conditions WDC and RRIC
 \cite{paul_hand_v1} and $d\ge 2$ can be any constant. 
\label{main_theorem}
\end{theorem}
We note that $\beta = \Theta(n)$ is the right choice of parameters since a smaller $\beta$ produces approximately random noise and a larger $\beta$ produces a nearly deterministic output.

\paragraph{Sketch of the proof:}
We analyze the landscape of the loss function $f$. It was already noted by \citet{hand2018global} that it has three points where the gradient vanishes: at the optimum $\x^*$, at a point $-\rho \x^*$ for some $\rho \in (0,1)$ and at $0$, a local maxima. In order to escape the stationary point $-\rho \x^*$, \citet{hand2018global} proposed to flip the sign of $\x$ whenever such flipping reduces the loss.
We write $f$ in a more compact fashion, obtaining that $-\rho\x^*$ is a saddle point. We show that the noise added by the Langevin dynamics can help escaping this point, and converging to some ball around $\x^*$. This is proven via a potential function argument: we construct a potential $V$ and show that it decreases in expectation after each iteration, as long as the current iteration is far from $\x^*$. We note that the expected change in $V$ is measured in the continuous dynamics by a Laplace operator $\mathcal LV$. In this paper, we use this to show that the potential decreases in the continuous dynamics, and compare the continuous to the discrete dynamics. 

Finally, our goal is to couple the discrete dynamics to the continuous dynamics that sample from $\mu$. Once we establish that the continuous and discrete dynamics arrive close to $\x^*$, we use the fact that $f$ is strongly convex in this region to couple them in such a way that they get closer in each iteration, until they are $\epsilon$-close, and this concludes the proof. The full proof and the detailed formal statement of the theorem can be found in the Appendix.

\section{Experimental Results}
\label{sec:experiments}

We use StyleGAN-2~\cite{stylegan, stylegan2} as our pre-trained GAN generator. Score-based models are trained as priors for internal StyleGAN-2 layers. We use a variant of the Vision Transformer~\cite{ViT} as the backbone architecture for our score-based models. To incorporate time information, we add Gaussian random features~\citep{tancik2020fourier} to the input sequence, as done in \citet{score_first} for the U-net~\citep{unet} architecture. The score-based models are trained with the Variance Preserving (VP) SDE, defined in \citet{song_sde}.

Transformers are not typically used for score-based modeling. This is probably due to the quadratic complexity of transformers with respect to the length of the input sequence, e.g. for training a 1024x1024x3 score-based model, the Transformer would require memory proportional to $1024^2\times1024^2\times3^2$. Since our score-based models learn the distribution of intermediate StyleGAN-2 layers, we work with much lower dimensional objects. For the score-based model, we use a ViT Transformer with $8$ layers, 1 attention head and dimension $1024$. For the VP-SDE we use the parameters in \citet{song_sde}. For more information on implementation and hyperparameters, please refer to our open-sourced \href{https://github.com/giannisdaras/sgilo}{code}.

\paragraph{Dataset and training.} The score-based model is trained by creating a dataset of intermediate StyleGAN inputs (latents and intermediate outputs). We inverted all images in FFHQ~\cite{stylegan} with ILO, and used the intermediate outputs as training data for our score-based model.  

We train score-based models to learn the distribution of: i) the latents that correspond to the inverted FFHQ and ii) the intermediate distributions of layers $\{1, 2, 3, 4\}$ of the generator. Consistently, we observed that the score-based models for the deeper priors were more powerful in terms of solving inverse problems. This is expected but comes with the cost of more expensive training, which is why we stopped at layer $4$, which is already powerful enough to give us excellent reconstructions.

\paragraph{Unconditional Image Generation.} The first experiment we run aims to demonstrate that the score-based models that we trained on the intermediate distributions are indeed capable of modeling accurately the gradient of the log-density. To this end, we use Annealed Langevin Dynamics, as in \citet{score_first}, to sample from the intermediate distribution of the fourth layer and the distribution of the inverted latents. The results are summarized in Figure \ref{fig:generations}. In the first two rows, we show results when sampling from the intermediate distribution (keeping the noises and the latent vectors fixed). In the last row, we show results when sampling from the distribution of the inverted latents (keeping the noises fixed). As shown, the combination of the score-models and the powerful StyleGAN generators leads to diverse, high-quality generations.

\begin{figure*}[!ht]
\captionsetup[subfigure]{labelformat=empty}
\captionsetup{justification=centering}
\begin{center}
\begin{adjustbox}{width=0.65\textwidth, center}
\begin{tabular}{ccccc} 
\subfloat{\includegraphics[scale=1]{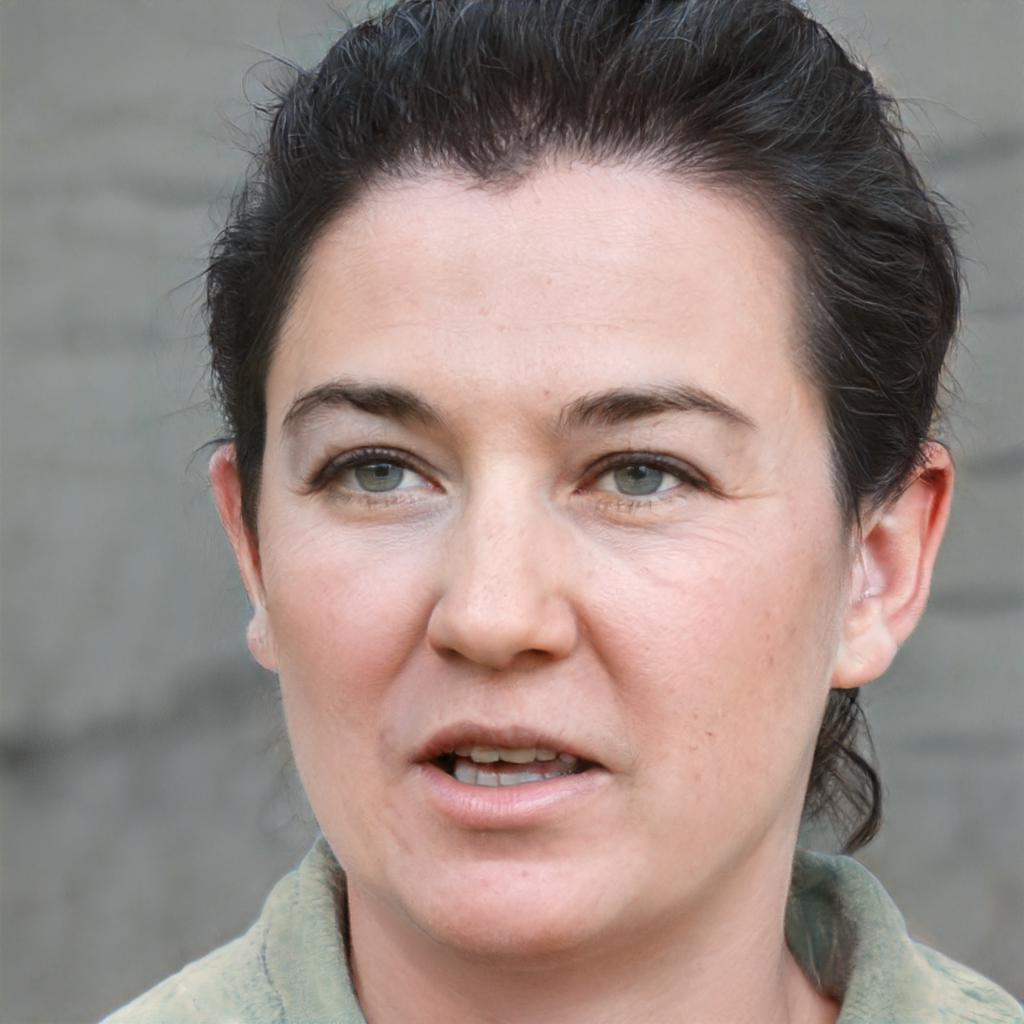}} & 
\subfloat{\includegraphics[scale=1]{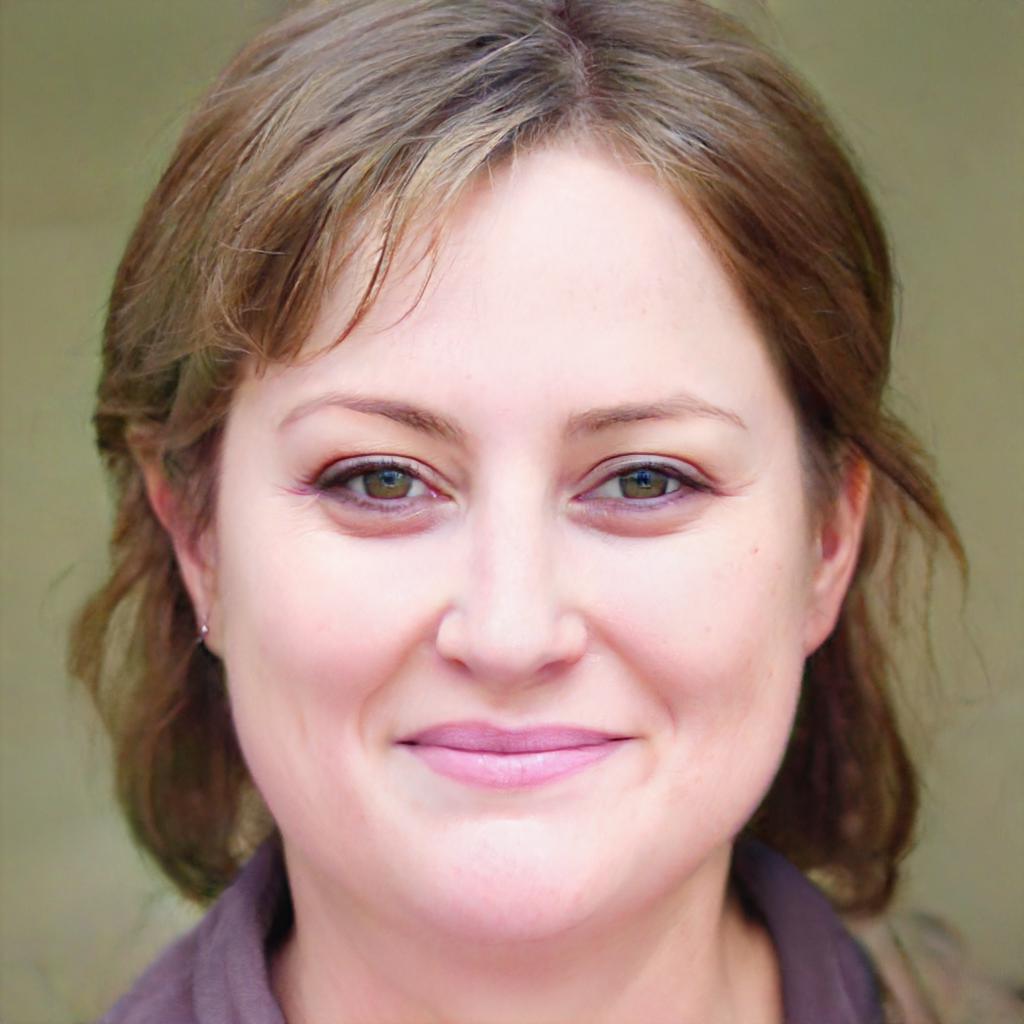}} &
\subfloat{\includegraphics[scale=1]{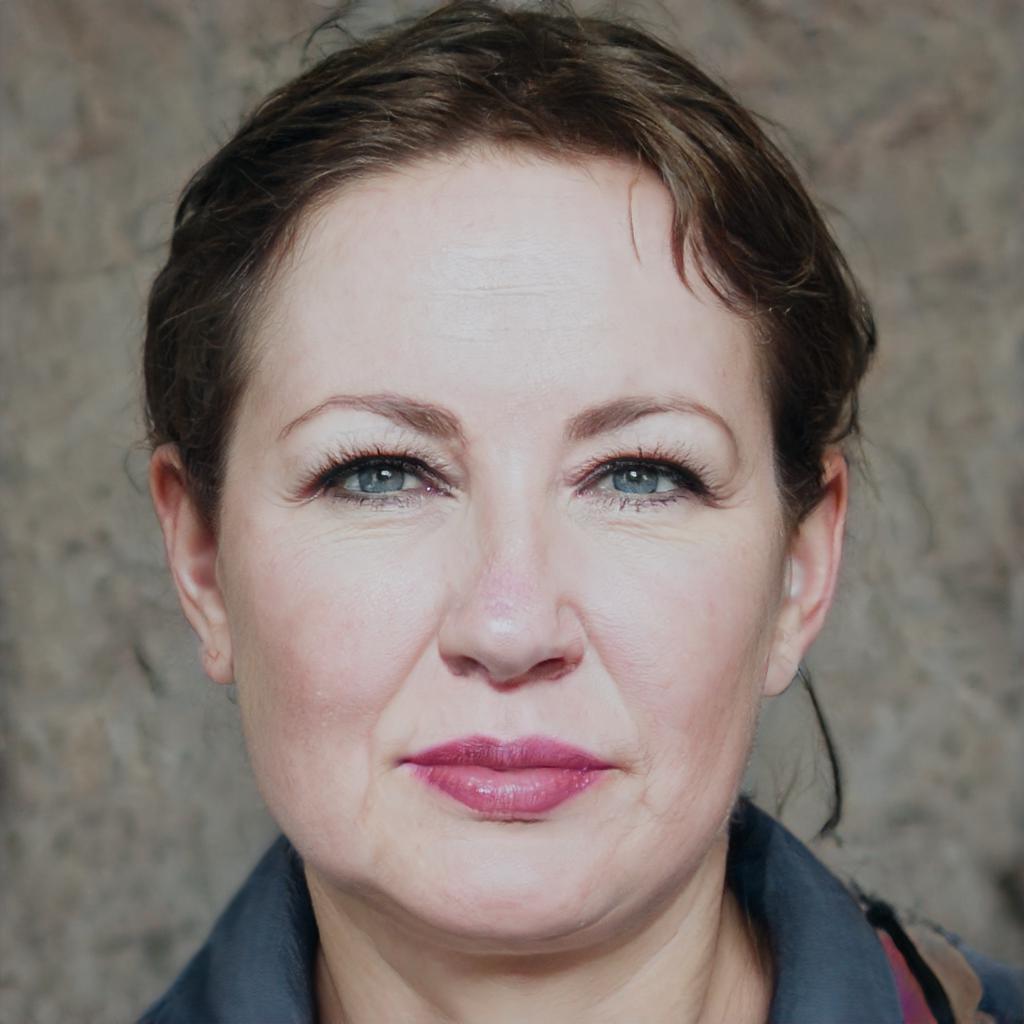}} &
\subfloat{\includegraphics[scale=1]{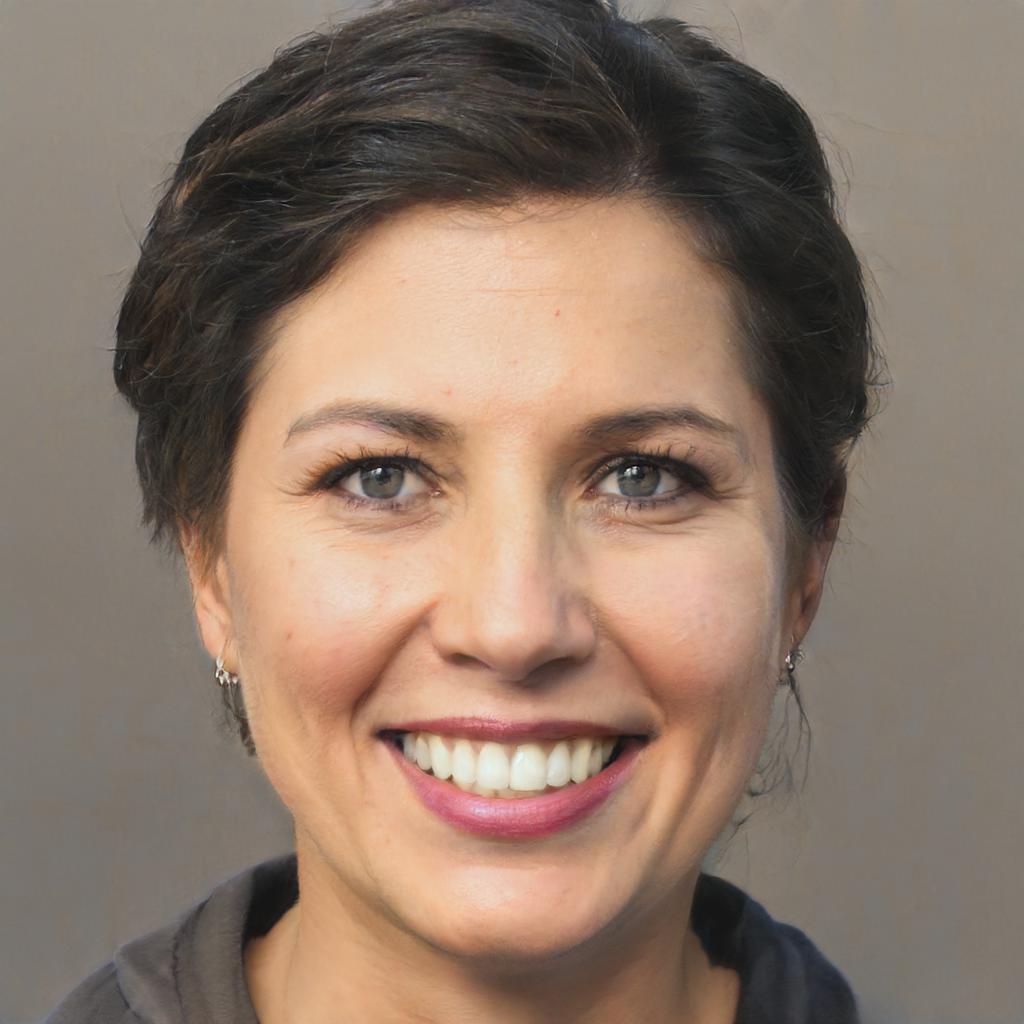}} &
\subfloat{\includegraphics[scale=1]{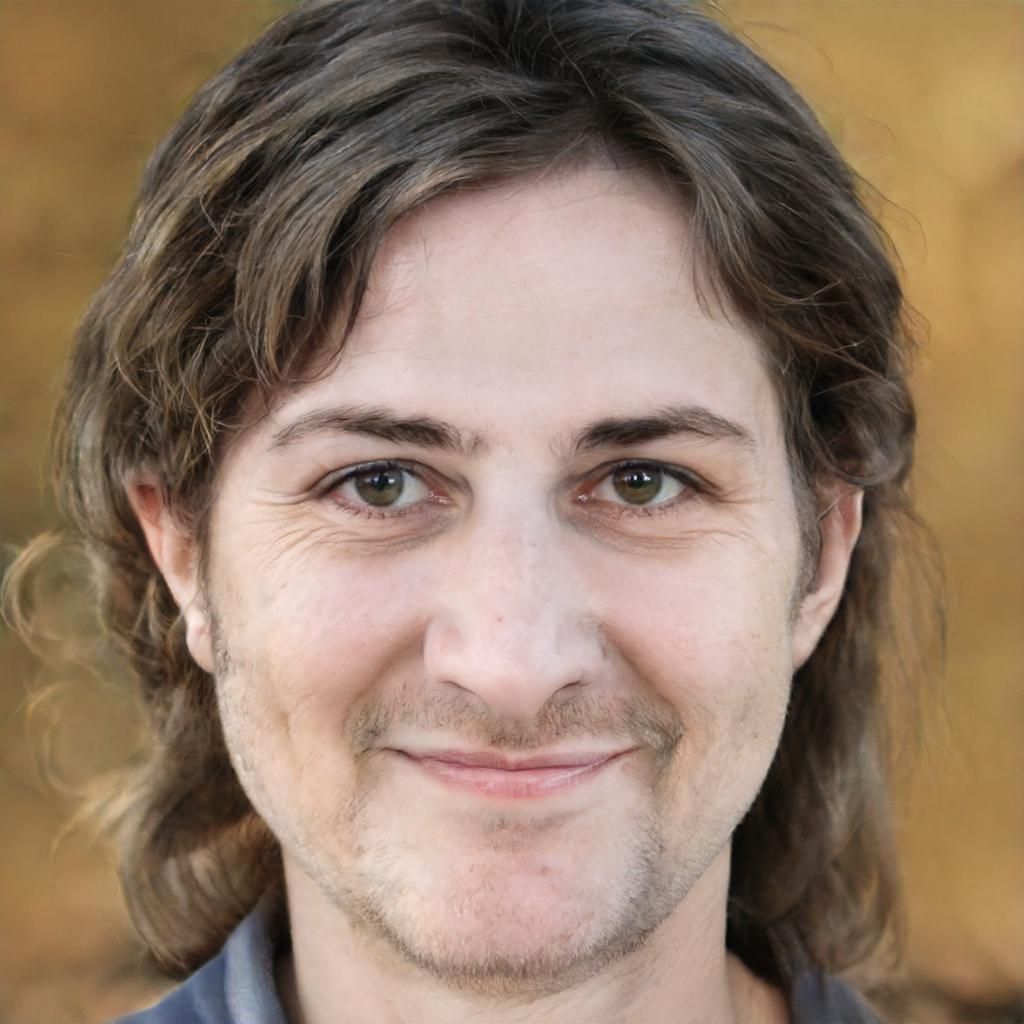}}
\\
\subfloat{\includegraphics[scale=1]{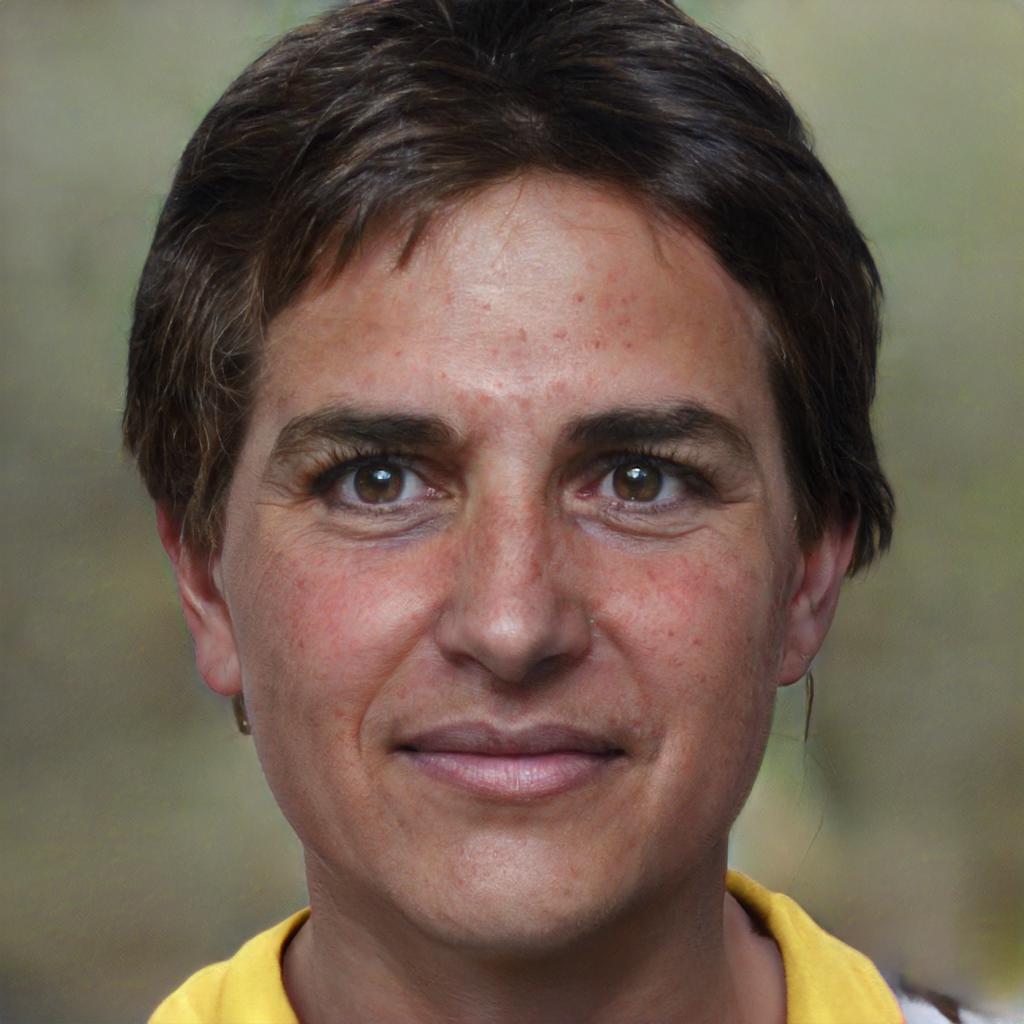}} & 
\subfloat{\includegraphics[scale=1]{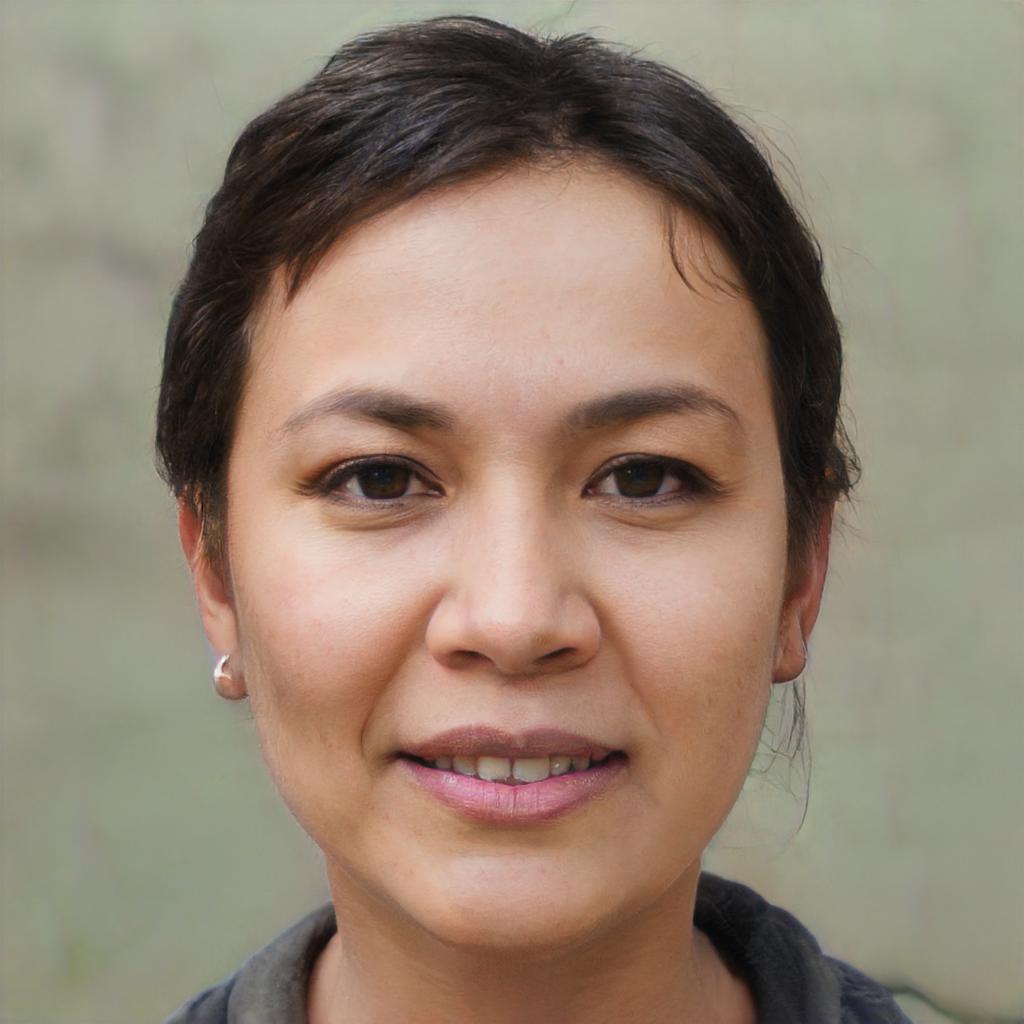}} &
\subfloat{\includegraphics[scale=1]{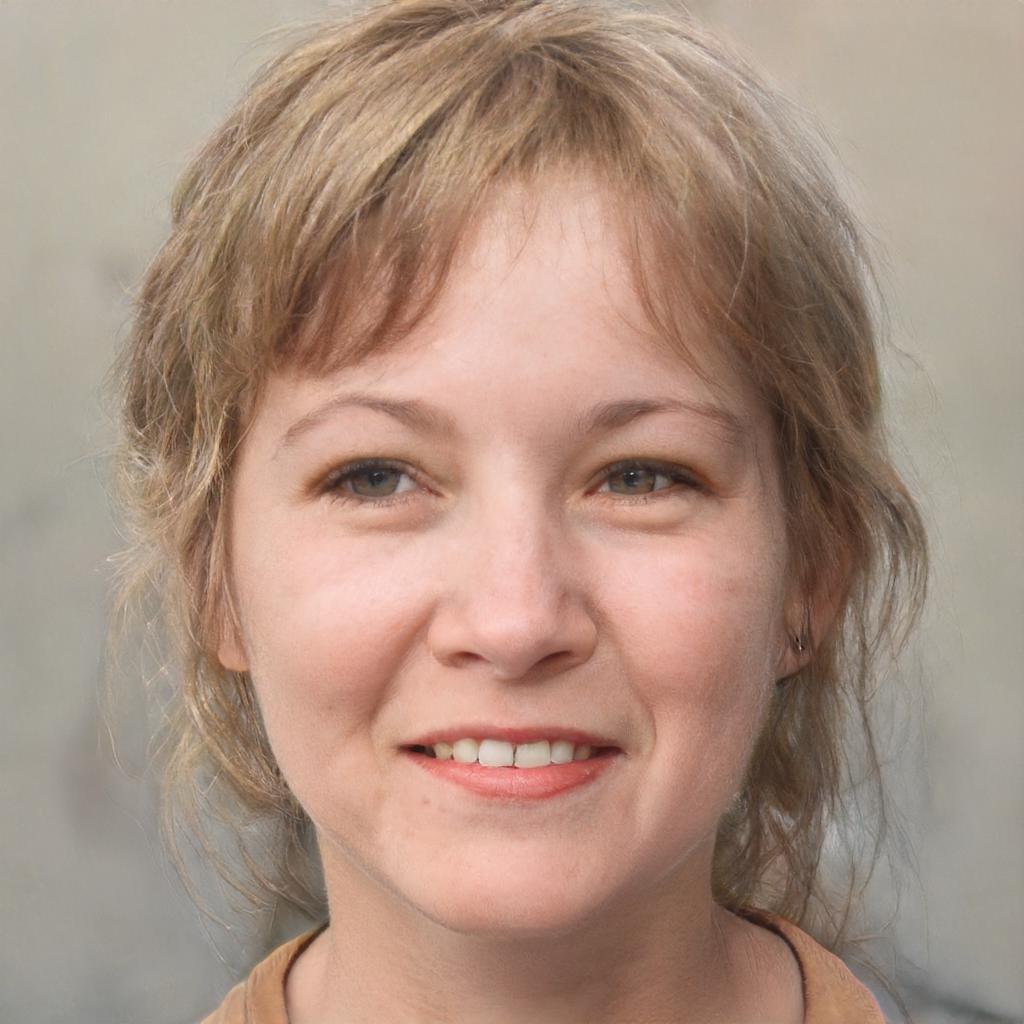}} &
\subfloat{\includegraphics[scale=1]{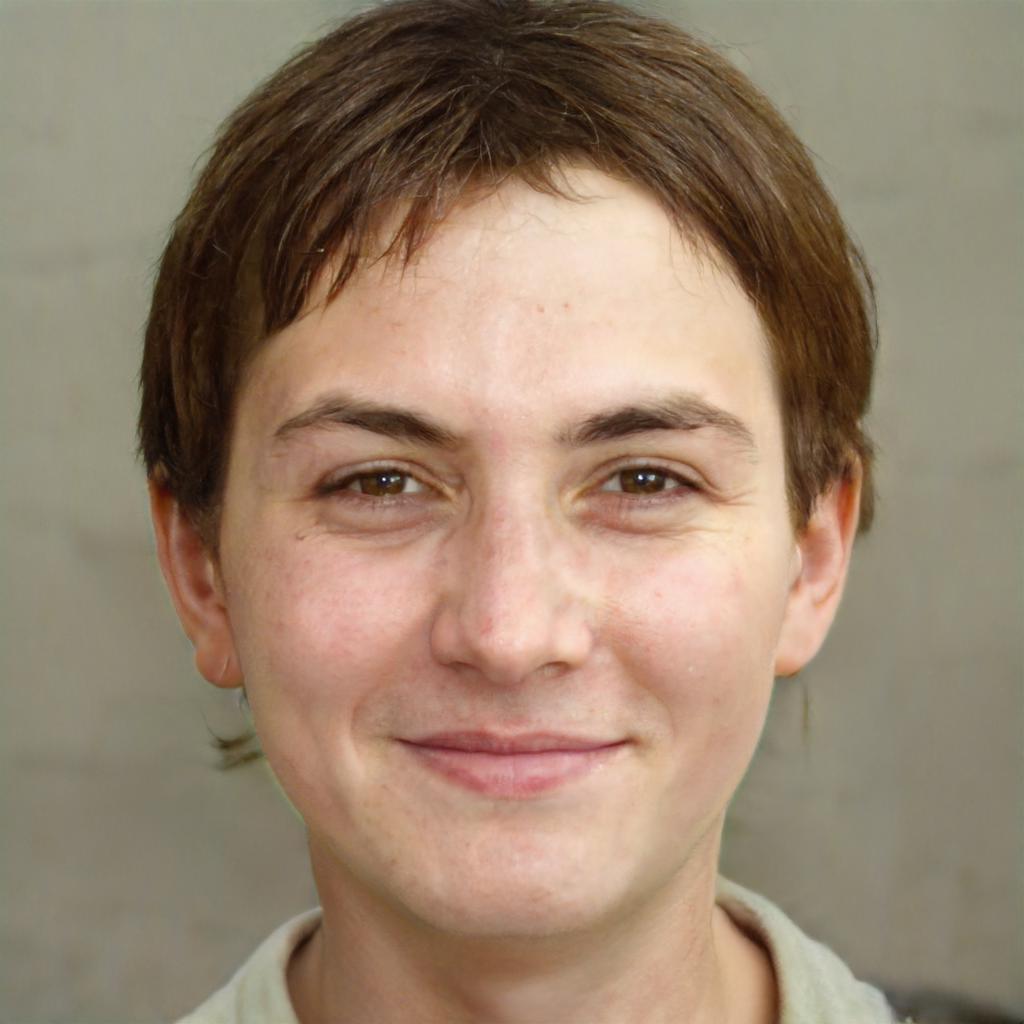}} &
\subfloat{\includegraphics[scale=1]{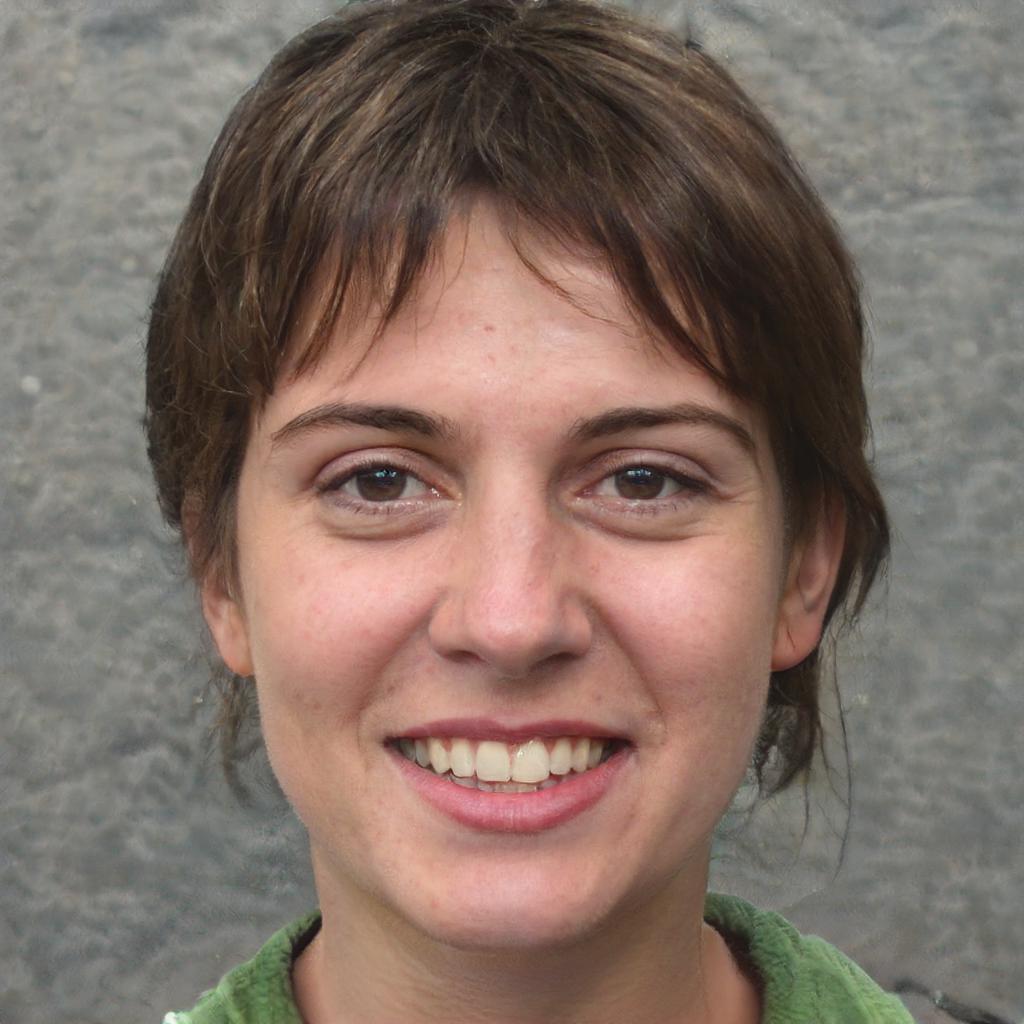}} 
\\
\subfloat{\includegraphics{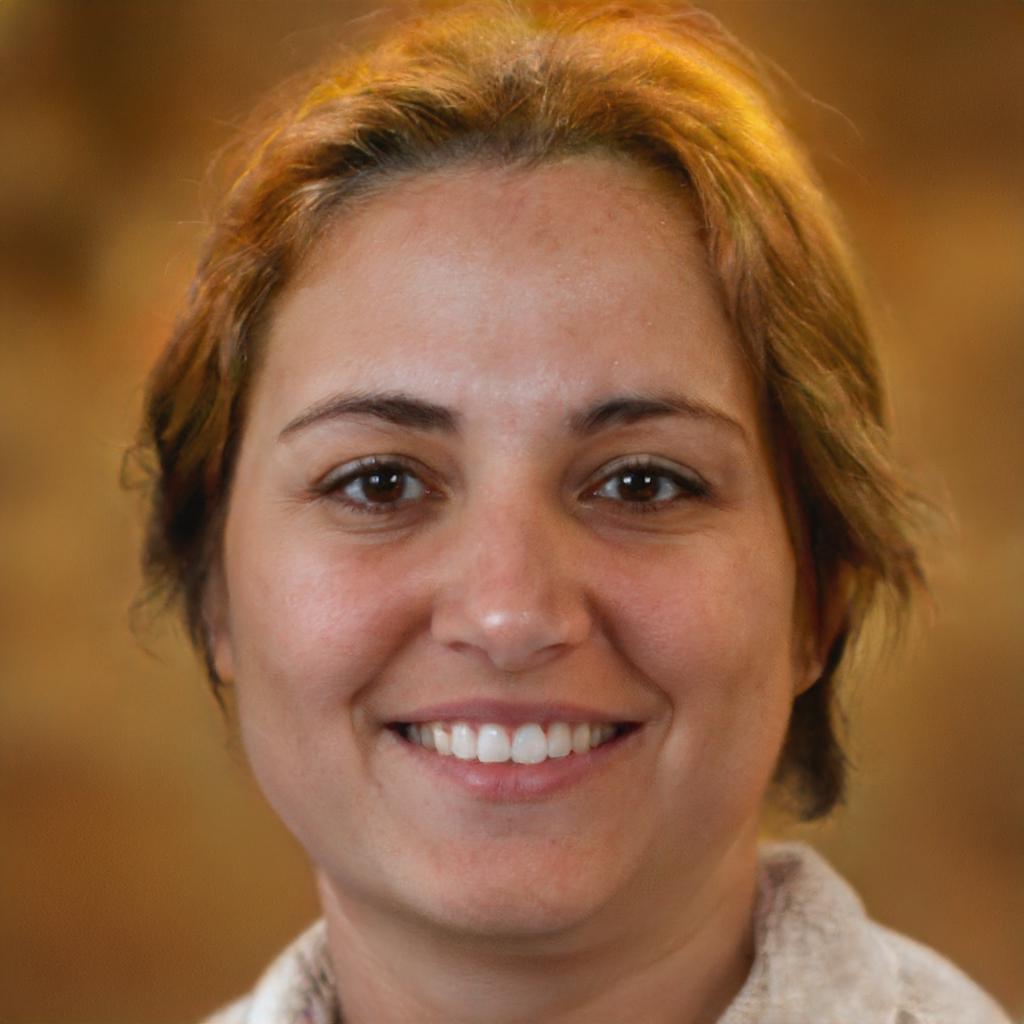}} & 
\subfloat{\includegraphics{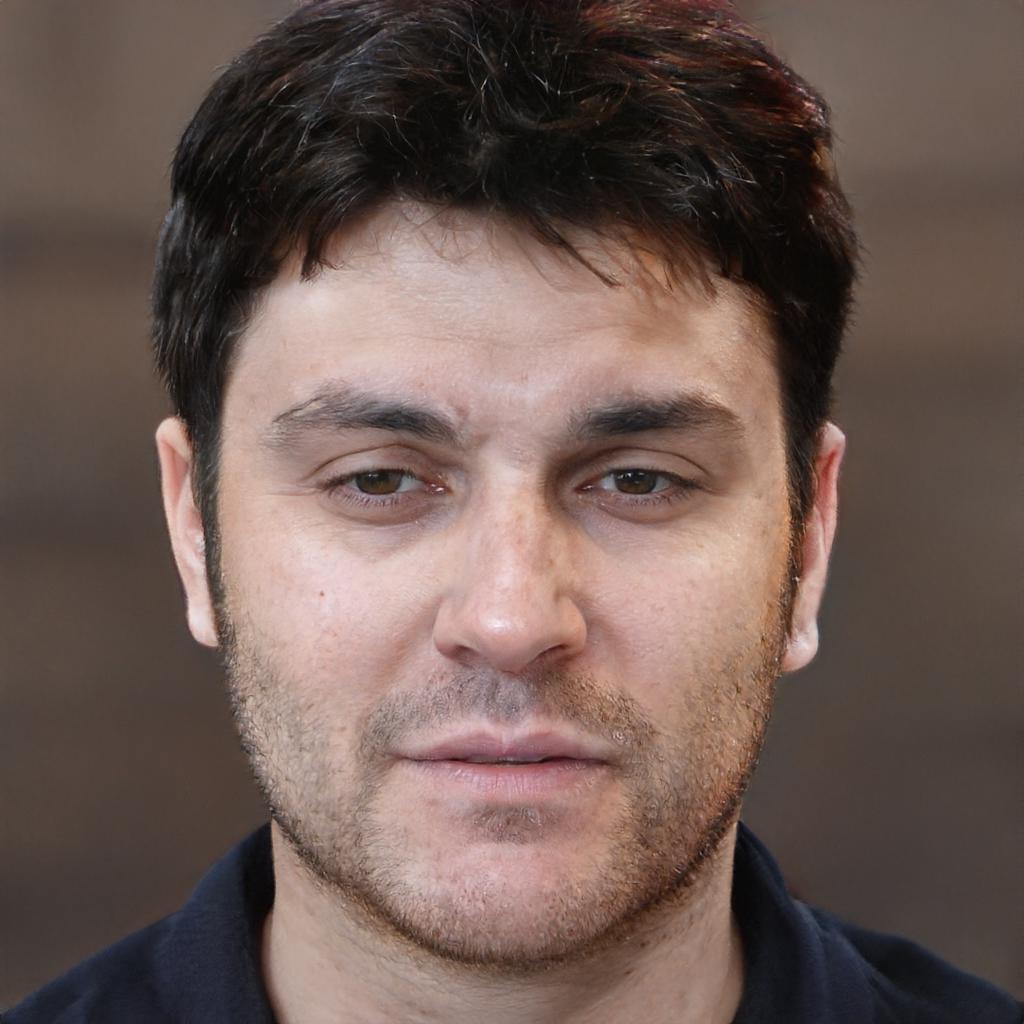}} & 
\subfloat{\includegraphics{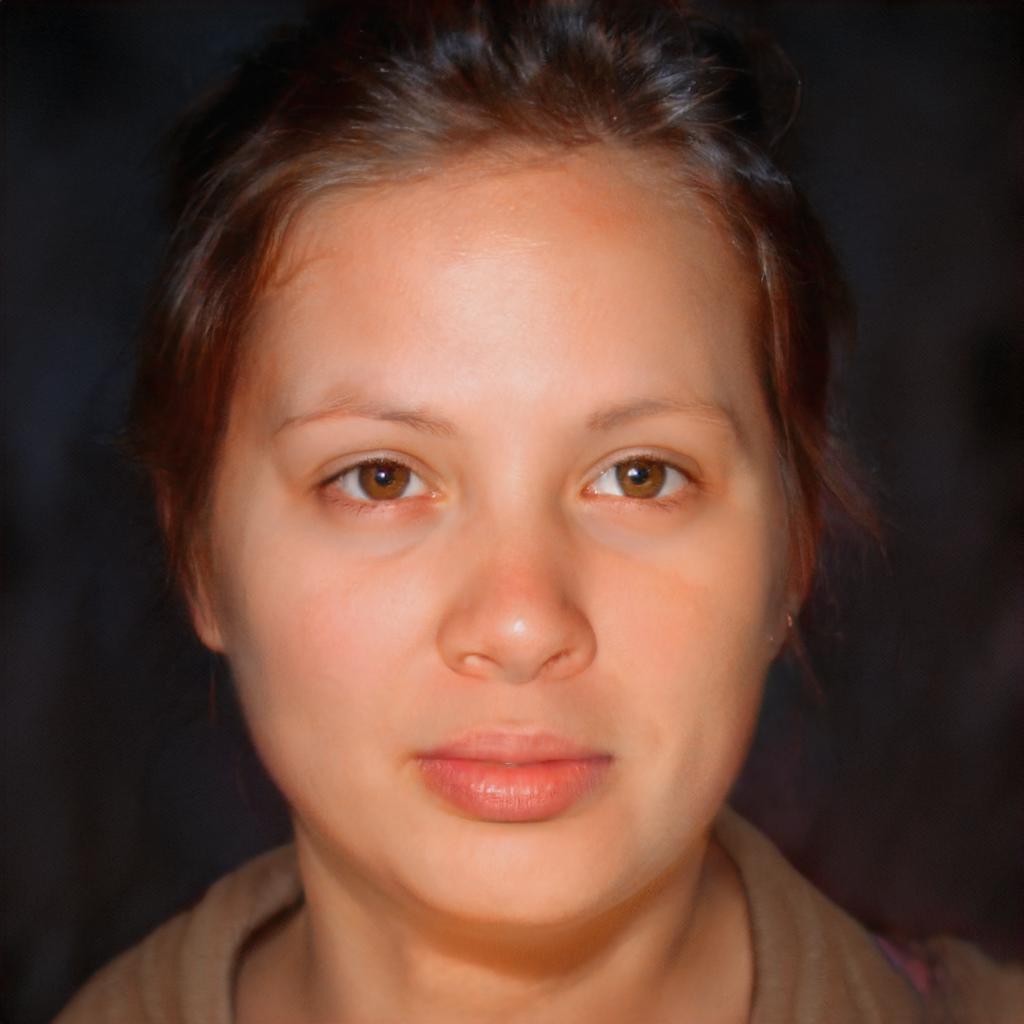}} & 
\subfloat{\includegraphics{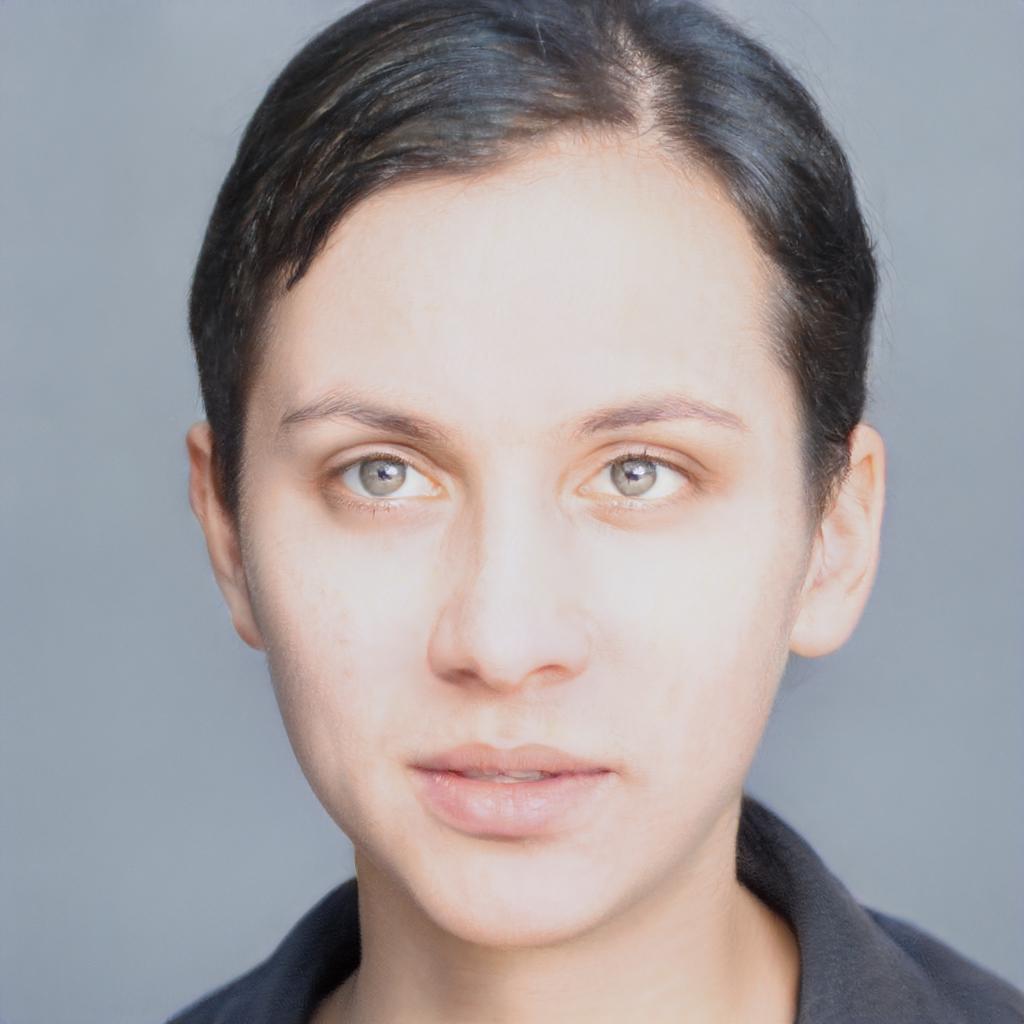}} & 
\subfloat{\includegraphics{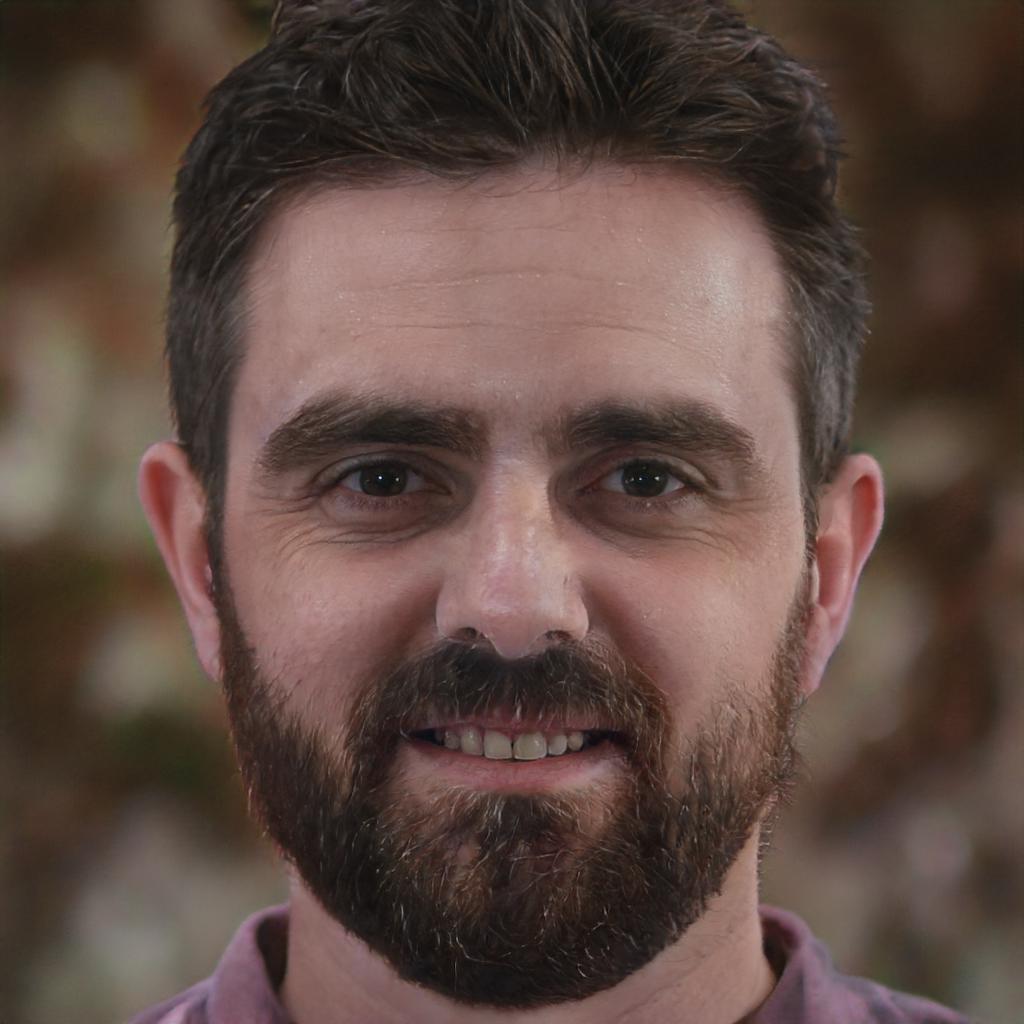}}
\end{tabular}
\end{adjustbox}
\end{center}
\caption{Images generated by a pre-trained StyleGAN-2~\cite{stylegan2} with inputs to intermediate layers sampled with our trained-score based models and the Annealed Langevin Dynamics algorithm~\cite{score_first}.}
\label{fig:generations}
\end{figure*}

\paragraph{Quantitative Results on Inverse Problems.}
We want to evaluate if our method qualitatively improves upon ILO~\cite{daras2021intermediate} which is the previous state-of-the-art method for solving inverse problems with pre-trained generators. We also compare with vanilla CSGM~\cite{bora2017compressed} which performs much worse. For a fair comparison, we choose 8 real images from FFHQ to tune the hyperparameters for each method at each measurement level, and then measure performance with respect to the ground truth on $30$ FFHQ test set images (never seen by the score-based model). For ILO, we also tried the default parameters (300, 300, 300, 100 steps) reported in~\citet{daras2021intermediate}. Finally, to make sure that the benefit of our method comes indeed from the prior and not from optimizing without the ILO sparsity constraints, we also test ILO without any constraint on the optimization space. In the Figures, for the ILO we report the minimum of ILO with tuned parameters, ILO with default parameters (from the paper) and ILO without any regularization. For the denoising experiments, we tried ILO with and without dynamic addition of noise (Stochastic Noise Addition) and we plotted the best score.

Figure \ref{plots} shows MSE and Perceptual distance between the ground truth image and the reconstructions of ILO, CSGM and SGILO (ours) as we vary the difficulty of the task (horizontal axis). The plots show results for denoising, compressed sensing  and inpainting. As shown, in more challenging regimes, our method outperforms ILO and CSGM. When the task is not very challenging, e.g. denoising when the standard deviation of the noise is $\sigma=0.1$, the prior becomes less important and SGILO performs on par with ILO. As expected, the contribution of the prior is significant when less information is available.

\begin{figure}[!ht]
\captionsetup{justification=centering}
\begin{center}
\begin{adjustbox}{width=0.4\textwidth, center}
\begin{tabular}{cccc} 
\begin{minipage}{0.2\linewidth}
\includegraphics[width=\linewidth]{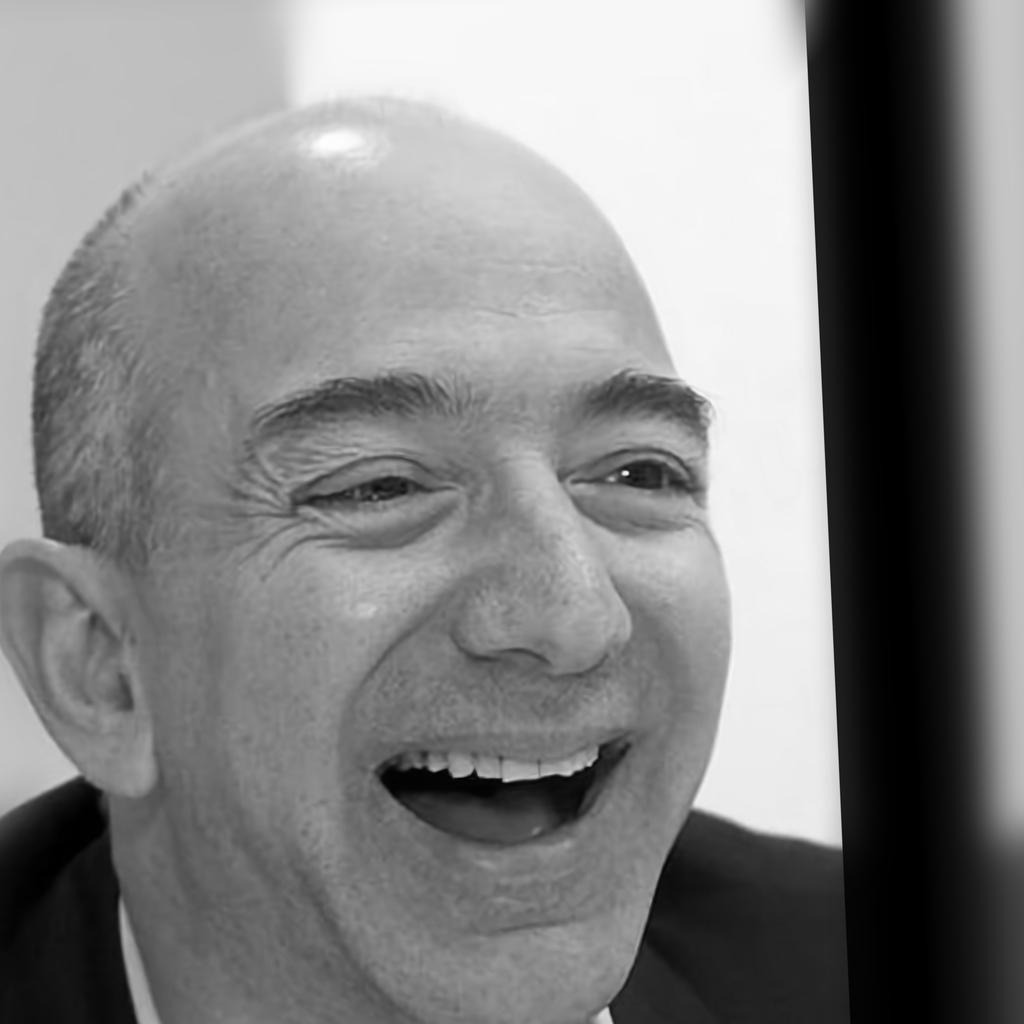} 
\end{minipage} &
\begin{minipage}{0.2\linewidth}
\includegraphics[width=\linewidth]{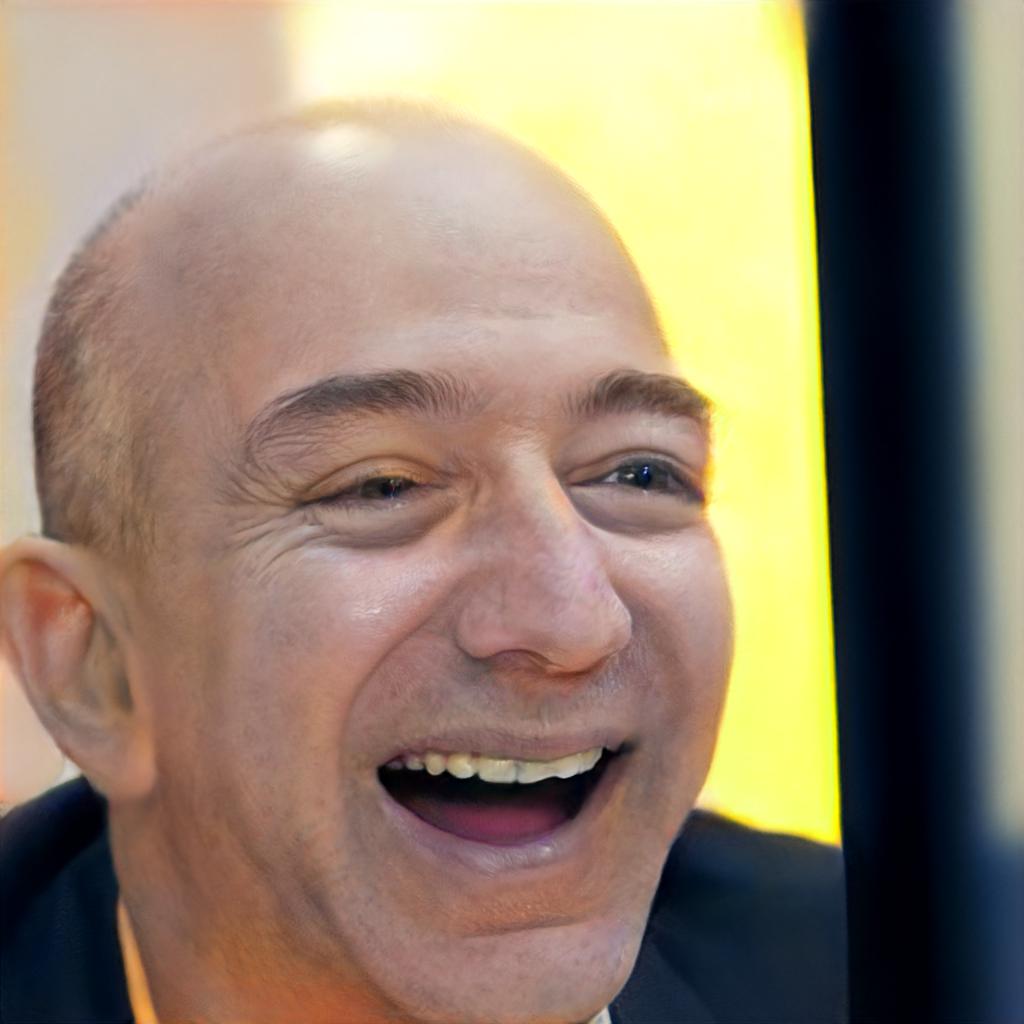} 
\end{minipage} &
\begin{minipage}{0.2\linewidth}
\includegraphics[width=\linewidth]{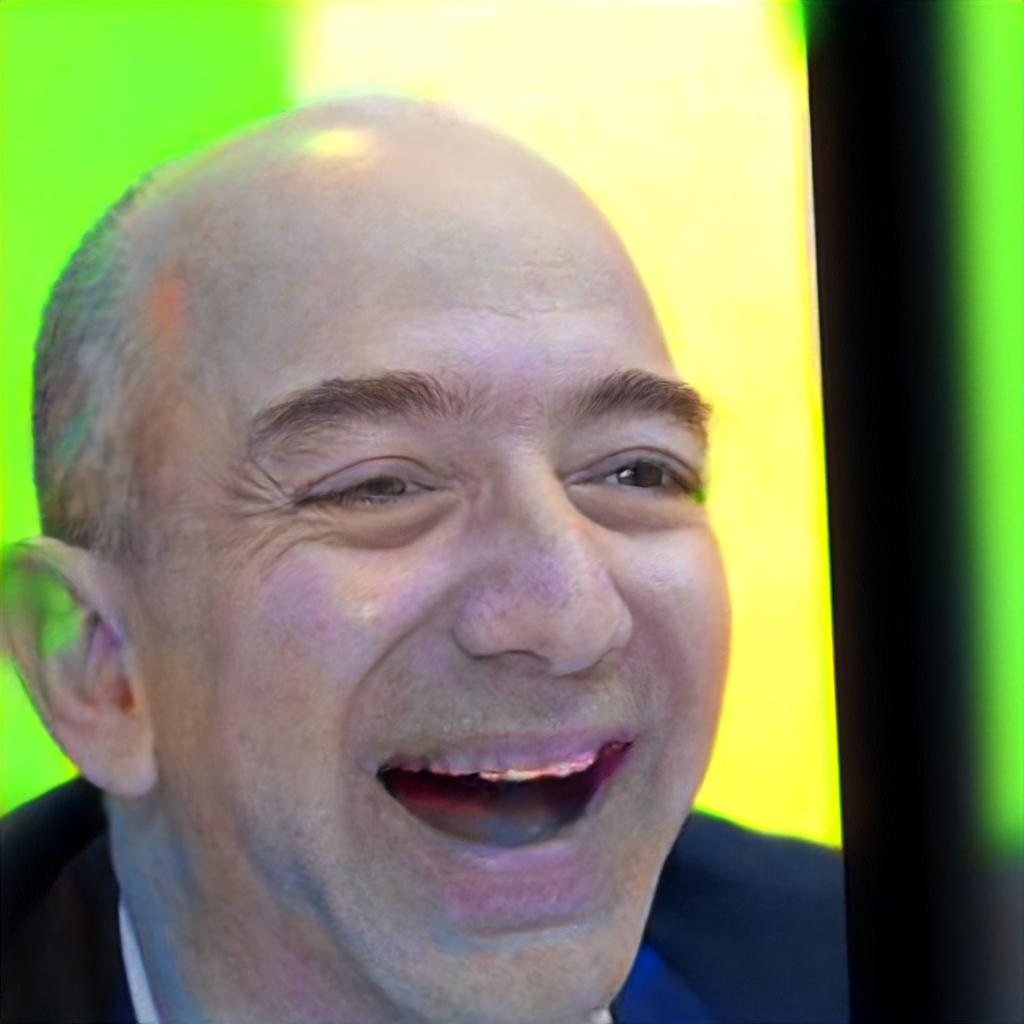} 
\end{minipage} &
\begin{minipage}{0.2\linewidth}
\includegraphics[width=\linewidth]{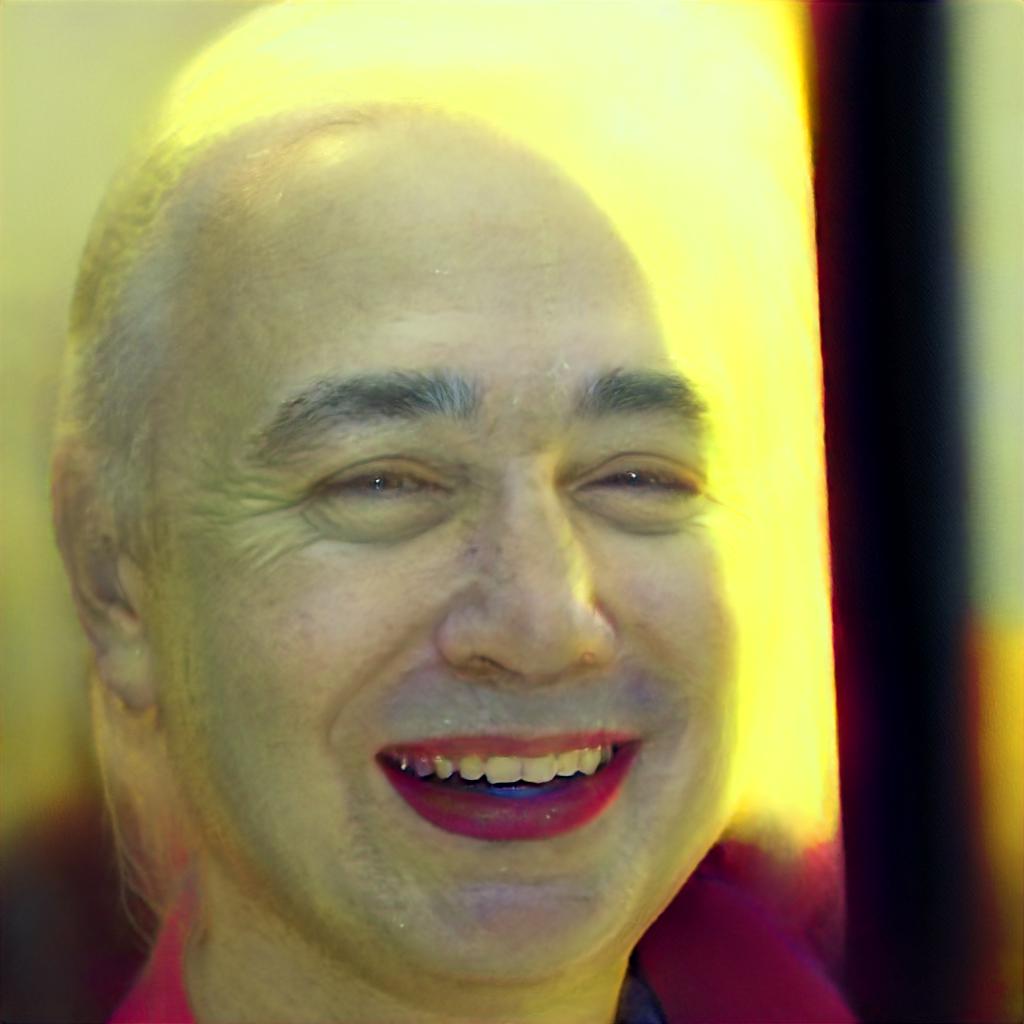} 
\end{minipage} \\
\begin{minipage}{0.2\linewidth}
\includegraphics[width=\linewidth]{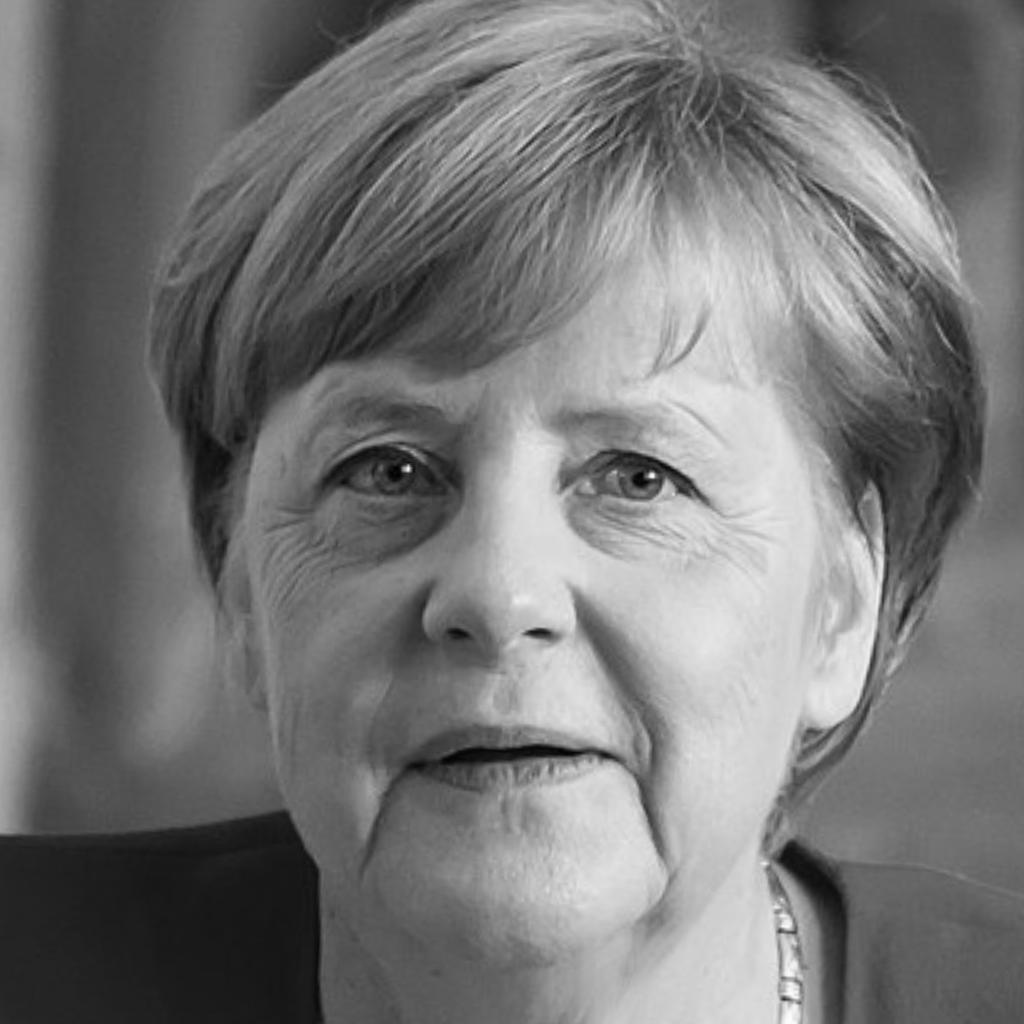} 
\end{minipage} &
\begin{minipage}{0.2\linewidth}
\includegraphics[width=\linewidth]{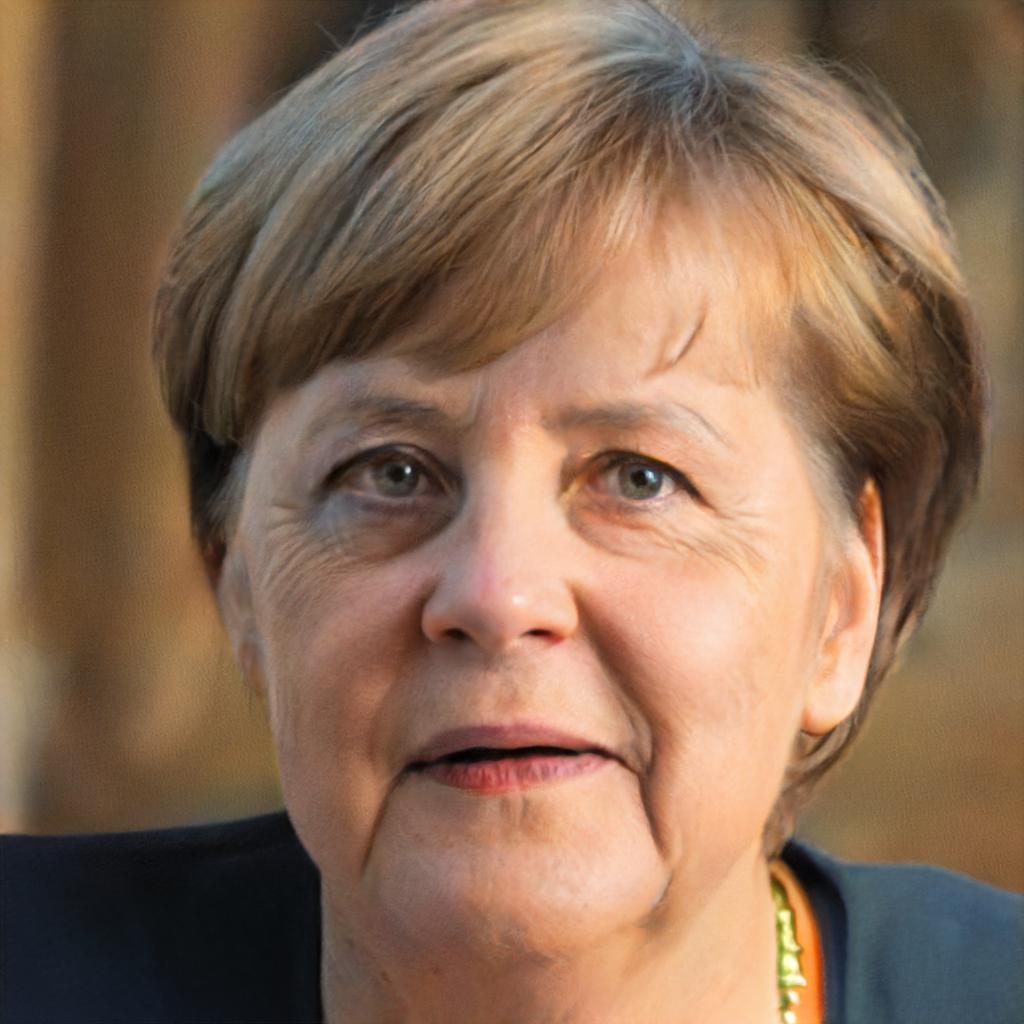} 
\end{minipage} &
\begin{minipage}{0.2\linewidth}
\includegraphics[width=\linewidth]{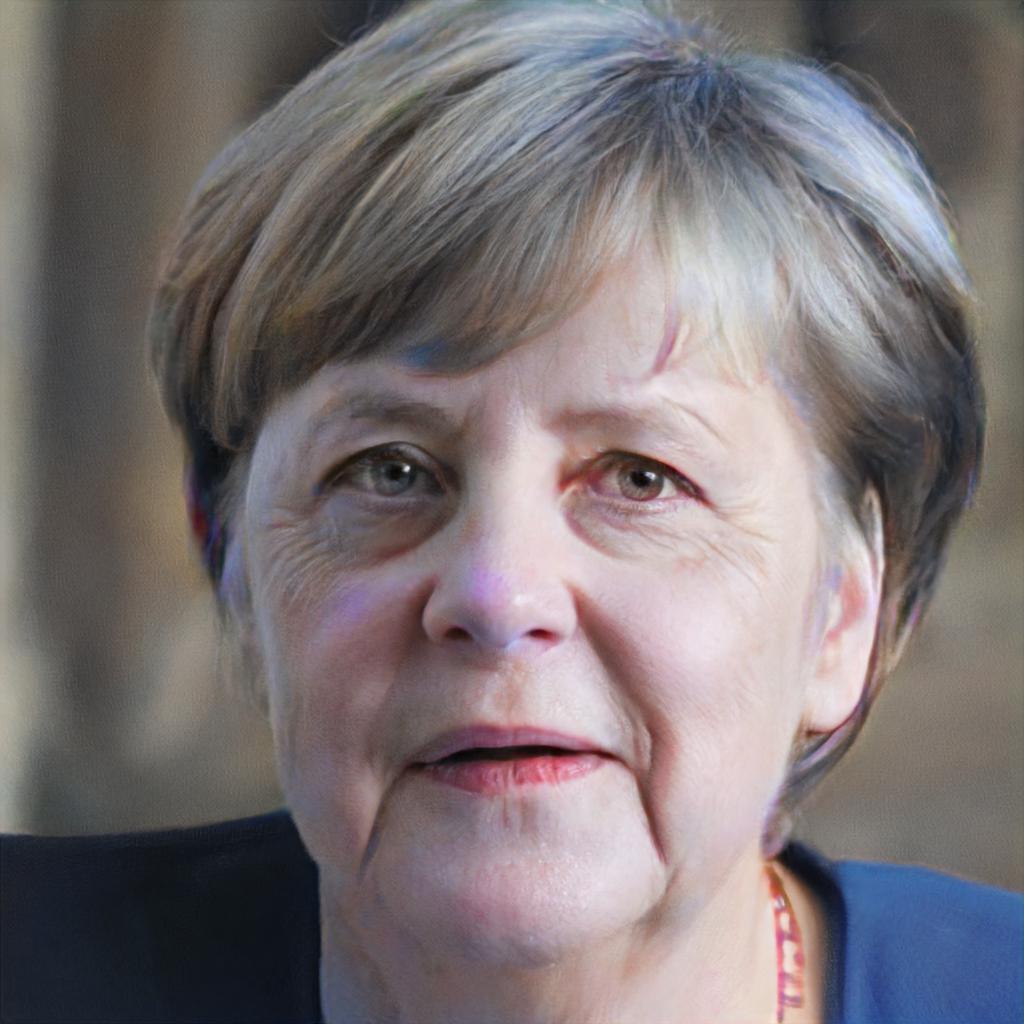} 
\end{minipage} &
\begin{minipage}{0.2\linewidth}
\includegraphics[width=\linewidth]{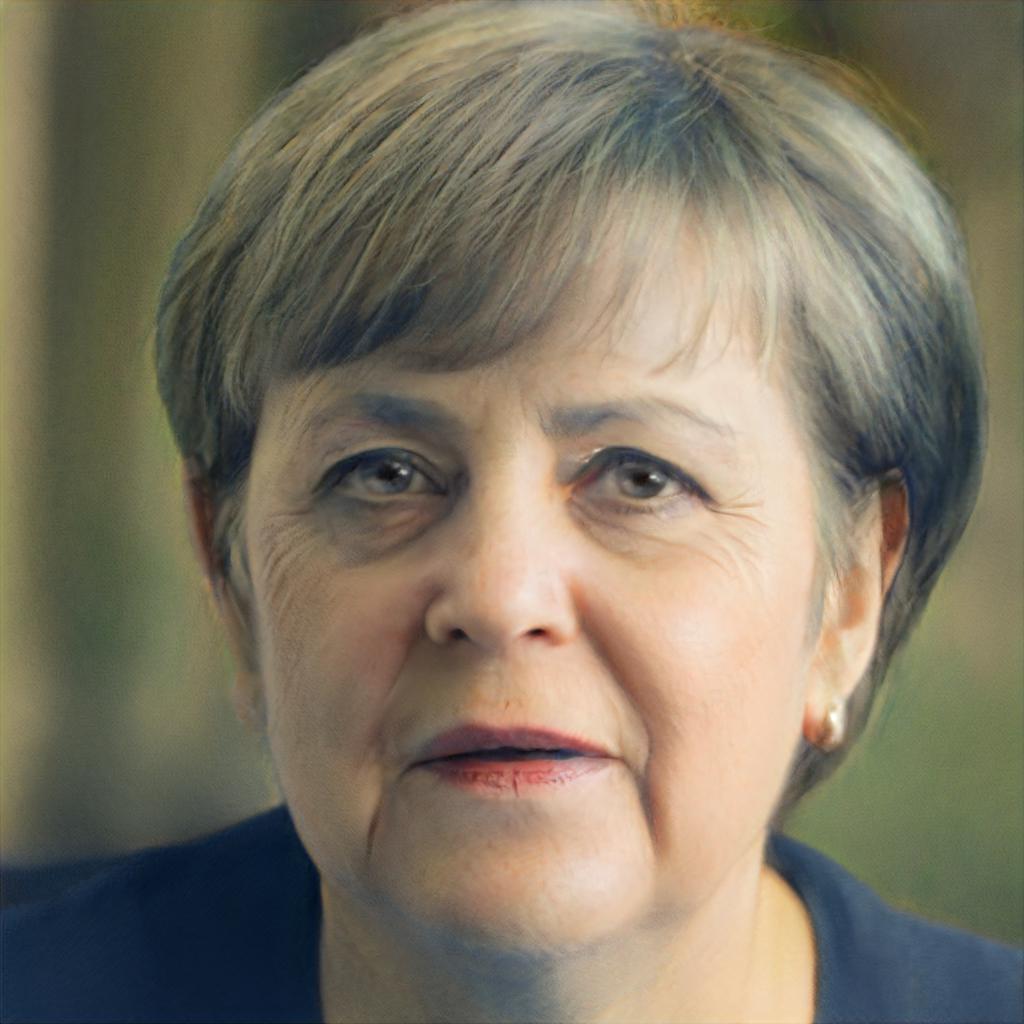} 
\end{minipage} \\
\begin{minipage}{0.2\linewidth}
\includegraphics[width=\linewidth]{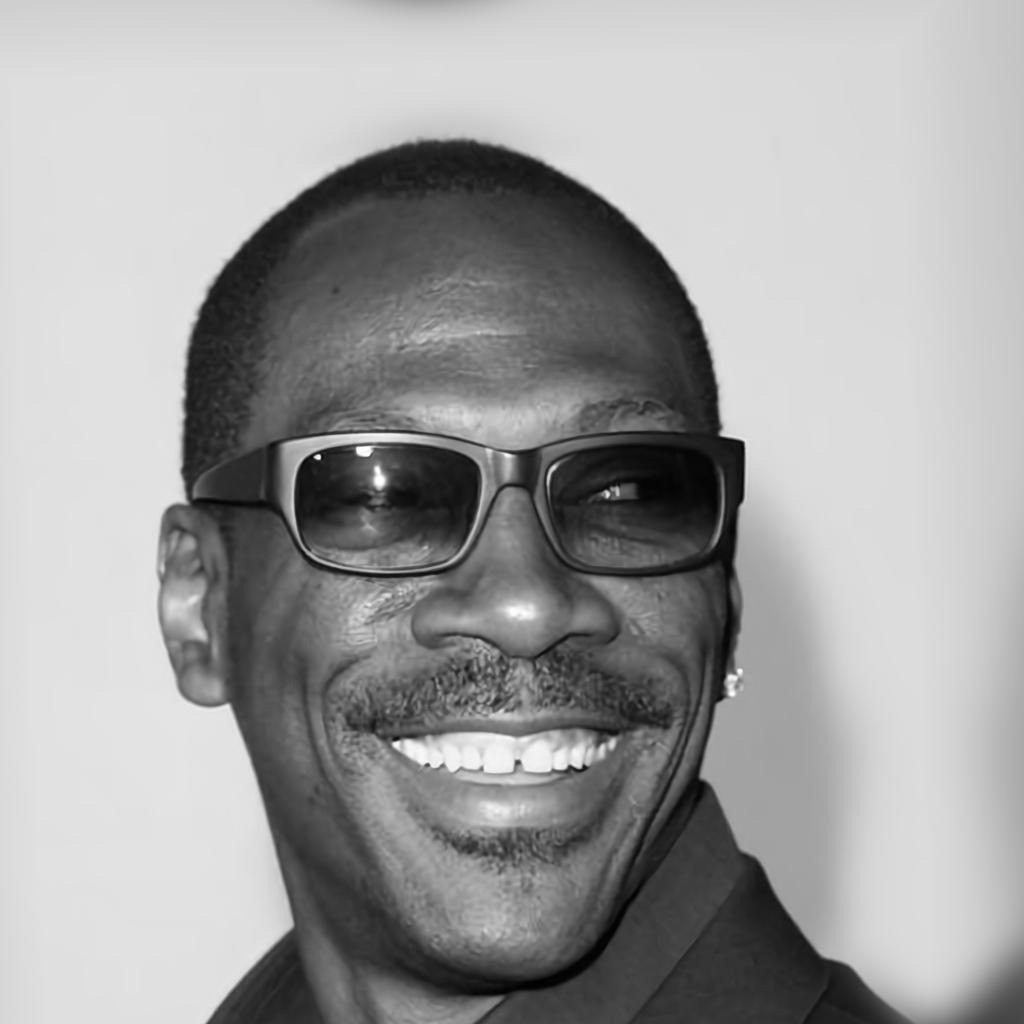} 
\end{minipage} &
\begin{minipage}{0.2\linewidth}
\includegraphics[width=\linewidth]{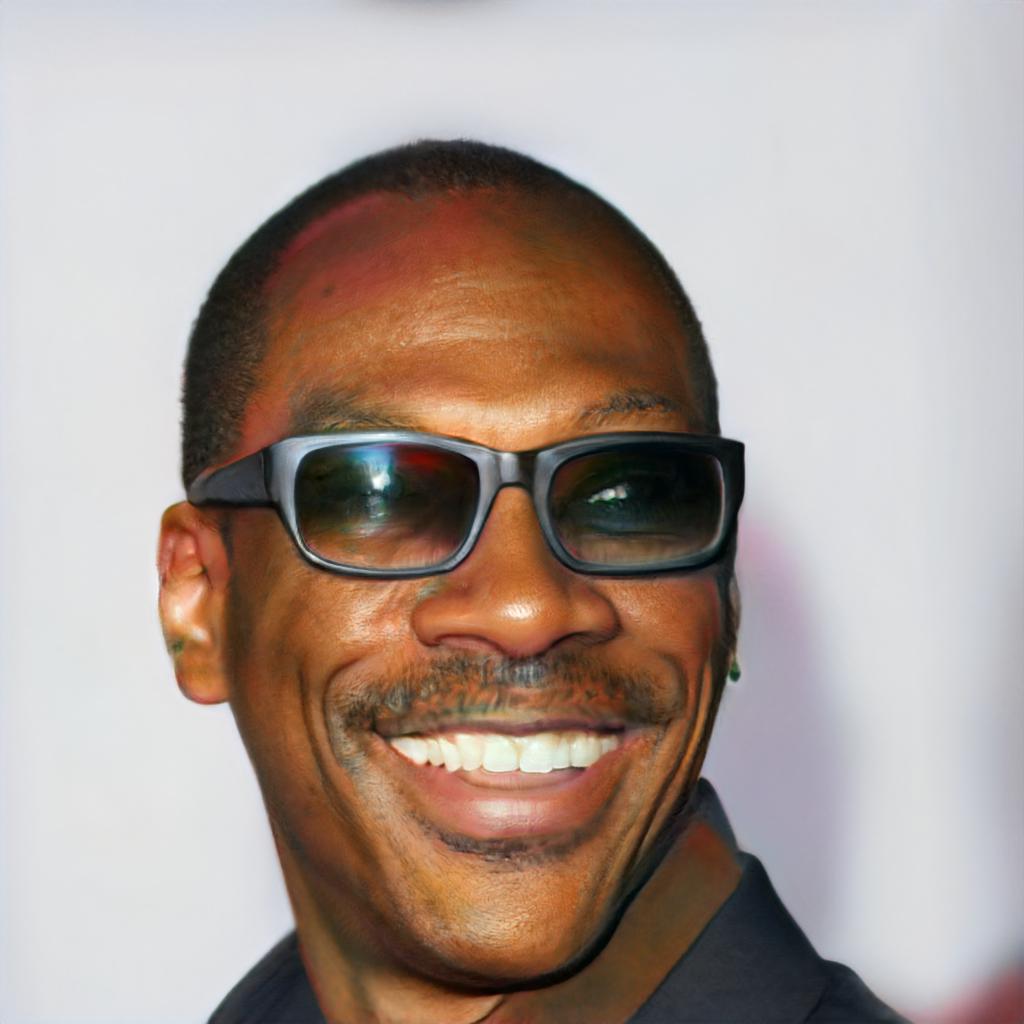} 
\end{minipage} &
\begin{minipage}{0.2\linewidth}
\includegraphics[width=\linewidth]{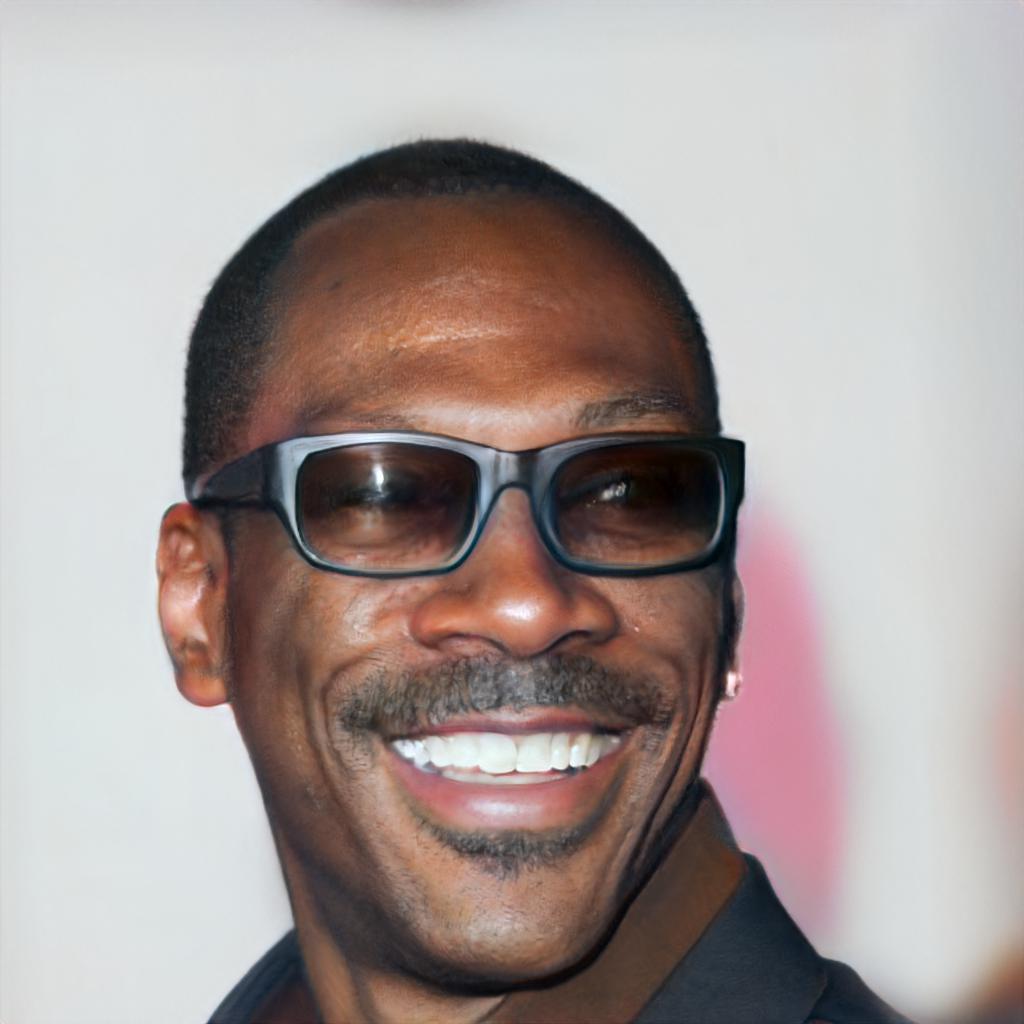} 
\end{minipage} &
\begin{minipage}{0.2\linewidth}
\includegraphics[width=\linewidth]{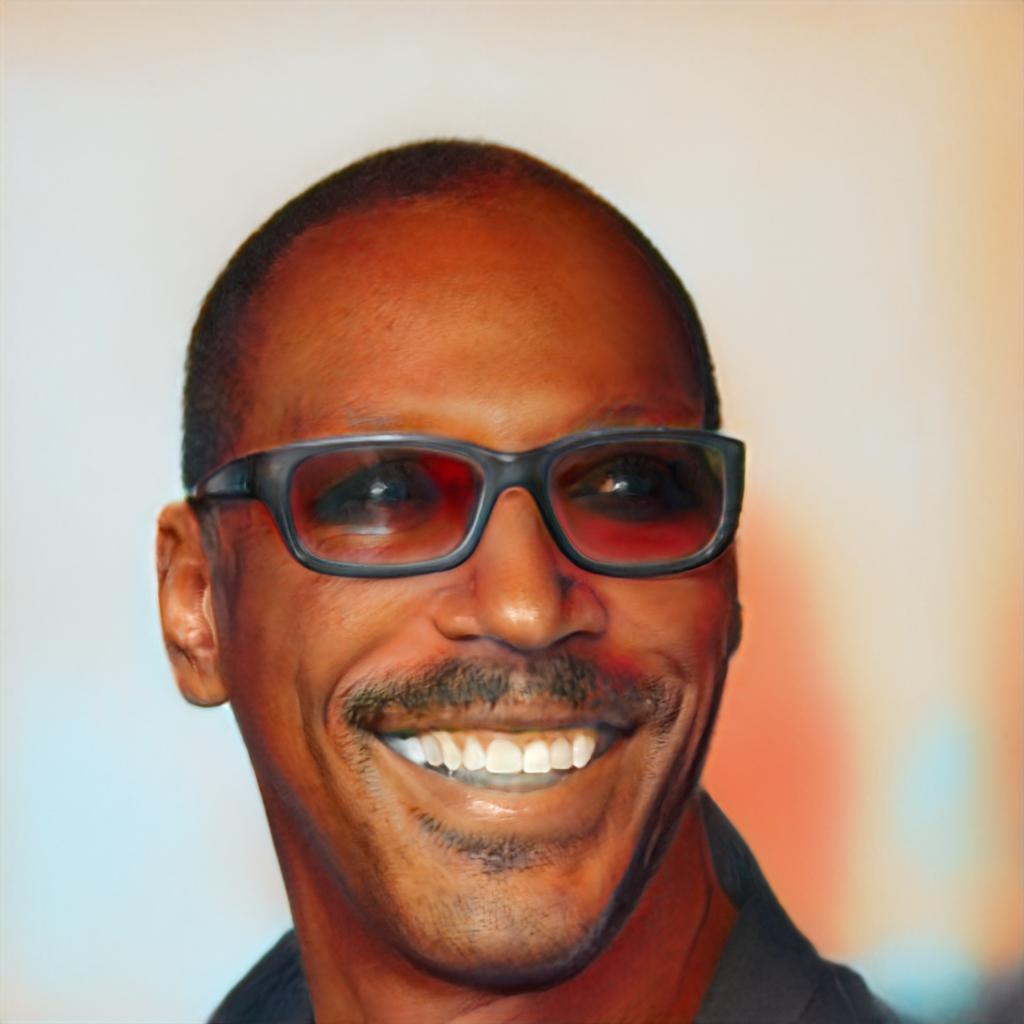} 
\end{minipage} \\
\begin{minipage}{0.2\linewidth}
\includegraphics[width=\linewidth]{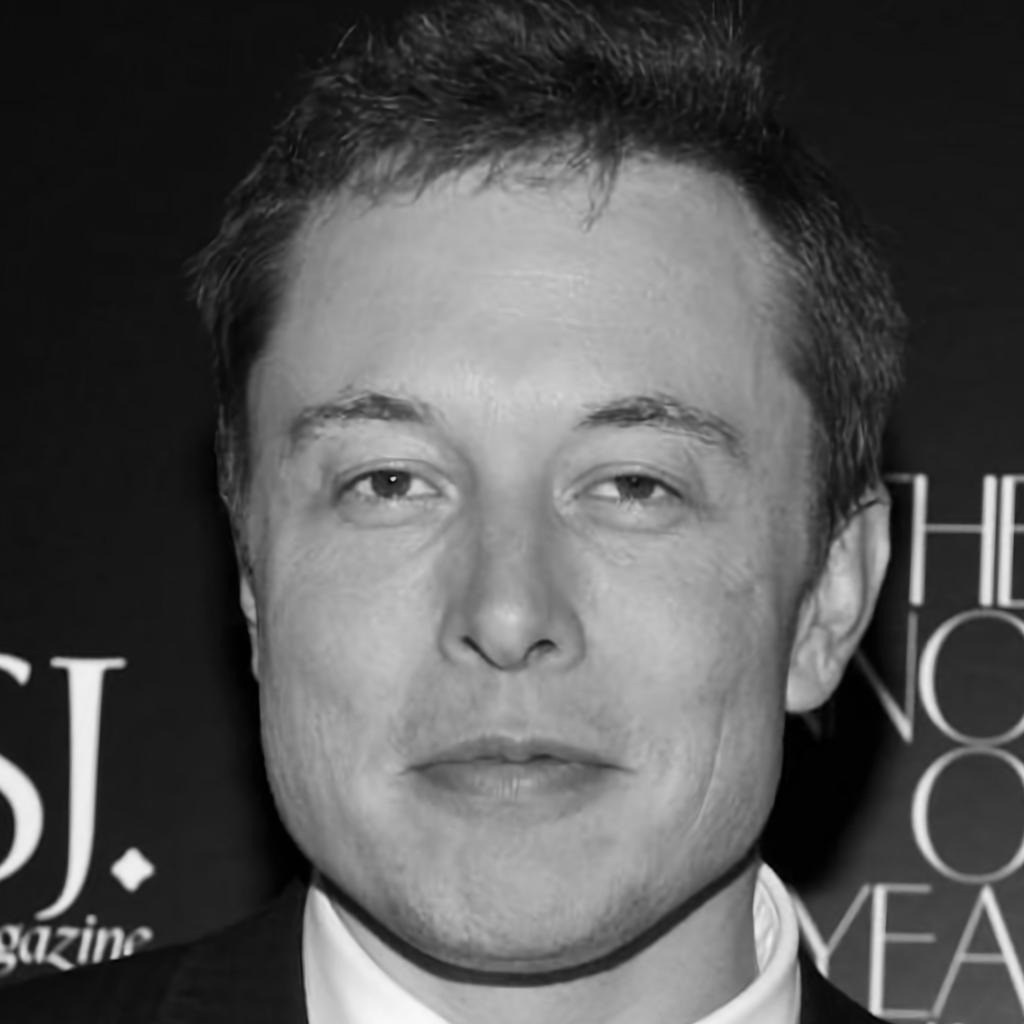} 
\end{minipage} &
\begin{minipage}{0.2\linewidth}
\includegraphics[width=\linewidth]{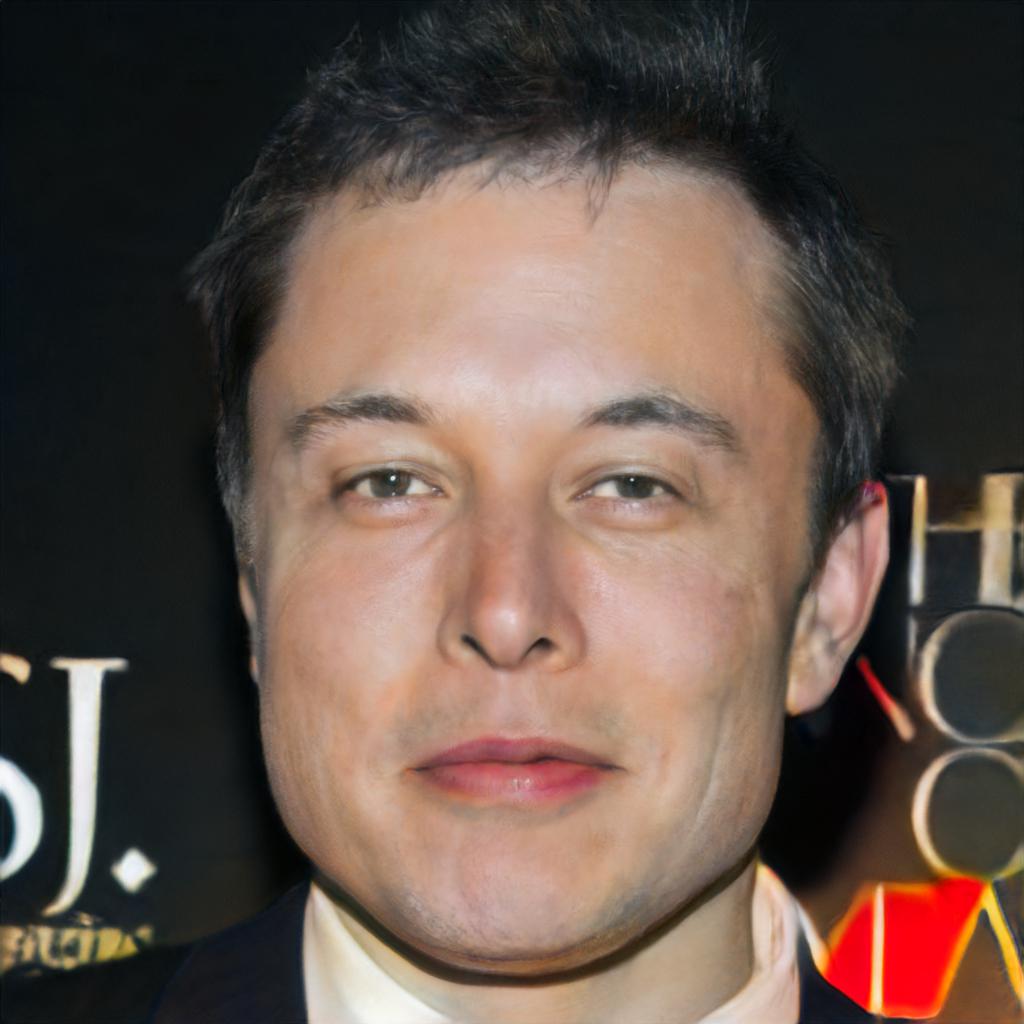} 
\end{minipage} &
\begin{minipage}{0.2\linewidth}
\includegraphics[width=\linewidth]{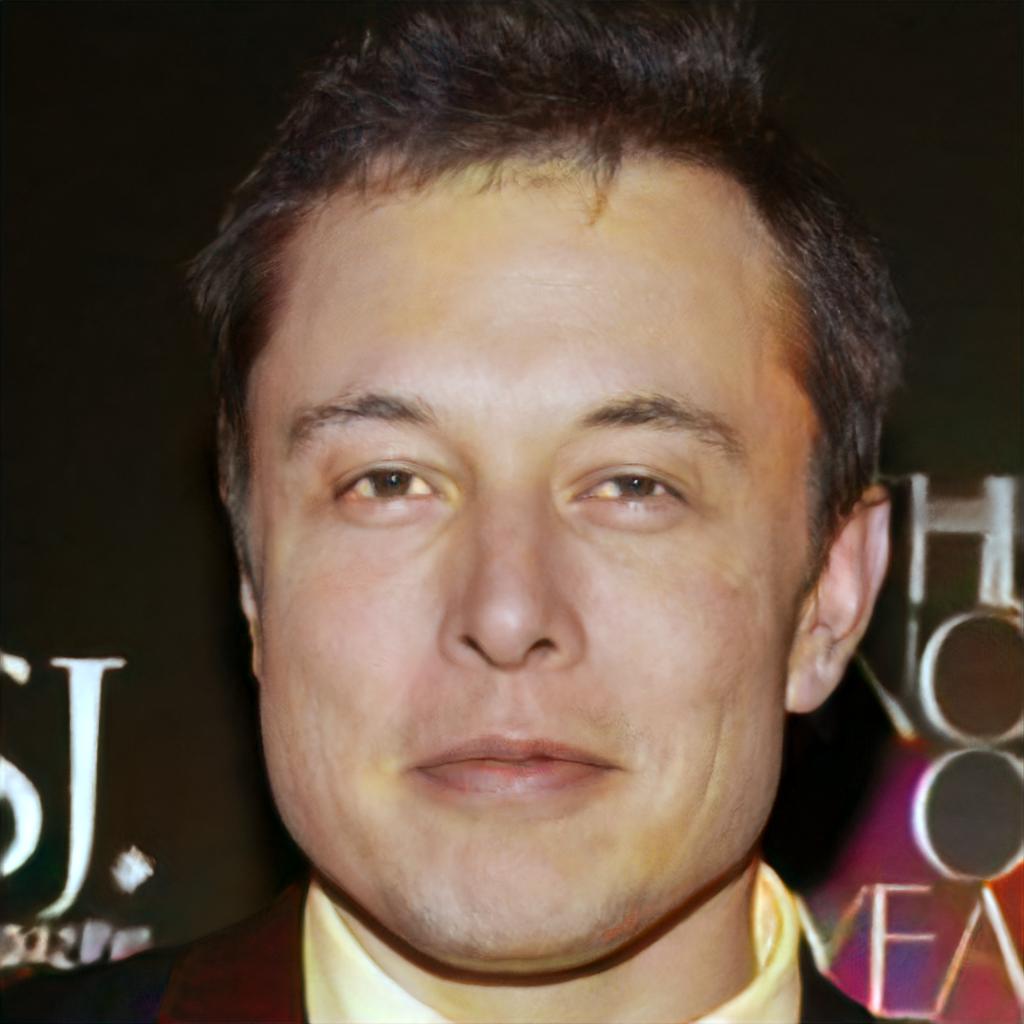} 
\end{minipage} &
\begin{minipage}{0.2\linewidth}
\includegraphics[width=\linewidth]{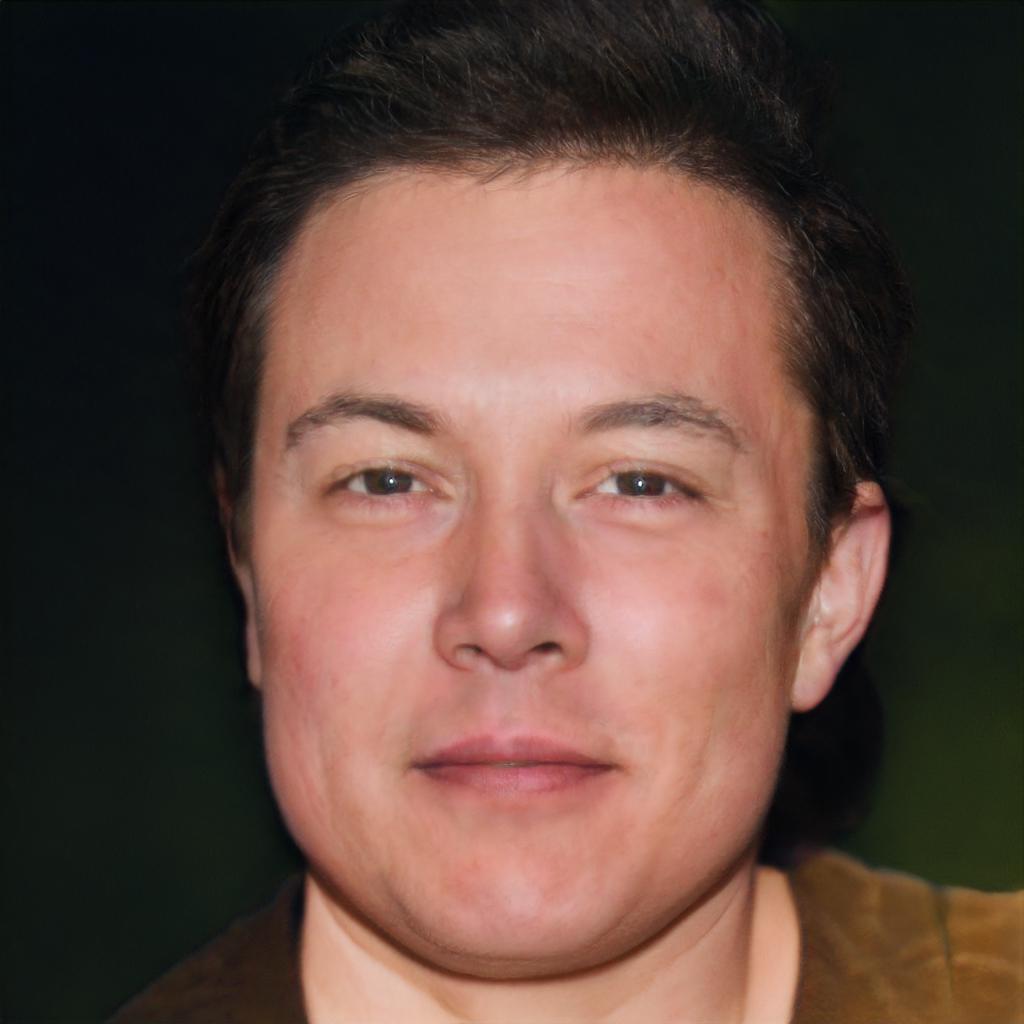} 
\end{minipage} \\
\begin{minipage}{0.2\linewidth}
\includegraphics[width=\linewidth]{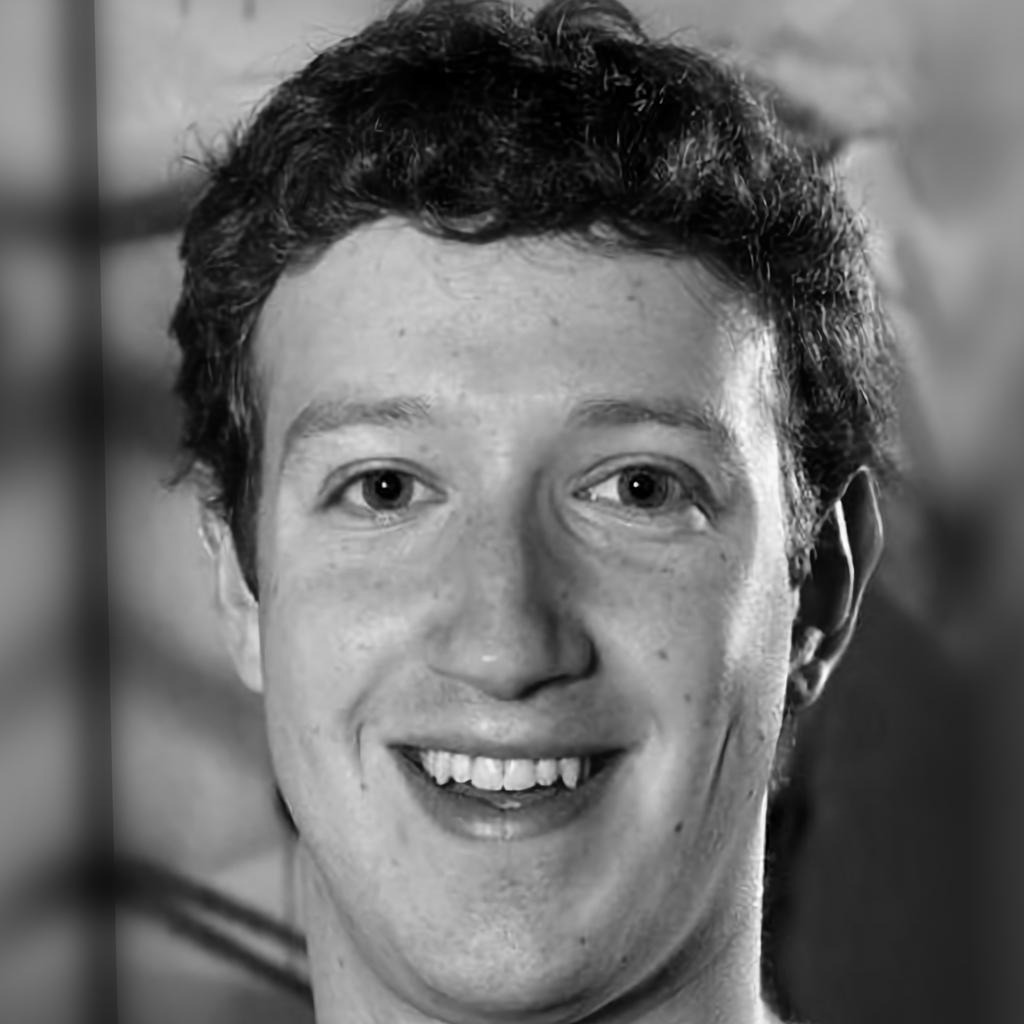} 
\caption*{Input}
\end{minipage} &
\begin{minipage}{0.2\linewidth}
\includegraphics[width=\linewidth]{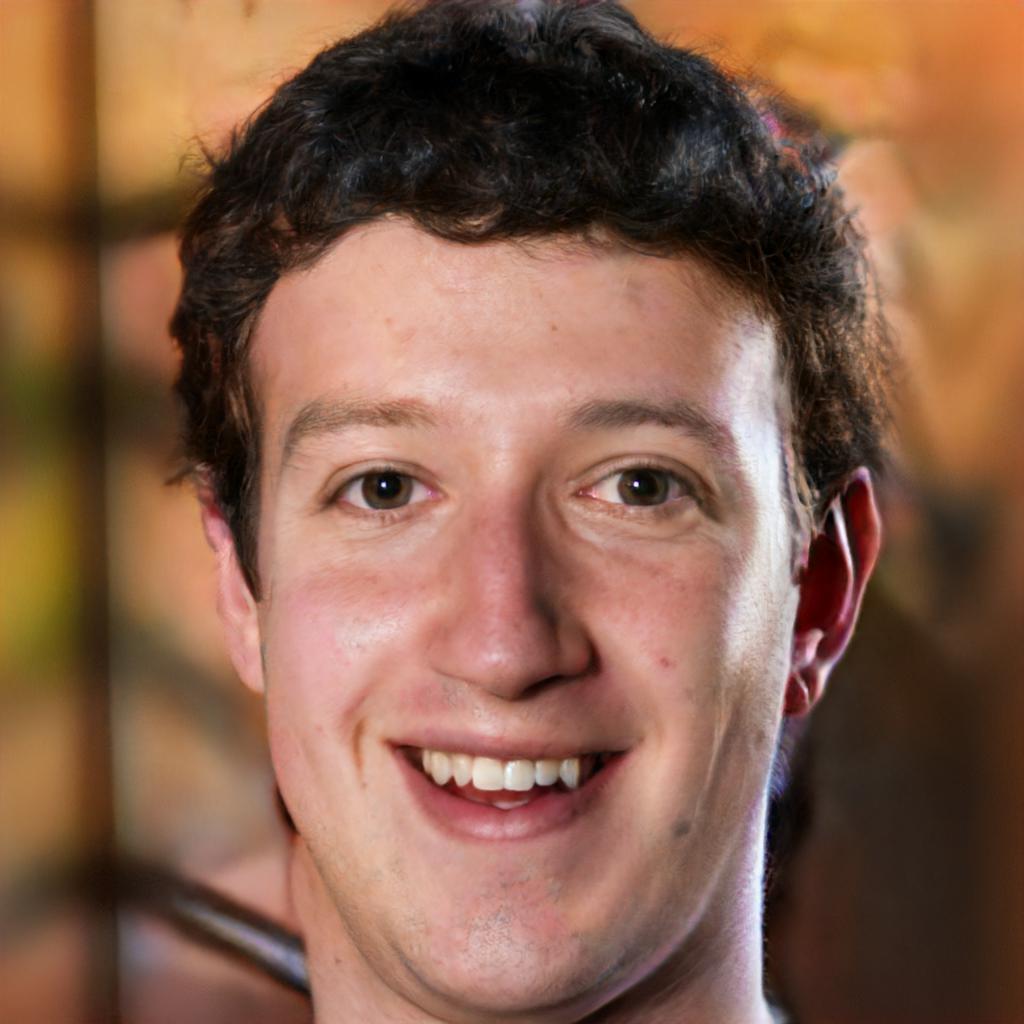} 
\caption*{SGILO}
\end{minipage} &
\begin{minipage}{0.2\linewidth}
\includegraphics[width=\linewidth]{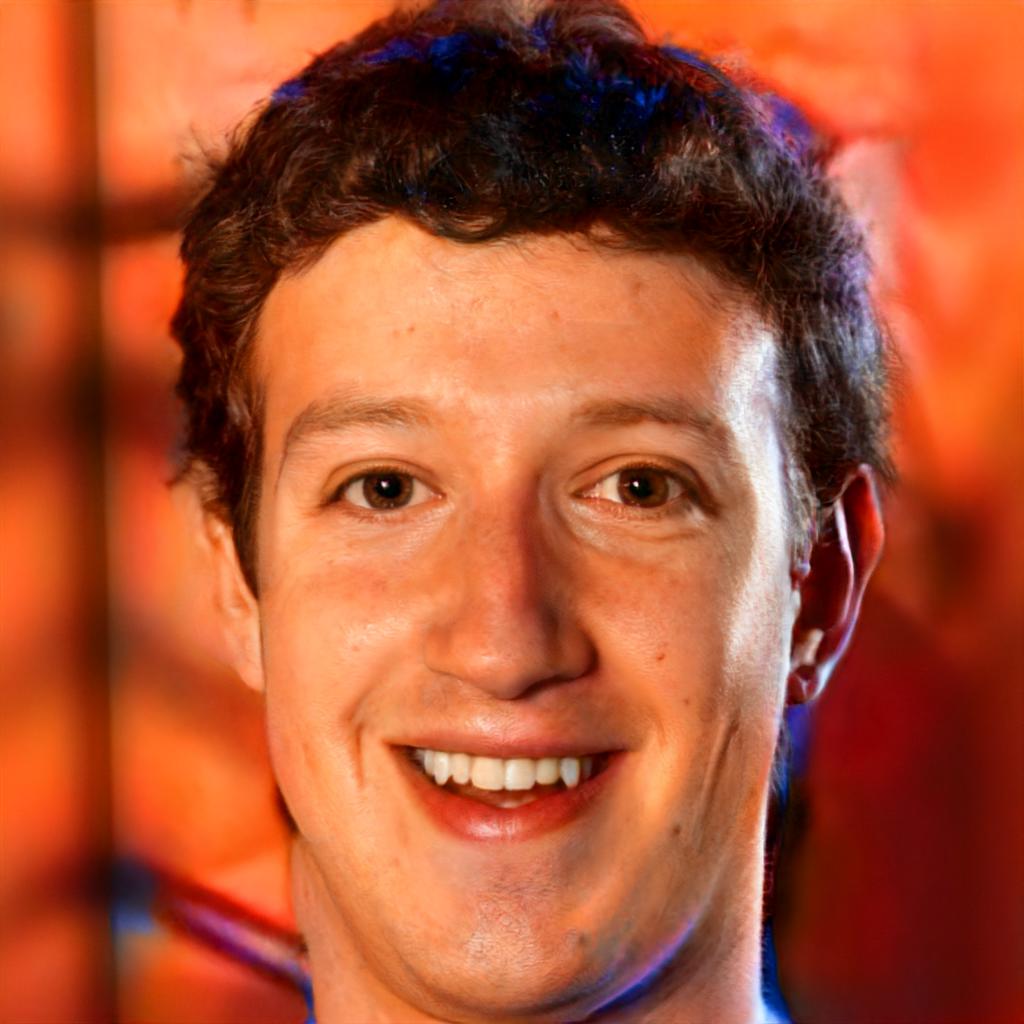} 
\caption*{ILO}
\end{minipage} &
\begin{minipage}{0.2\linewidth}
\includegraphics[width=\linewidth]{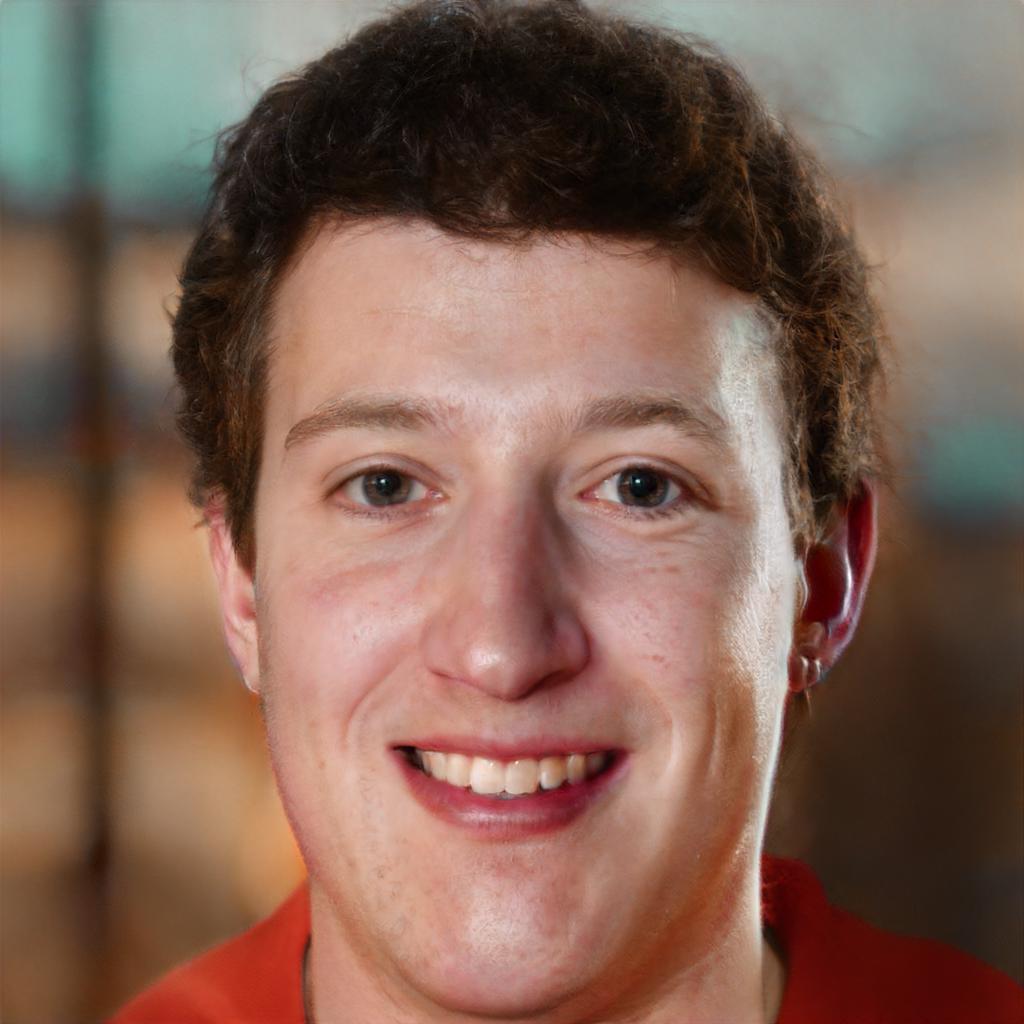} 
\caption*{CSGM}
\end{minipage}
\end{tabular}
\end{adjustbox}
\caption{\small{Results on colorization. ILO introduces artifacts, e.g. Row 1, column 3. Those artifacts are mostly corrected by SGILO, that displays more natural colors than prior work.}}
\label{fig:colorization}
\end{center}
\end{figure}

\paragraph{Out of distribution projections.} This experiment demonstrates that following the learned vector field of the log-likelihood in the latent space leads to more natural images. Specifically, we start with an out-of-distribution image and we invert it with ILO. For the purposes of this experiment, we intentionally stop ILO early, to arrive at a solution that has visual artifacts. We now use solely the score-based prior to see if we can correct these artifacts. We obtain the early stopped ILO solution $z^*_0 \in \R^p$ (where $p$ is the dimension of the intermediate layer, optimizing over the third layer of StyleGAN-2), and use the Forward SDE to sample from $p(z_t|z^*_0)$. This corresponds to sampling from a Gaussian centered in $z_0^*$ with variance that grows with $t$. Then, we use Annealed Langevin Dynamics with the learned Score-Based network to sample from $p(z_0|z_t)$. The choice of $t$ is affecting how close will be the final sample to $z_0^*$. Since we started with an unnatural latent, we expect that $t$ is controlling the trade-off between matching the measurements and increasing the naturalness of the image. Results are shown in Figure \ref{fig:natural}.

\begin{figure*}[!ht]
\captionsetup[subfigure]{labelformat=empty}
\captionsetup{justification=centering}
\begin{center}
\begin{adjustbox}{width=\textwidth, center}
\begin{tabular}{cccc} 
\subfloat{\includegraphics[scale=1]{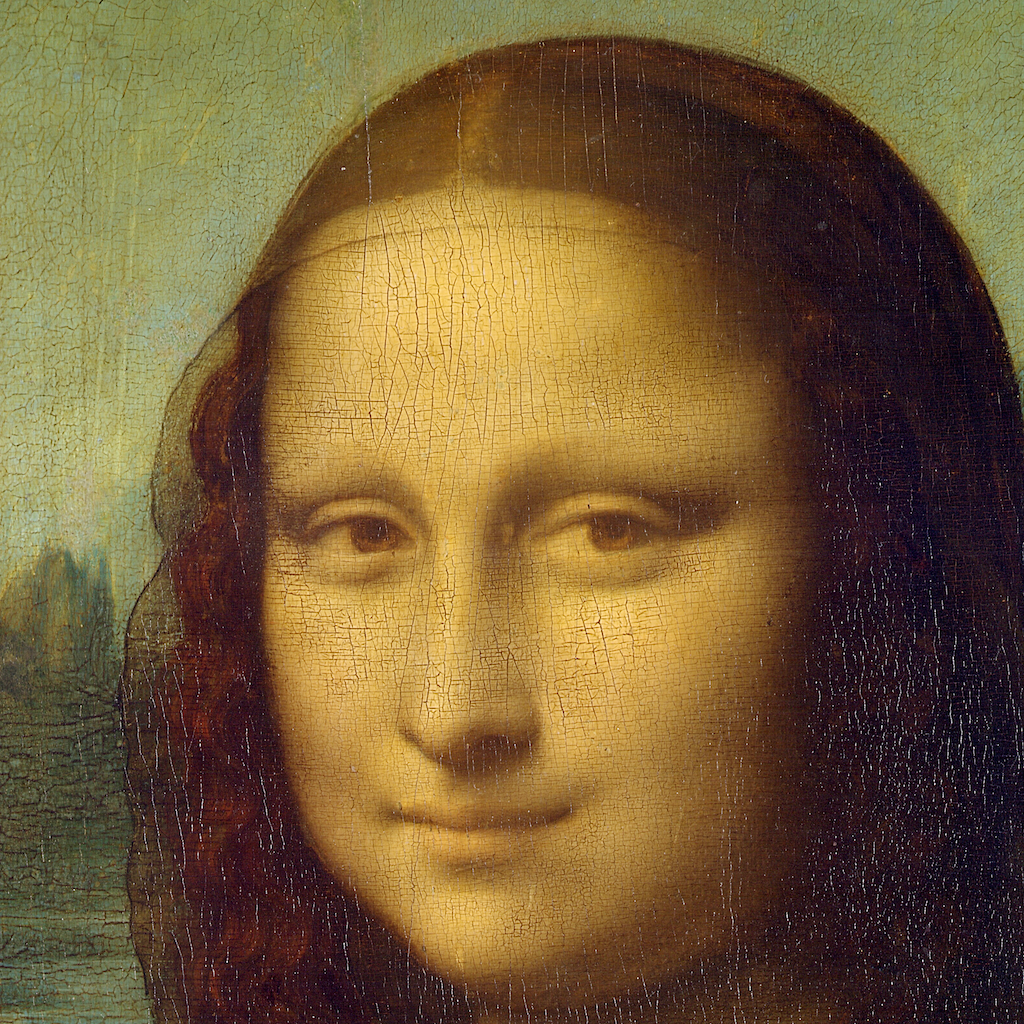}} & 
\subfloat{\includegraphics[scale=1]{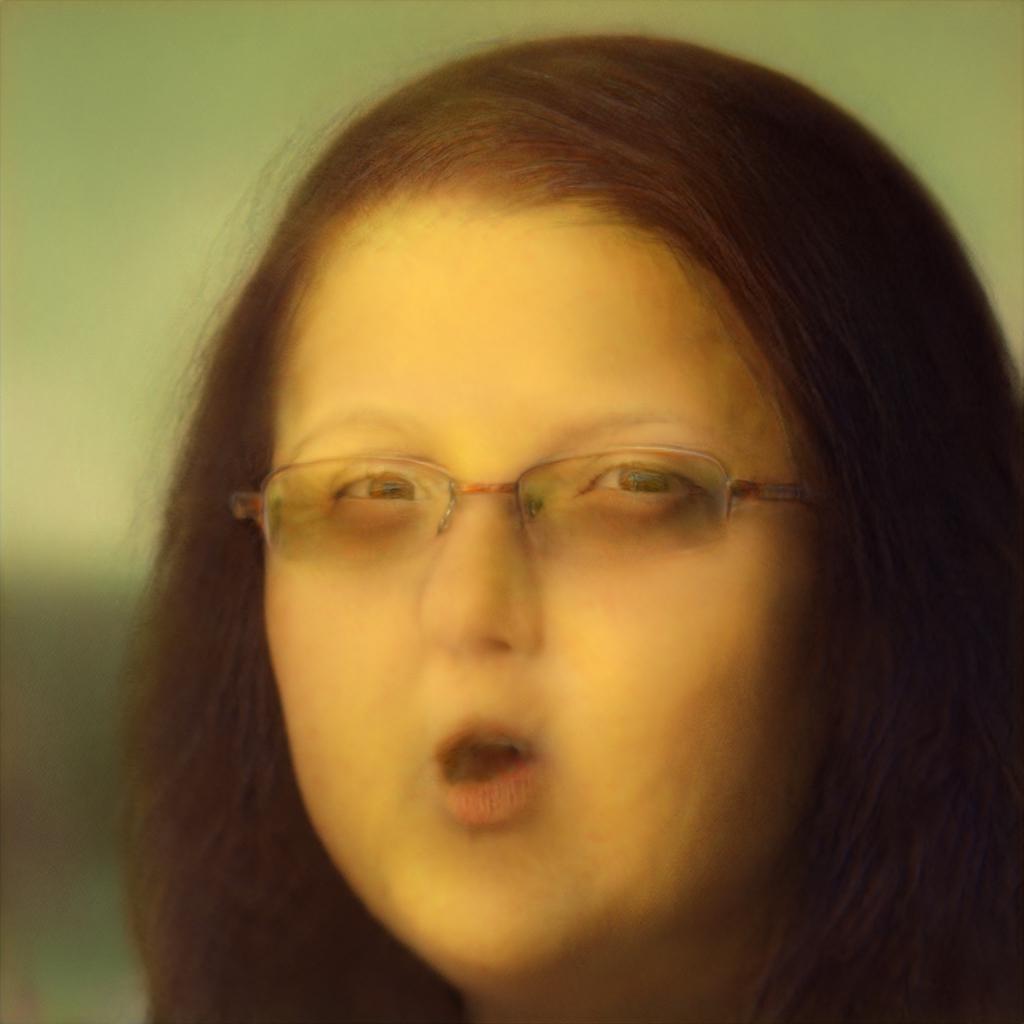}} & 
\subfloat{\includegraphics[scale=1]{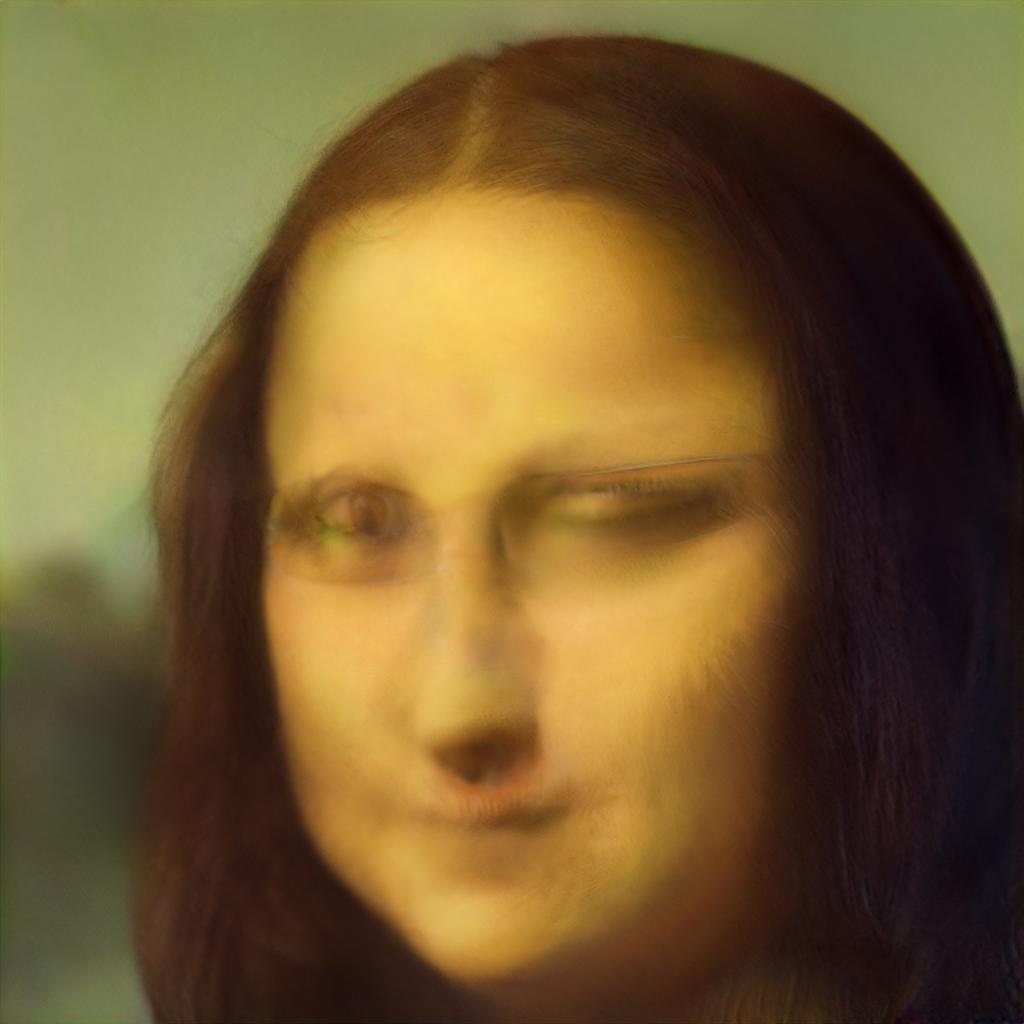}} & 
\subfloat{\includegraphics[scale=1]{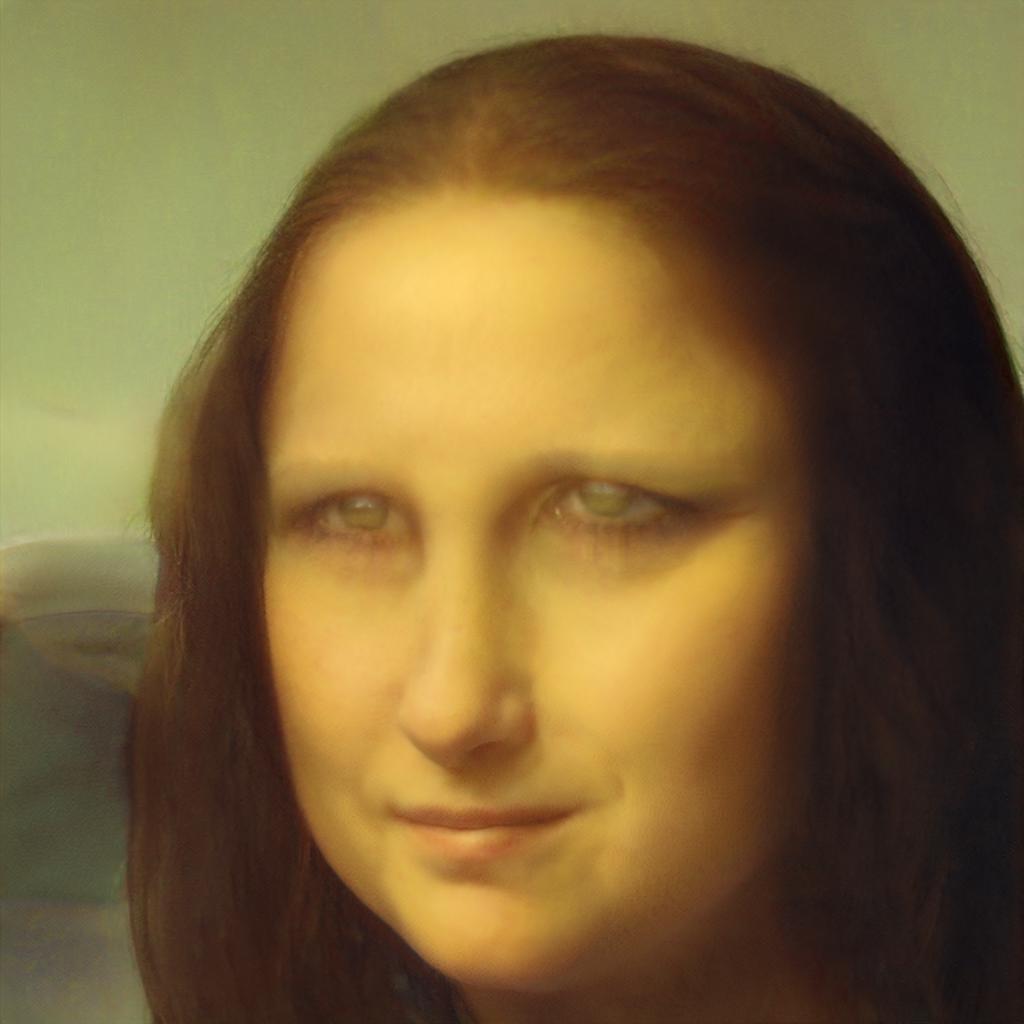}} \vspace{5em}\\ 
\{\fontsize{100}{400}\selectfont Input & \fontsize{100}{400}\selectfont CSGM & \fontsize{100}{400}\selectfont ILO & \fontsize{100}{400}\selectfont SGILO\}
\end{tabular}
\end{adjustbox}
\end{center}
\caption{\small Out of distribution projections. The initial painting image is our of the StyleGAN2 distribution. Using CSGM frequently fails especially when features are not perfectly aligned with the learned manifold. ILO produces a better image that is still not very realistic. Our method (SGILO) uses the Score-prior to improve on ILO and produce a more realistic image. Note that the goal is not to reconstruct the input image, but to demonstrate that by exclusively following the learned score in the latent space makes the generated image more natural.}
\label{fig:natural}
\end{figure*}

\paragraph{Other Inverse Problems.} SGILO is a general framework that can be applied to solve any inverse problem, as long as the forward operator $\mathcal A$ is known and differentiable. Figure \ref{inpainting} shows results for randomized inpainting in the extreme regime of $0.75\%$ observed measurements.

Figure \ref{fig:colorization} shows results for the task of colorization. As shown, both ILO and CSGM introduce artifacts, e.g. see columns 3 and 4 of Row 1. Those artifacts are mostly corrected by our framework, SGILO, that displays more natural colors than prior work.

A final experiment we performed is generating samples using a pre-trained classifier to deviate from the learned distribution. We use a classifier to bias our Langevin algorithm to produce samples that look like ImageNet classes. We use gradients from robust classifiers~\cite{santurkar2019image} to get samples from the class `Bullfrog'. As shown in Figure \ref{frogs},  SGILO is flexible to produce samples outside its learned distribution and retains interesting visual features.

\begin{figure*}[!ht]
\captionsetup{justification=centering}
\begin{center}
\begin{adjustbox}{width=0.55\textwidth, center}
\begin{tabular}{ccccc}
\begin{minipage}{0.33\textwidth}
\includegraphics[width=\textwidth]{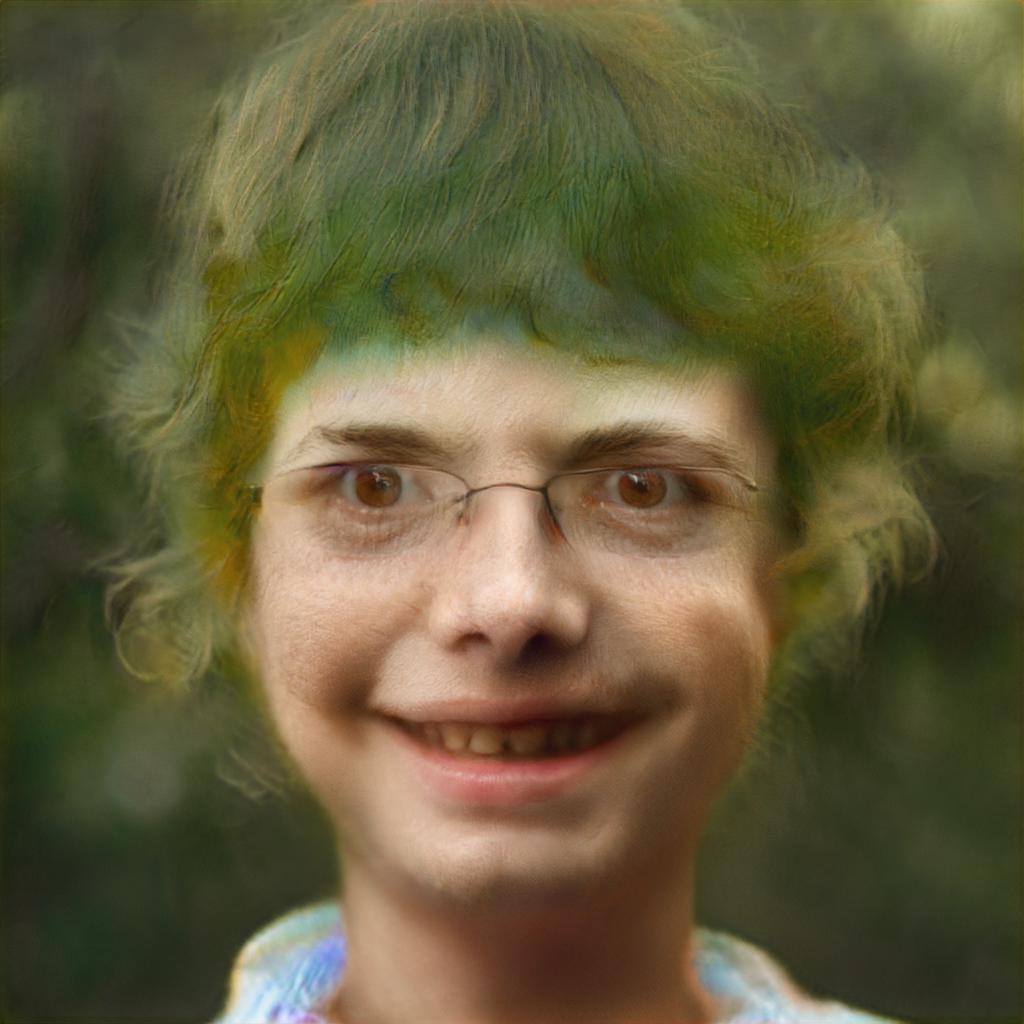}
\end{minipage}
\begin{minipage}{0.33\textwidth}
\includegraphics[width=\textwidth]{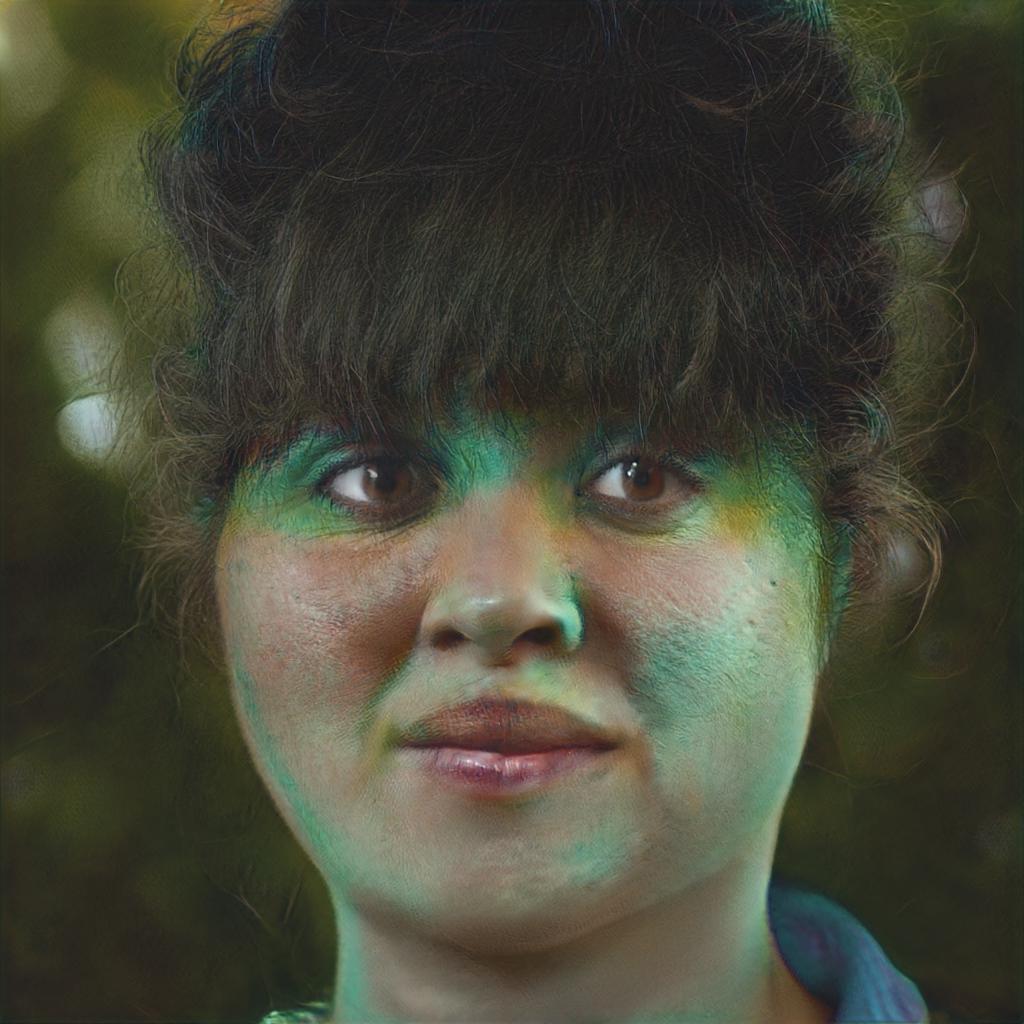}
\end{minipage}
\begin{minipage}{0.33\textwidth}
\includegraphics[width=\textwidth]{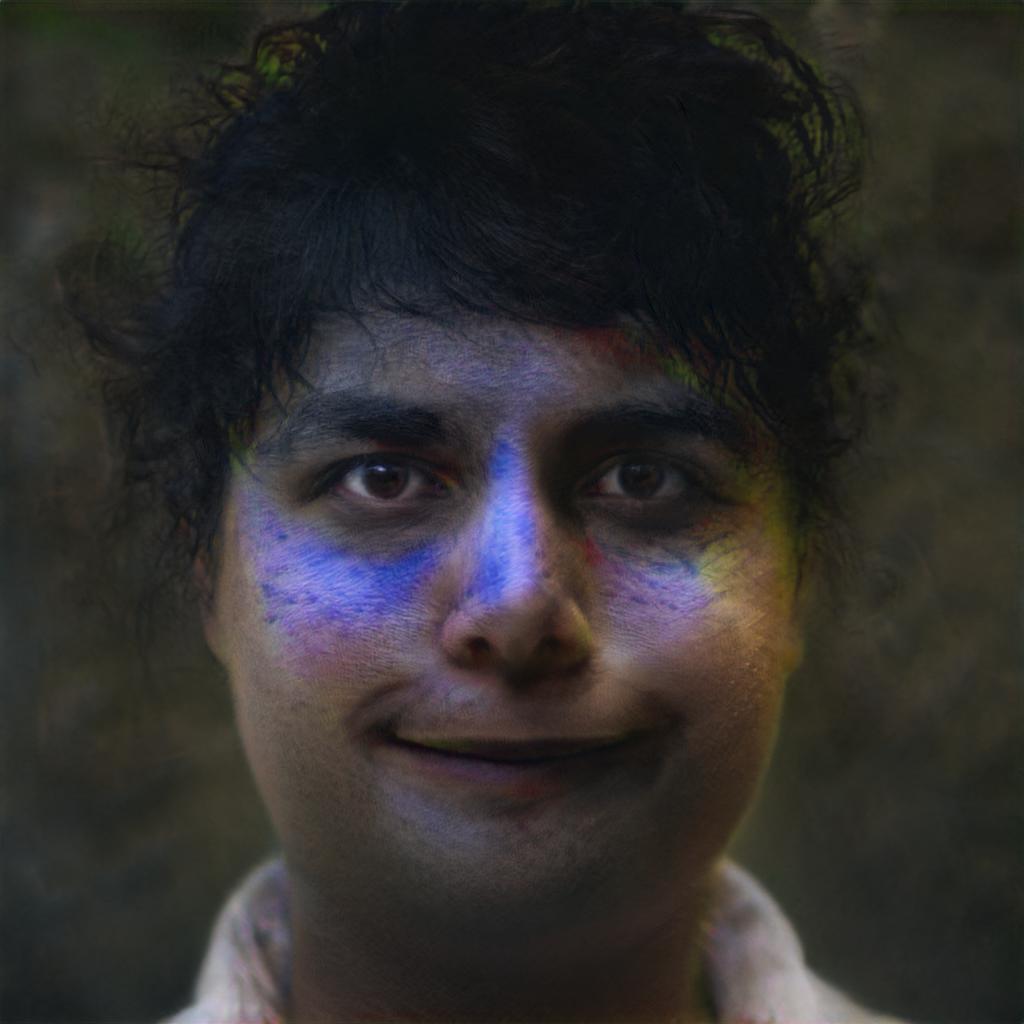}
\end{minipage}
\end{tabular}
\end{adjustbox}
\end{center}
\caption{\small Samples of the posterior using a Bullfrog classifier as a differentiable forward operator. SGILO is flexible and able to extend outside its learned distribution as it produces interesting blendings of human and frog characteristics.}
\label{frogs}
\end{figure*}

\begin{figure}[!htp]
\captionsetup[subfigure]{justification=centering}
\begin{adjustbox}{width=80mm, center}
\begin{tabular}{ccccc}
\captionsetup{justification=centering}
\subfloat{\includegraphics[width=0.2\textwidth]{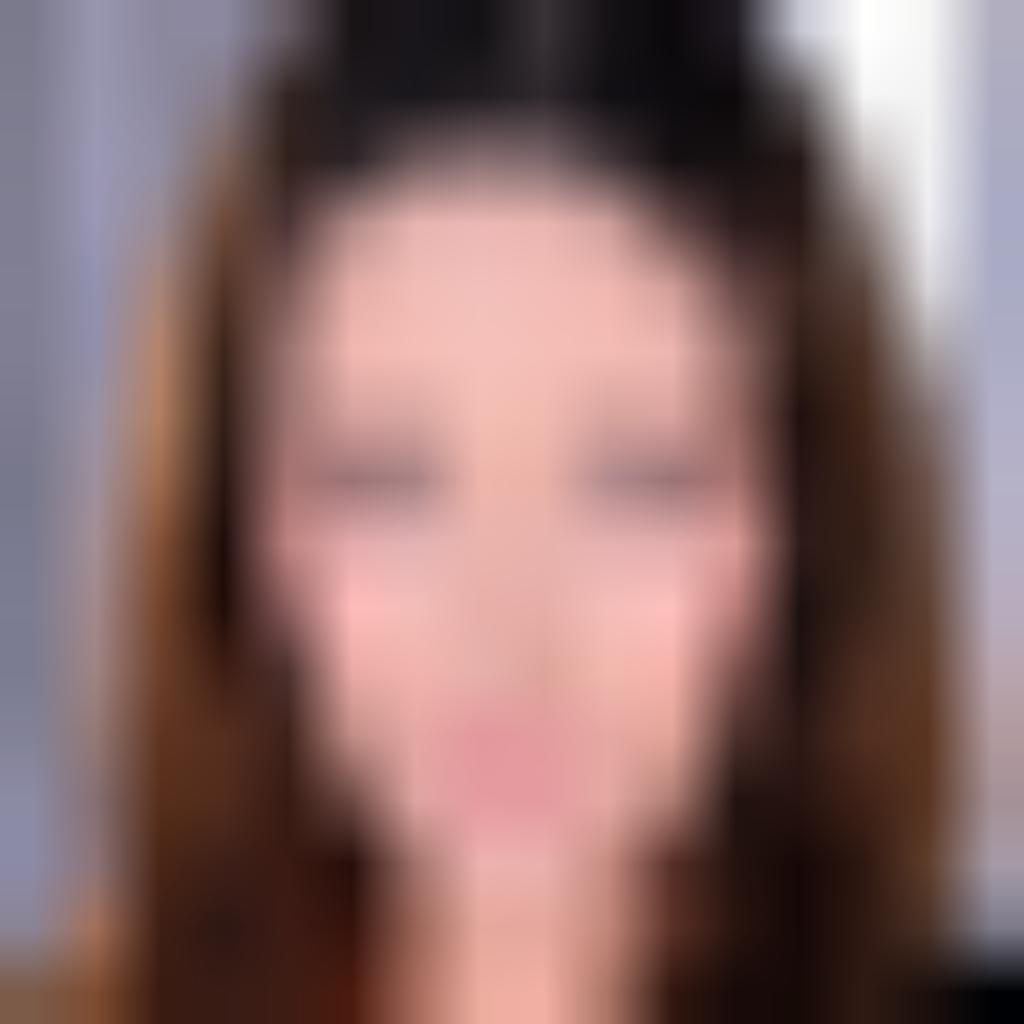}} & \hspace{-4mm}
\subfloat{\includegraphics[width=0.2\textwidth]{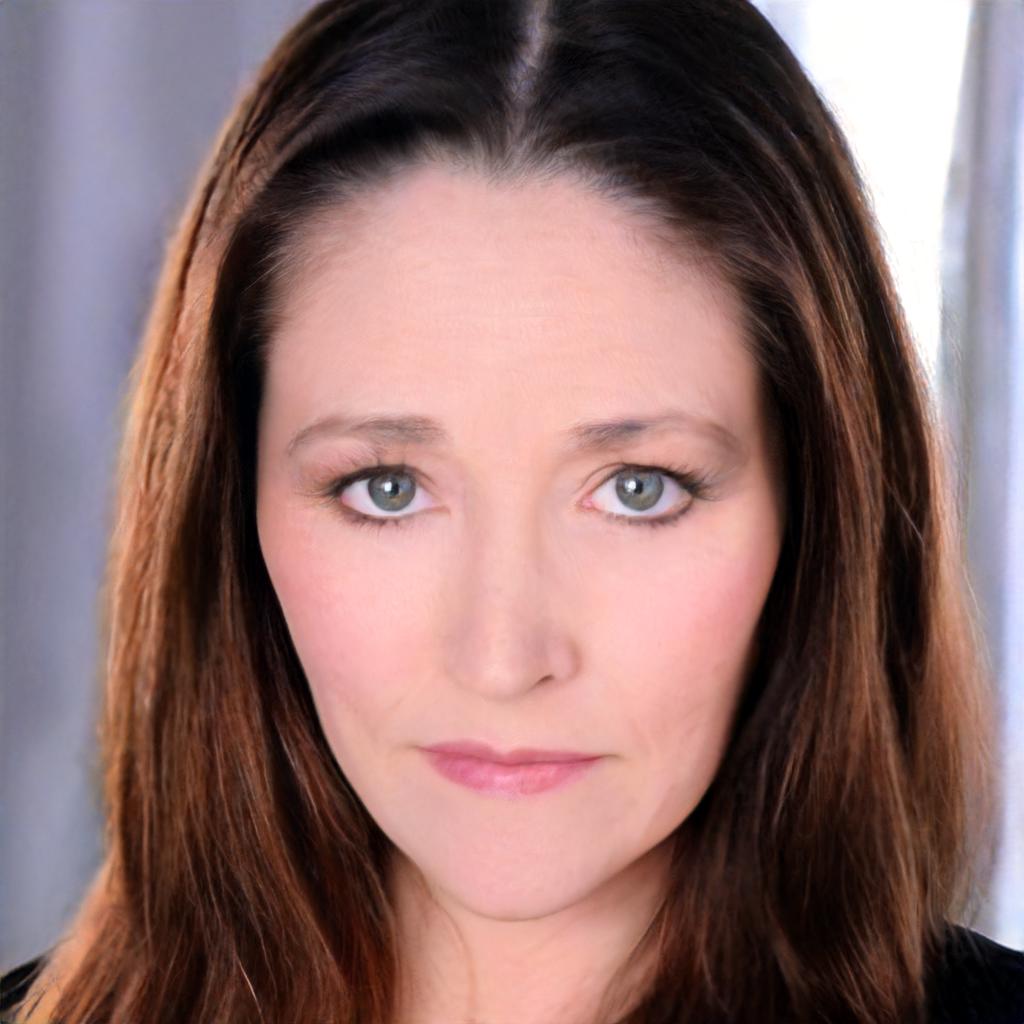}} & \hspace{-4mm}
\subfloat{\includegraphics[width=0.2\textwidth]{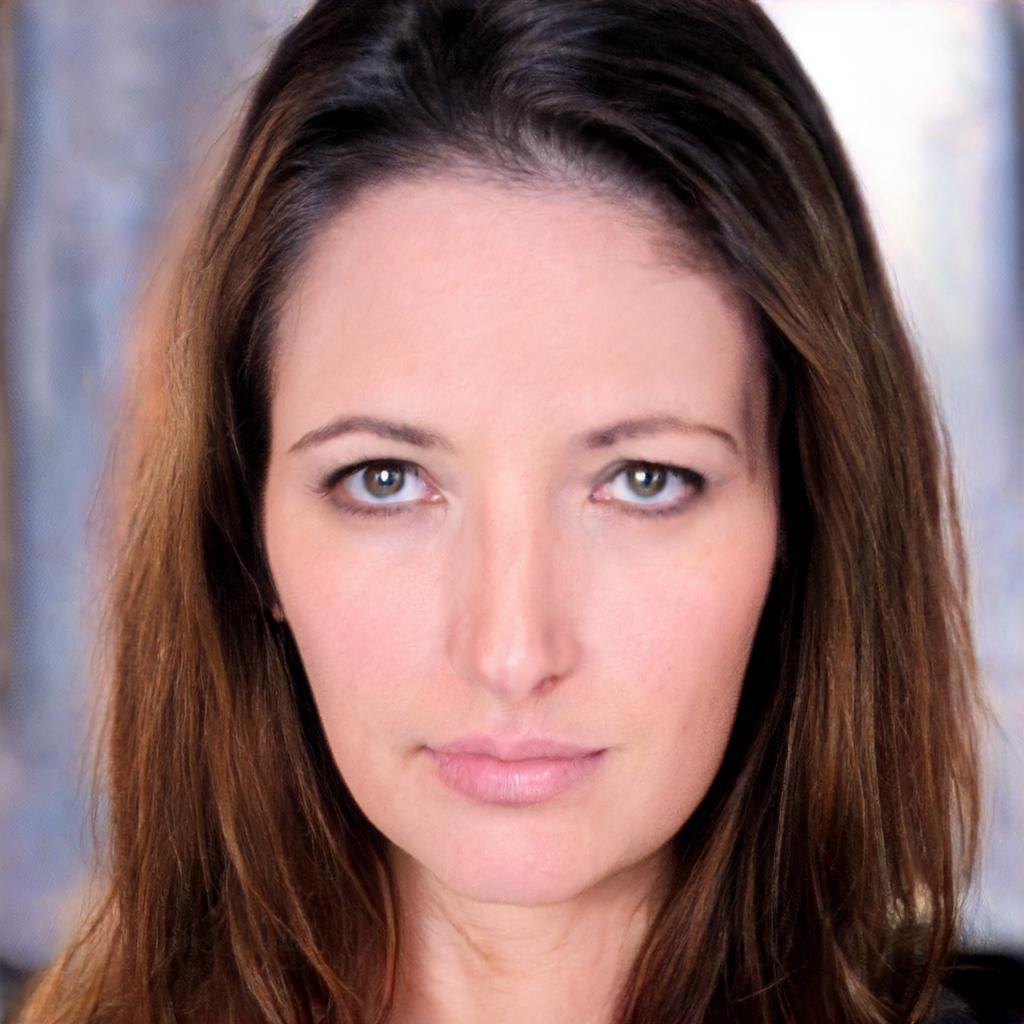}} & \hspace{-4mm} 
\subfloat{\includegraphics[width=0.2\textwidth]{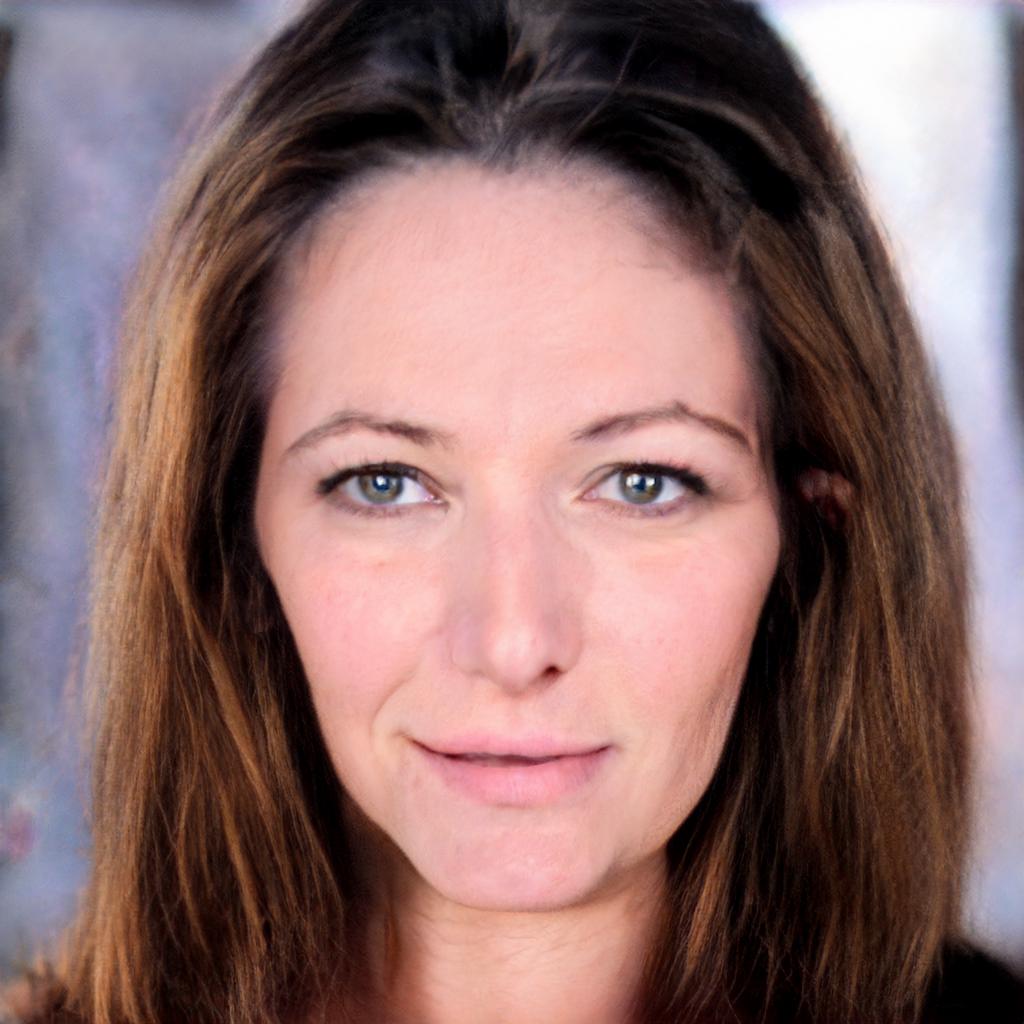}} & \hspace{-4mm}
\subfloat{\includegraphics[width=0.2\textwidth]{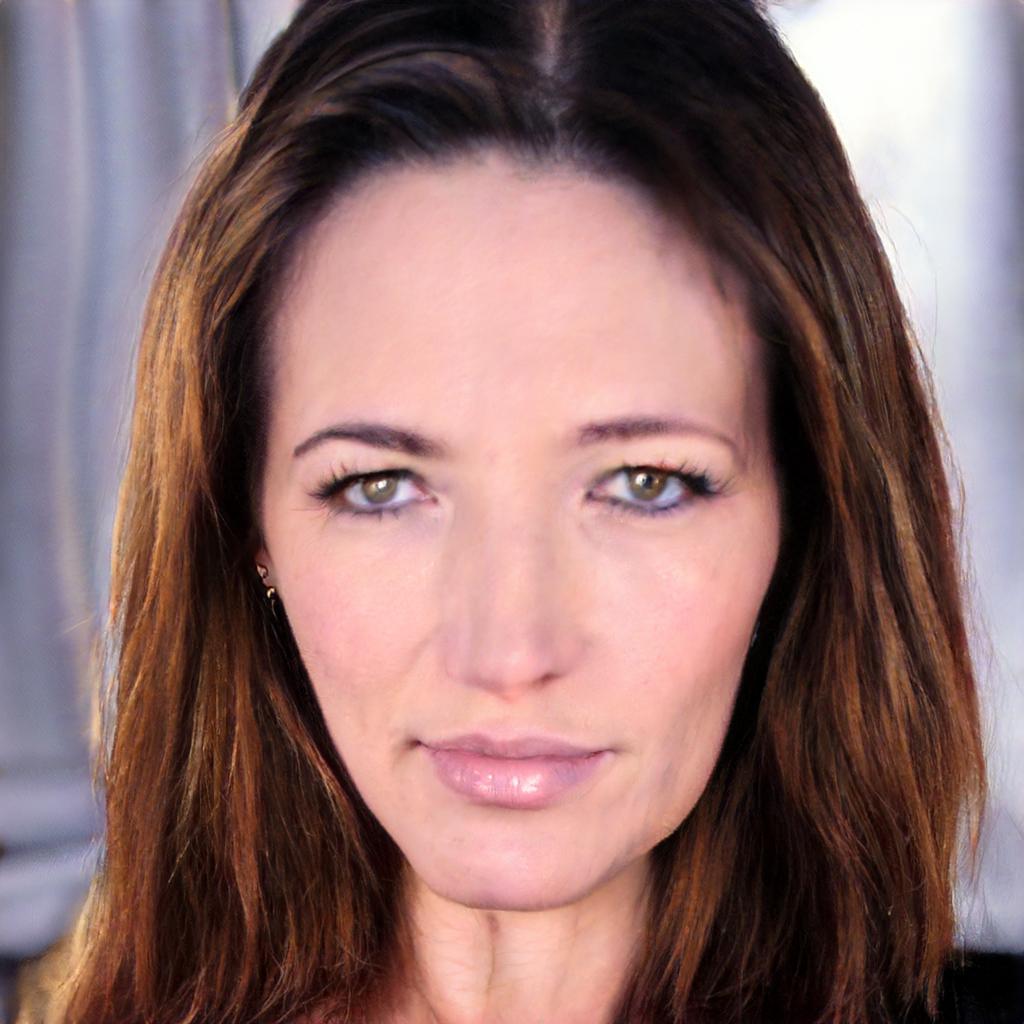}} \\
\subfloat{\includegraphics[width=0.2\textwidth]{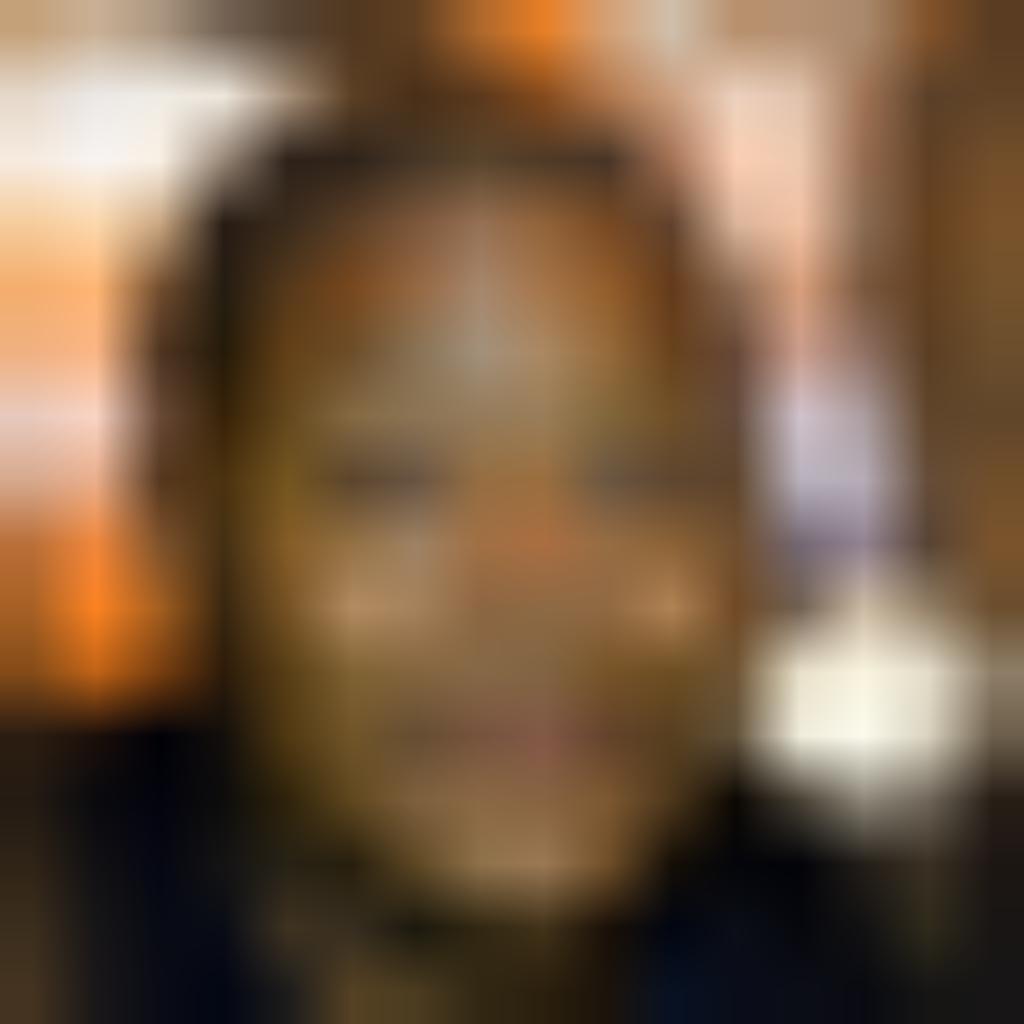}} & \hspace{-4mm}
\subfloat{\includegraphics[width=0.2\textwidth]{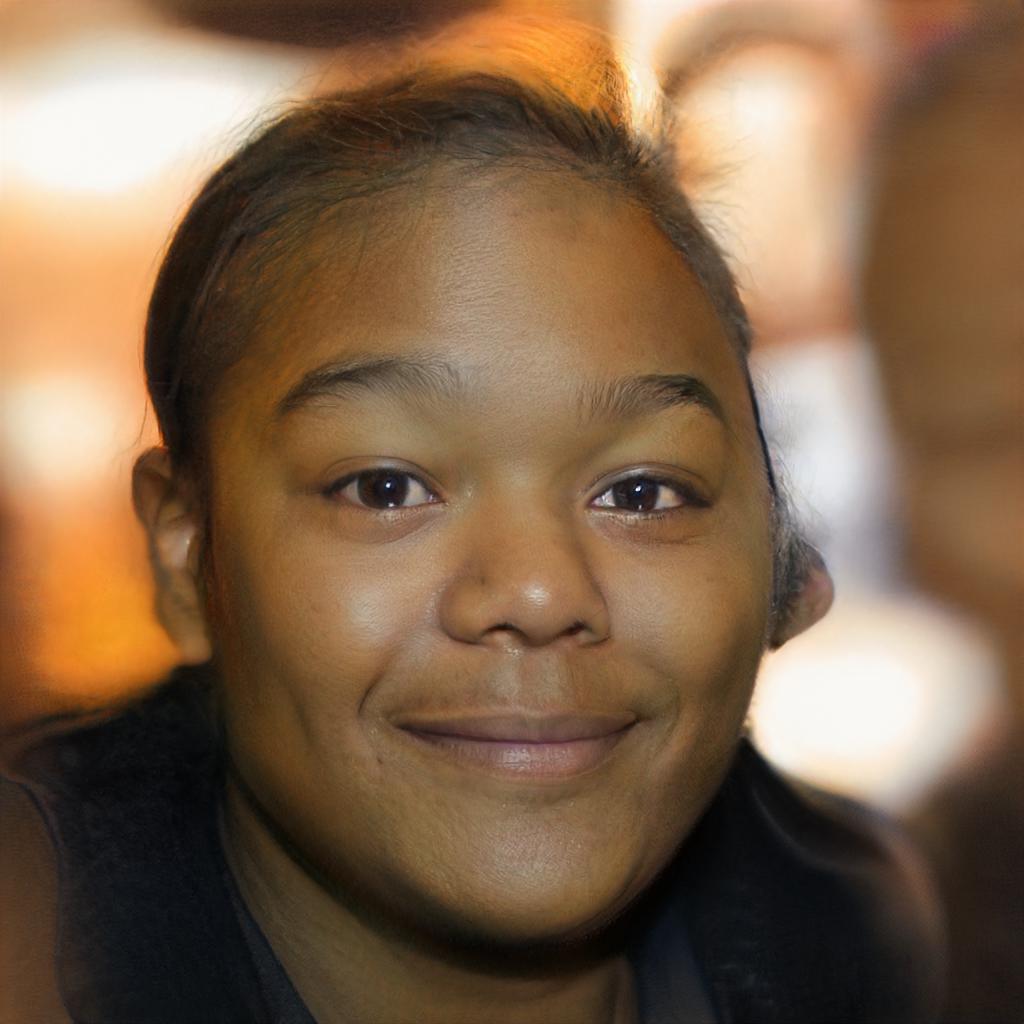}} & \hspace{-4mm}
\subfloat{\includegraphics[width=0.2\textwidth]{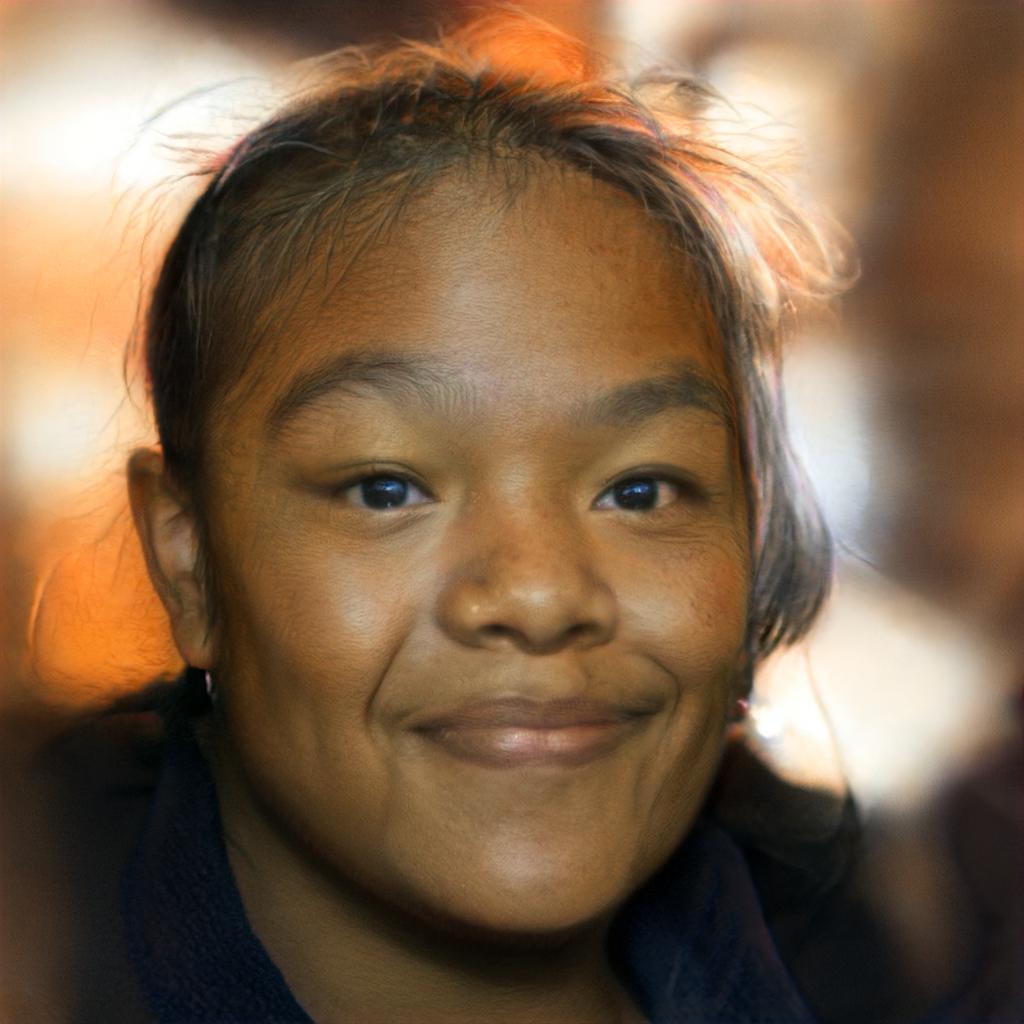}} & \hspace{-4mm} 
\subfloat{\includegraphics[width=0.2\textwidth]{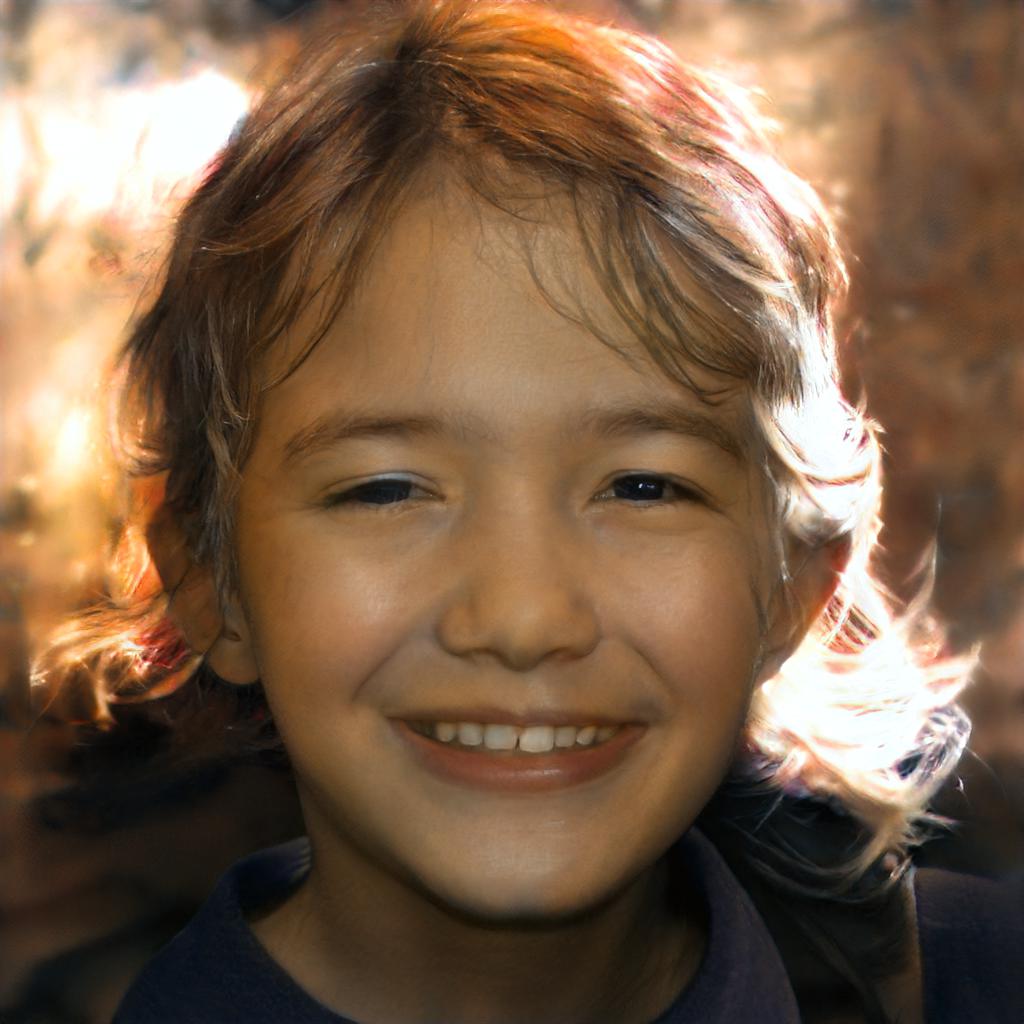}} & \hspace{-4mm}
\subfloat{\includegraphics[width=0.2\textwidth]{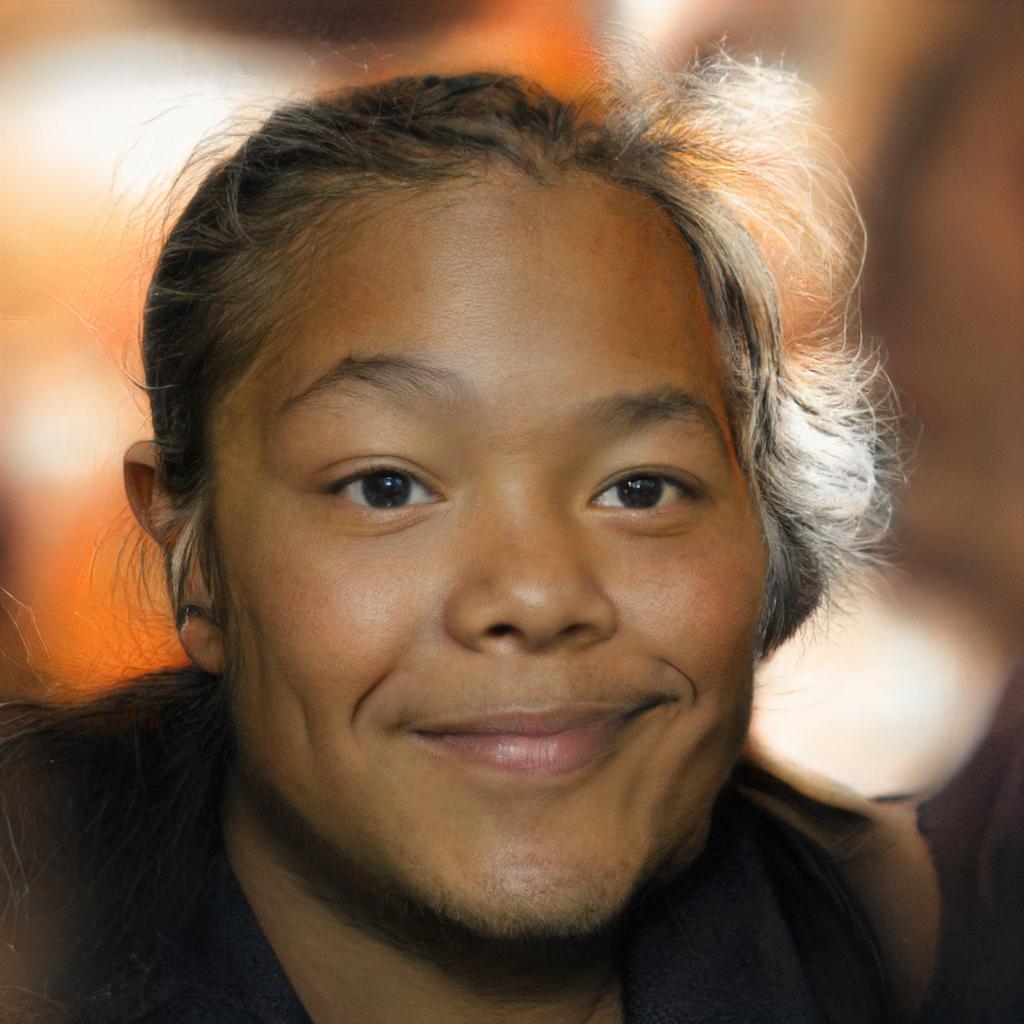}}
\end{tabular}
\end{adjustbox}
\vspace{-2mm}
\caption{\small{Diverse reconstructions with posterior sampling.}}
\label{fig:diverse_reconstructions}
\end{figure}

\paragraph{Posterior Sampling Ablation} As we argued in Sections \ref{sec:introduction}, \ref{sec:method}, SGILO is a posterior sampling method. Among others, posterior sampling offers i) diverse reconstructions, ii) reduced bias. We perform two experiments to examine how well SGILO performs with respect to i) and ii). Figure \ref{fig:diverse_reconstructions} shows different reconstructions we get for the measurements given in the first column. As shown, the generated images have variability with respect to age, ethnicity, eye color, etc. We also perform a preliminary experiment to examine whether SGILO has the potential to reduce dataset biases. To that end, we downsample $64$ images of men and women (each) by $\times 256$ and then reconstruct them using ILO and SGILO. For each of the reconstructed images, we use CLIP~\cite{clip} to predict the gender. ILO predicts correctly the gender in $78/128$, while SGILO succeeds in $\mathbf{89/128}$. This experiment aligns with the findings of \citet{jalal2021fairness} that shows that reconstruction methods based on posterior sampling usually lead to increased fairness.

\paragraph{Speed Ablation} One advantage of SGILO over other sampling frameworks with conventional Score-Based networks is the speed. The reasons SGILO is fast are twofold: i) the model is working on a low-dimensional space and ii) one might not need to reverse the whole diffusion, since any step of SGLD can serve as a hot-start for the reverse diffusion. For most of our experiments, instead of using directly the gradient of $\log p(z_0)$, we sample for some small $t$ one $p(z_t|z_0)$ according to the SDE and then we run the reverse SDE for the interval $[t, 0]$. This can give more flexibility to the score-based model to guide the solutions of the Intermediate Layer Optimization and is still pretty fast as long as $t$ is small. This is similar in spirit to the SDEdit~\citep{meng2021sdedit} paper. The only time we revert the whole diffusion is when we treat the score-based model as a generator (instead of regularizer for ILO), as we do for ablation purposes in Figure \ref{fig:generations}. Specifically, each inverse problem takes $1-2$ minutes to get solved on a single V100 GPU. Figure \ref{fig:speed_plots} of the Appendix shows how MSE changes as time goes. SGILO typically requires $300$ function evaluations which corresponds to $1-2$ minutes. Most score-based models, like NCSNv3, require thousands of steps. For image generation, SGILO needs $\sim40$ seconds for a single sample on a single GPU, which is $~10\times$ faster than score models in the pixel space. We note that recently many other methods for accelerating diffusion models have been proposed~\citep{karras2022elucidating, salimans2022progressive, nichol2021improved, song2020denoising, jolicoeurmartineau2021gotta} that are orthogonal to (and hence, can be combined with) our approach.

\section{Related Work}
The CSGM paper~\cite{bora2017compressed} introduced the unsupervised framework for inverse problems and this has been shown to be optimal under natural assumptions~\cite{inf_theoretic,kamath2019lower}. 
Recent works have investigated methods of expanding the range of the generator.  Optimizing an intermediate layer was first proposed in the context of inversion as a way to identify data distribution modes dropped by GANs~\cite{bau2019seeing}. The same technique has been rediscovered in the GAN surgery paper~\cite{surgery}, in which the authors demonstrated (among other things) that the expansion of the range is useful for out-of-distribution generation. Intermediate Layer Optimization~\cite{daras2021intermediate} improved prior work by i) using the powerful StyleGAN generator (as was first pioneered in PULSE~\cite{pulse} for the special case of super-resolution), ii) gradually transitioning to higher layers with sequential optimization, iii) regularizing the solutions of CSGM~\cite{bora2017compressed} by only allowing sparse deviations from the range of some intermediate layer. 
\citet{dhar2018modeling} previously proposed extending the range by sparse deviations from the output space, but ILO generalized this by allowing deviations from the range of \textit{any} layer.

Score-based modeling was proposed by \citet{score_first} using Score Matching~\cite{score_matching} and further work~\cite{song_sde, ddpm_beats_gans, song2020improved} significantly improved score-based performance. 
Our work is not the first one to train score-based model in the latent space. \citet{vahdat2021score} also trains score-based models in the latent space of a VAE to improve generation quality and sampling time. Our work is related but
we are training score-based networks on already pre-trained generators and we are focusing on solving inverse problems (instead of generation) by formulating the SGILO algorithm. Algorithms for solving inverse problems with score-based models in pixel-space have been developed in the interesting works of~\citet{kawar2021snips, ddrm, jalal2021robust}. We do not compare directly with these methods since they use different generators as priors for solving inverse problems.

On the theoretical side, our work extends the seminal work of ~\citet{paul_hand_v1, paul_hand_v2}. Prior work showed that a variant of Gradient Descent converges polynomially for MAP estimation using random weight ReLU Generators.
Our result is that Langevin Dynamics gives polynomial convergence to the posterior distribution under the same setting. Prior work has also analyzed convergence of Langevin Dynamics for non-convex optimization under different set of assumptions, e.g. see \citet{raginsky2017non, block2020fast, xu2017global}. For theoretical guarantees for sampling with generative models with latent diffusions, we also refer the interested reader to the relevant work of \citet{tzen2019theoretical}.

Finally, it is useful to underline that in the presence of 
enough training data, end-to-end supervised methods usually outperform unsupervised methods, e.g. see \citet{Tian_2020, Sun2020, tripathi2018correction} for denoising, \citet{Sun_2020, yang2019deep} for super-resolution and \citet{Yu_2019, Liu_2019} for inpainting. The main disadvantages of solving inverse problems with end-to-end supervised methods are that: i) separate training is required for each problem, ii) there is
significant fragility to forward operator changes (robustness issues)~\citep{darestani2021measuring,ongie2020deep}.

\section{Conclusions}
This paper introduced Score-Guided Intermediate Layer Optimization (SGILO), a framework for posterior sampling in the latent space of a pre-trained generator. Our work extends the sparsity prior that appeared in prior work, with a powerful generative prior that is used to solve inverse problems. On the theoretical side, we proved fast convergence of the Langevin Algorithm for random weights generators, for the simplified case of uniform prior over the latents.

\section{Acknowledgments}
This research has been supported by NSF Grants CCF 1763702, 1934932, AF 1901281, 2008710, 2019844 the NSF IFML 2019844 award as well as research gifts by Western Digital, WNCG and MLL, computing resources from TACC and the Archie Straiton Fellowship. This work is also supported by NSF Awards CCF-1901292, DMS-2022448 and DMS2134108, a Simons
Investigator Award, the Simons Collaboration on the Theory of Algorithmic Fairness, a DSTA
grant, the DOE PhILMs project (DE-AC05-76RL01830).

\bibliography{references}

\begin{thebibliography}{71}
\providecommand{\natexlab}[1]{#1}
\providecommand{\url}[1]{\texttt{#1}}
\expandafter\ifx\csname urlstyle\endcsname\relax
  \providecommand{\doi}[1]{doi: #1}\else
  \providecommand{\doi}{doi: \begingroup \urlstyle{rm}\Url}\fi

\bibitem[Asim et~al.(2019)Asim, Ahmed, and Hand]{asim2019invertible}
Asim, M., Ahmed, A., and Hand, P.
\newblock Invertible generative models for inverse problems: mitigating
  representation error and dataset bias.
\newblock \emph{arXiv preprint arXiv:1905.11672}, 2019.

\bibitem[Bau et~al.(2019)Bau, Zhu, Wulff, Peebles, Strobelt, Zhou, and
  Torralba]{bau2019seeing}
Bau, D., Zhu, J.-Y., Wulff, J., Peebles, W., Strobelt, H., Zhou, B., and
  Torralba, A.
\newblock Seeing what a gan cannot generate, 2019.

\bibitem[Block et~al.(2020)Block, Mroueh, Rakhlin, and Ross]{block2020fast}
Block, A., Mroueh, Y., Rakhlin, A., and Ross, J.
\newblock Fast mixing of multi-scale langevin dynamics under the manifold
  hypothesis, 2020.

\bibitem[Bora et~al.(2017)Bora, Jalal, Price, and Dimakis]{bora2017compressed}
Bora, A., Jalal, A., Price, E., and Dimakis, A.~G.
\newblock Compressed sensing using generative models.
\newblock In \emph{International Conference on Machine Learning}, pp.\
  537--546. PMLR, 2017.

\bibitem[Brock et~al.(2019)Brock, Donahue, and Simonyan]{biggan}
Brock, A., Donahue, J., and Simonyan, K.
\newblock Large scale gan training for high fidelity natural image synthesis,
  2019.

\bibitem[Chan et~al.(2021)Chan, Monteiro, Kellnhofer, Wu, and Wetzstein]{pigan}
Chan, E.~R., Monteiro, M., Kellnhofer, P., Wu, J., and Wetzstein, G.
\newblock pi-gan: Periodic implicit generative adversarial networks for
  3d-aware image synthesis.
\newblock \emph{2021 IEEE/CVF Conference on Computer Vision and Pattern
  Recognition (CVPR)}, Jun 2021.
\newblock \doi{10.1109/cvpr46437.2021.00574}.
\newblock URL \url{http://dx.doi.org/10.1109/CVPR46437.2021.00574}.

\bibitem[Daras et~al.(2020)Daras, Odena, Zhang, and Dimakis]{ylg}
Daras, G., Odena, A., Zhang, H., and Dimakis, A.~G.
\newblock Your local gan: Designing two dimensional local attention mechanisms
  for generative models.
\newblock In \emph{Proceedings of the IEEE/CVF Conference on Computer Vision
  and Pattern Recognition}, 2020.

\bibitem[Daras et~al.(2021{\natexlab{a}})Daras, Chu, Kumar, Lagun, and
  Dimakis]{daras2021solving}
Daras, G., Chu, W.-S., Kumar, A., Lagun, D., and Dimakis, A.~G.
\newblock Solving inverse problems with nerfgans.
\newblock \emph{arXiv preprint arXiv:2112.09061}, 2021{\natexlab{a}}.

\bibitem[Daras et~al.(2021{\natexlab{b}})Daras, Dean, Jalal, and
  Dimakis]{daras2021intermediate}
Daras, G., Dean, J., Jalal, A., and Dimakis, A.~G.
\newblock Intermediate layer optimization for inverse problems using deep
  generative models.
\newblock In \emph{ICML 2021}, 2021{\natexlab{b}}.

\bibitem[Darestani et~al.(2021)Darestani, Chaudhari, and
  Heckel]{darestani2021measuring}
Darestani, M.~Z., Chaudhari, A., and Heckel, R.
\newblock Measuring robustness in deep learning based compressive sensing.
\newblock \emph{arXiv preprint arXiv:2102.06103}, 2021.

\bibitem[Daskalakis et~al.(2020{\natexlab{a}})Daskalakis, Rohatgi, and
  Zampetakis]{daskalakis_constant_expansion}
Daskalakis, C., Rohatgi, D., and Zampetakis, E.
\newblock Constant-expansion suffices for compressed sensing with generative
  priors.
\newblock In Larochelle, H., Ranzato, M., Hadsell, R., Balcan, M.~F., and Lin,
  H. (eds.), \emph{Advances in Neural Information Processing Systems},
  volume~33, pp.\  13917--13926. Curran Associates, Inc., 2020{\natexlab{a}}.
\newblock URL
  \url{https://proceedings.neurips.cc/paper/2020/file/9fa83fec3cf3810e5680ed45f7124dce-Paper.pdf}.

\bibitem[Daskalakis et~al.(2020{\natexlab{b}})Daskalakis, Rohatgi, and
  Zampetakis]{daskalakis2020constant}
Daskalakis, C., Rohatgi, D., and Zampetakis, M.
\newblock Constant-expansion suffices for compressed sensing with generative
  priors. in the.
\newblock In \emph{34th Annual Conference on Neural Information Processing
  Systems (NeurIPS), NeurIPS 2020}, 2020{\natexlab{b}}.

\bibitem[Dhar et~al.(2018)Dhar, Grover, and Ermon]{dhar2018modeling}
Dhar, M., Grover, A., and Ermon, S.
\newblock Modeling sparse deviations for compressed sensing using generative
  models.
\newblock In \emph{International Conference on Machine Learning}, pp.\
  1214--1223. PMLR, 2018.

\bibitem[Dhariwal \& Nichol(2021)Dhariwal and Nichol]{ddpm_beats_gans}
Dhariwal, P. and Nichol, A.
\newblock Diffusion models beat gans on image synthesis, 2021.

\bibitem[Dosovitskiy et~al.(2020)Dosovitskiy, Beyer, Kolesnikov, Weissenborn,
  Zhai, Unterthiner, Dehghani, Minderer, Heigold, Gelly, Uszkoreit, and
  Houlsby]{ViT}
Dosovitskiy, A., Beyer, L., Kolesnikov, A., Weissenborn, D., Zhai, X.,
  Unterthiner, T., Dehghani, M., Minderer, M., Heigold, G., Gelly, S.,
  Uszkoreit, J., and Houlsby, N.
\newblock An image is worth 16x16 words: Transformers for image recognition at
  scale.
\newblock \emph{arXiv preprint arXiv:2010.11929}, 2020.

\bibitem[Hand \& Voroninski(2018{\natexlab{a}})Hand and
  Voroninski]{hand2018global}
Hand, P. and Voroninski, V.
\newblock Global guarantees for enforcing deep generative priors by empirical
  risk.
\newblock In \emph{Conference On Learning Theory}, pp.\  970--978. PMLR,
  2018{\natexlab{a}}.

\bibitem[Hand \& Voroninski(2018{\natexlab{b}})Hand and
  Voroninski]{paul_hand_v1}
Hand, P. and Voroninski, V.
\newblock Global guarantees for enforcing deep generative priors by empirical
  risk.
\newblock In Bubeck, S., Perchet, V., and Rigollet, P. (eds.),
  \emph{Proceedings of the 31st Conference On Learning Theory}, volume~75 of
  \emph{Proceedings of Machine Learning Research}, pp.\  970--978. PMLR, 06--09
  Jul 2018{\natexlab{b}}.
\newblock URL \url{http://proceedings.mlr.press/v75/hand18a.html}.

\bibitem[Hand et~al.(2018)Hand, Leong, and Voroninski]{hand2018phase}
Hand, P., Leong, O., and Voroninski, V.
\newblock Phase retrieval under a generative prior.
\newblock In \emph{Advances in Neural Information Processing Systems}, pp.\
  9136--9146, 2018.

\bibitem[Heckel \& Hand(2018)Heckel and Hand]{heckel2018deep}
Heckel, R. and Hand, P.
\newblock Deep decoder: Concise image representations from untrained
  non-convolutional networks.
\newblock \emph{arXiv preprint arXiv:1810.03982}, 2018.

\bibitem[Ho et~al.(2020)Ho, Jain, and Abbeel]{ho2020denoising}
Ho, J., Jain, A., and Abbeel, P.
\newblock Denoising diffusion probabilistic models.
\newblock \emph{arXiv preprint arxiv:2006.11239}, 2020.

\bibitem[Huang et~al.(2018)Huang, Hand, Heckel, and Voroninski]{paul_hand_v2}
Huang, W., Hand, P., Heckel, R., and Voroninski, V.
\newblock A provably convergent scheme for compressive sensing under random
  generative priors, 12 2018.

\bibitem[Huang et~al.(2021)Huang, Hand, Heckel, and
  Voroninski]{huang2021provably}
Huang, W., Hand, P., Heckel, R., and Voroninski, V.
\newblock A provably convergent scheme for compressive sensing under random
  generative priors.
\newblock \emph{Journal of Fourier Analysis and Applications}, 27\penalty0
  (2):\penalty0 1--34, 2021.

\bibitem[Hyv{{\"a}}rinen(2005)]{score_matching}
Hyv{{\"a}}rinen, A.
\newblock Estimation of non-normalized statistical models by score matching.
\newblock \emph{Journal of Machine Learning Research}, 6\penalty0
  (24):\penalty0 695--709, 2005.
\newblock URL \url{http://jmlr.org/papers/v6/hyvarinen05a.html}.

\bibitem[Jalal et~al.(2020)Jalal, Karmalkar, Dimakis, and
  Price]{jalal_cond_resampling}
Jalal, A., Karmalkar, S., Dimakis, A., and Price, E.
\newblock Compressed sensing with approximate priors via conditional
  resampling.
\newblock In \emph{NeurIPS 2020 Workshop on Deep Learning and Inverse
  Problems}, 2020.
\newblock URL \url{https://openreview.net/forum?id=8ozSD4Oymw}.

\bibitem[Jalal et~al.(2021{\natexlab{a}})Jalal, Arvinte, Daras, Price, Dimakis,
  and Tamir]{jalal2021robust}
Jalal, A., Arvinte, M., Daras, G., Price, E., Dimakis, A.~G., and Tamir, J.~I.
\newblock Robust compressed sensing mri with deep generative priors,
  2021{\natexlab{a}}.

\bibitem[Jalal et~al.(2021{\natexlab{b}})Jalal, Karmalkar, Dimakis, and
  Price]{jalal2021instanceoptimal}
Jalal, A., Karmalkar, S., Dimakis, A.~G., and Price, E.
\newblock Instance-optimal compressed sensing via posterior sampling,
  2021{\natexlab{b}}.

\bibitem[Jalal et~al.(2021{\natexlab{c}})Jalal, Karmalkar, Hoffmann, Dimakis,
  and Price]{jalal2021fairness}
Jalal, A., Karmalkar, S., Hoffmann, J., Dimakis, A., and Price, E.
\newblock Fairness for image generation with uncertain sensitive attributes.
\newblock In \emph{International Conference on Machine Learning}, pp.\
  4721--4732. PMLR, 2021{\natexlab{c}}.

\bibitem[Jolicoeur-Martineau et~al.(2021)Jolicoeur-Martineau, Li,
  Piché-Taillefer, Kachman, and Mitliagkas]{jolicoeurmartineau2021gotta}
Jolicoeur-Martineau, A., Li, K., Piché-Taillefer, R., Kachman, T., and
  Mitliagkas, I.
\newblock Gotta go fast when generating data with score-based models, 2021.

\bibitem[Kamath et~al.(2019)Kamath, Karmalkar, and Price]{kamath2019lower}
Kamath, A., Karmalkar, S., and Price, E.
\newblock Lower bounds for compressed sensing with generative models.
\newblock \emph{arXiv preprint arXiv:1912.02938}, 2019.

\bibitem[Karras et~al.(2019)Karras, Laine, and Aila]{stylegan}
Karras, T., Laine, S., and Aila, T.
\newblock A style-based generator architecture for generative adversarial
  networks.
\newblock \emph{2019 IEEE/CVF Conference on Computer Vision and Pattern
  Recognition (CVPR)}, Jun 2019.
\newblock \doi{10.1109/cvpr.2019.00453}.
\newblock URL \url{http://dx.doi.org/10.1109/CVPR.2019.00453}.

\bibitem[Karras et~al.(2020)Karras, Laine, Aittala, Hellsten, Lehtinen, and
  Aila]{stylegan2}
Karras, T., Laine, S., Aittala, M., Hellsten, J., Lehtinen, J., and Aila, T.
\newblock Analyzing and improving the image quality of stylegan.
\newblock \emph{2020 IEEE/CVF Conference on Computer Vision and Pattern
  Recognition (CVPR)}, Jun 2020.
\newblock \doi{10.1109/cvpr42600.2020.00813}.
\newblock URL \url{http://dx.doi.org/10.1109/cvpr42600.2020.00813}.

\bibitem[Karras et~al.(2022)Karras, Aittala, Aila, and
  Laine]{karras2022elucidating}
Karras, T., Aittala, M., Aila, T., and Laine, S.
\newblock Elucidating the design space of diffusion-based generative models,
  2022.

\bibitem[Kawar et~al.(2021)Kawar, Vaksman, and Elad]{kawar2021snips}
Kawar, B., Vaksman, G., and Elad, M.
\newblock Snips: Solving noisy inverse problems stochastically.
\newblock \emph{Advances in Neural Information Processing Systems},
  34:\penalty0 21757--21769, 2021.

\bibitem[Kawar et~al.(2022)Kawar, Elad, Ermon, and Song]{ddrm}
Kawar, B., Elad, M., Ermon, S., and Song, J.
\newblock Denoising diffusion restoration models, 2022.

\bibitem[Kelkar \& Anastasio(2021)Kelkar and Anastasio]{kelkar2021prior}
Kelkar, V.~A. and Anastasio, M.~A.
\newblock Prior image-constrained reconstruction using style-based generative
  models.
\newblock \emph{arXiv preprint arXiv:2102.12525}, 2021.

\bibitem[Lin et~al.(2022)Lin, Lindell, Chan, and Wetzstein]{lin20223d}
Lin, C.~Z., Lindell, D.~B., Chan, E.~R., and Wetzstein, G.
\newblock 3d gan inversion for controllable portrait image animation, 2022.

\bibitem[Liu et~al.(2019)Liu, Jiang, Xiao, and Yang]{Liu_2019}
Liu, H., Jiang, B., Xiao, Y., and Yang, C.
\newblock Coherent semantic attention for image inpainting.
\newblock \emph{2019 IEEE/CVF International Conference on Computer Vision
  (ICCV)}, Oct 2019.
\newblock \doi{10.1109/iccv.2019.00427}.
\newblock URL \url{http://dx.doi.org/10.1109/ICCV.2019.00427}.

\bibitem[Liu \& Scarlett(2020)Liu and Scarlett]{inf_theoretic}
Liu, Z. and Scarlett, J.
\newblock Information-theoretic lower bounds for compressive sensing with
  generative models.
\newblock \emph{IEEE Journal on Selected Areas in Information Theory},
  1\penalty0 (1):\penalty0 292–303, May 2020.
\newblock ISSN 2641-8770.
\newblock \doi{10.1109/jsait.2020.2980676}.
\newblock URL \url{http://dx.doi.org/10.1109/JSAIT.2020.2980676}.

\bibitem[Meng et~al.(2021)Meng, He, Song, Song, Wu, Zhu, and
  Ermon]{meng2021sdedit}
Meng, C., He, Y., Song, Y., Song, J., Wu, J., Zhu, J.-Y., and Ermon, S.
\newblock Sdedit: Guided image synthesis and editing with stochastic
  differential equations, 2021.

\bibitem[Menon et~al.(2020)Menon, Damian, Hu, Ravi, and Rudin]{pulse}
Menon, S., Damian, A., Hu, S., Ravi, N., and Rudin, C.
\newblock Pulse: Self-supervised photo upsampling via latent space exploration
  of generative models.
\newblock \emph{2020 IEEE/CVF Conference on Computer Vision and Pattern
  Recognition (CVPR)}, 2020.
\newblock \doi{10.1109/cvpr42600.2020.00251}.
\newblock URL \url{http://dx.doi.org/10.1109/cvpr42600.2020.00251}.

\bibitem[Nguyen et~al.(2021)Nguyen, Jagatap, and Hegde]{nguyen2021provable}
Nguyen, T.~V., Jagatap, G., and Hegde, C.
\newblock Provable compressed sensing with generative priors via langevin
  dynamics, 2021.

\bibitem[Nichol \& Dhariwal(2021)Nichol and Dhariwal]{nichol2021improved}
Nichol, A. and Dhariwal, P.
\newblock Improved denoising diffusion probabilistic models, 2021.

\bibitem[Nichol et~al.(2021)Nichol, Dhariwal, Ramesh, Shyam, Mishkin, McGrew,
  Sutskever, and Chen]{nichol2021glide}
Nichol, A., Dhariwal, P., Ramesh, A., Shyam, P., Mishkin, P., McGrew, B.,
  Sutskever, I., and Chen, M.
\newblock Glide: Towards photorealistic image generation and editing with
  text-guided diffusion models, 2021.

\bibitem[Ongie et~al.(2020)Ongie, Jalal, Metzler, Baraniuk, Dimakis, and
  Willett]{ongie2020deep}
Ongie, G., Jalal, A., Metzler, C.~A., Baraniuk, R.~G., Dimakis, A.~G., and
  Willett, R.
\newblock Deep learning techniques for inverse problems in imaging.
\newblock \emph{IEEE Journal on Selected Areas in Information Theory},
  1\penalty0 (1):\penalty0 39--56, 2020.

\bibitem[Park et~al.(2020)Park, Smedemark-Margulies, Daniels, Yu, van~de Meent,
  and HAnd]{surgery}
Park, J.~Y., Smedemark-Margulies, N., Daniels, M., Yu, R., van~de Meent, J.-W.,
  and HAnd, P.
\newblock Generator surgery for compressed sensing.
\newblock In \emph{NeurIPS 2020 Workshop on Deep Learning and Inverse
  Problems}, 2020.
\newblock URL \url{https://openreview.net/forum?id=s2EucjZ6d2s}.

\bibitem[Radford et~al.(2021)Radford, Kim, Hallacy, Ramesh, Goh, Agarwal,
  Sastry, Askell, Mishkin, Clark, Krueger, and Sutskever]{clip}
Radford, A., Kim, J.~W., Hallacy, C., Ramesh, A., Goh, G., Agarwal, S., Sastry,
  G., Askell, A., Mishkin, P., Clark, J., Krueger, G., and Sutskever, I.
\newblock Learning transferable visual models from natural language
  supervision, 2021.

\bibitem[Raginsky et~al.(2017)Raginsky, Rakhlin, and
  Telgarsky]{raginsky2017non}
Raginsky, M., Rakhlin, A., and Telgarsky, M.
\newblock Non-convex learning via stochastic gradient langevin dynamics: a
  nonasymptotic analysis.
\newblock In \emph{Conference on Learning Theory}, pp.\  1674--1703. PMLR,
  2017.

\bibitem[Ramesh et~al.(2022)Ramesh, Dhariwal, Nichol, Chu, and Chen]{dalle2}
Ramesh, A., Dhariwal, P., Nichol, A., Chu, C., and Chen, M.
\newblock Hierarchical text-conditional image generation with clip latents.
\newblock \emph{arXiv preprint arXiv:2204.06125}, 2022.

\bibitem[Ronneberger et~al.(2015)Ronneberger, Fischer, and Brox]{unet}
Ronneberger, O., Fischer, P., and Brox, T.
\newblock U-net: Convolutional networks for biomedical image segmentation.
\newblock In Navab, N., Hornegger, J., Wells, W.~M., and Frangi, A.~F. (eds.),
  \emph{Medical Image Computing and Computer-Assisted Intervention -- MICCAI
  2015}, pp.\  234--241, Cham, 2015. Springer International Publishing.
\newblock ISBN 978-3-319-24574-4.

\bibitem[Saharia et~al.(2022)Saharia, Chan, Saxena, Li, Whang, Denton,
  Ghasemipour, Ayan, Mahdavi, Lopes, Salimans, Ho, Fleet, and Norouzi]{imagen}
Saharia, C., Chan, W., Saxena, S., Li, L., Whang, J., Denton, E., Ghasemipour,
  S. K.~S., Ayan, B.~K., Mahdavi, S.~S., Lopes, R.~G., Salimans, T., Ho, J.,
  Fleet, D.~J., and Norouzi, M.
\newblock Photorealistic text-to-image diffusion models with deep language
  understanding, 2022.

\bibitem[Salimans \& Ho(2022)Salimans and Ho]{salimans2022progressive}
Salimans, T. and Ho, J.
\newblock Progressive distillation for fast sampling of diffusion models.
\newblock \emph{arXiv preprint arXiv:2202.00512}, 2022.

\bibitem[Santurkar et~al.(2019)Santurkar, Tsipras, Tran, Ilyas, Engstrom, and
  Madry]{santurkar2019image}
Santurkar, S., Tsipras, D., Tran, B., Ilyas, A., Engstrom, L., and Madry, A.
\newblock Image synthesis with a single (robust) classifier, 2019.

\bibitem[Song et~al.(2020)Song, Meng, and Ermon]{song2020denoising}
Song, J., Meng, C., and Ermon, S.
\newblock Denoising diffusion implicit models, 2020.

\bibitem[Song \& Ermon(2019)Song and Ermon]{score_first}
Song, Y. and Ermon, S.
\newblock Generative modeling by estimating gradients of the data distribution.
\newblock In Wallach, H., Larochelle, H., Beygelzimer, A., d\textquotesingle
  Alch\'{e}-Buc, F., Fox, E., and Garnett, R. (eds.), \emph{Advances in Neural
  Information Processing Systems}, volume~32. Curran Associates, Inc., 2019.
\newblock URL
  \url{https://proceedings.neurips.cc/paper/2019/file/3001ef257407d5a371a96dcd947c7d93-Paper.pdf}.

\bibitem[Song \& Ermon(2020)Song and Ermon]{song2020improved}
Song, Y. and Ermon, S.
\newblock Improved techniques for training score-based generative models, 2020.

\bibitem[Song et~al.(2021{\natexlab{a}})Song, Shen, Xing, and
  Ermon]{song2021solving}
Song, Y., Shen, L., Xing, L., and Ermon, S.
\newblock Solving inverse problems in medical imaging with score-based
  generative models, 2021{\natexlab{a}}.

\bibitem[Song et~al.(2021{\natexlab{b}})Song, Sohl-Dickstein, Kingma, Kumar,
  Ermon, and Poole]{song_sde}
Song, Y., Sohl-Dickstein, J., Kingma, D.~P., Kumar, A., Ermon, S., and Poole,
  B.
\newblock Score-based generative modeling through stochastic differential
  equations.
\newblock In \emph{International Conference on Learning Representations},
  2021{\natexlab{b}}.
\newblock URL \url{https://openreview.net/forum?id=PxTIG12RRHS}.

\bibitem[Sun \& Chen(2020)Sun and Chen]{Sun_2020}
Sun, W. and Chen, Z.
\newblock Learned image downscaling for upscaling using content adaptive
  resampler.
\newblock \emph{IEEE Transactions on Image Processing}, 29:\penalty0
  4027–4040, 2020.
\newblock ISSN 1941-0042.
\newblock \doi{10.1109/tip.2020.2970248}.
\newblock URL \url{http://dx.doi.org/10.1109/TIP.2020.2970248}.

\bibitem[Sun et~al.(2020)Sun, Liu, and Kamilov]{Sun2020}
Sun, Y., Liu, J., and Kamilov, U.~S.
\newblock Block coordinate regularization by denoising.
\newblock \emph{IEEE Transactions on Computational Imaging}, 6:\penalty0
  908–921, 2020.
\newblock ISSN 2573-0436.
\newblock \doi{10.1109/tci.2020.2996385}.
\newblock URL \url{http://dx.doi.org/10.1109/TCI.2020.2996385}.

\bibitem[Tancik et~al.(2020)Tancik, Srinivasan, Mildenhall, Fridovich-Keil,
  Raghavan, Singhal, Ramamoorthi, Barron, and Ng]{tancik2020fourier}
Tancik, M., Srinivasan, P.~P., Mildenhall, B., Fridovich-Keil, S., Raghavan,
  N., Singhal, U., Ramamoorthi, R., Barron, J.~T., and Ng, R.
\newblock Fourier features let networks learn high frequency functions in low
  dimensional domains, 2020.

\bibitem[Tian et~al.(2020)Tian, Fei, Zheng, Xu, Zuo, and Lin]{Tian_2020}
Tian, C., Fei, L., Zheng, W., Xu, Y., Zuo, W., and Lin, C.-W.
\newblock Deep learning on image denoising: An overview.
\newblock \emph{Neural Networks}, 131:\penalty0 251–275, Nov 2020.
\newblock ISSN 0893-6080.
\newblock \doi{10.1016/j.neunet.2020.07.025}.
\newblock URL \url{http://dx.doi.org/10.1016/j.neunet.2020.07.025}.

\bibitem[Tripathi et~al.(2018)Tripathi, Lipton, and
  Nguyen]{tripathi2018correction}
Tripathi, S., Lipton, Z.~C., and Nguyen, T.~Q.
\newblock Correction by projection: Denoising images with generative
  adversarial networks.
\newblock \emph{arXiv preprint arXiv:1803.04477}, 2018.

\bibitem[Tzen \& Raginsky(2019)Tzen and Raginsky]{tzen2019theoretical}
Tzen, B. and Raginsky, M.
\newblock Theoretical guarantees for sampling and inference in generative
  models with latent diffusions.
\newblock In \emph{Conference on Learning Theory}, pp.\  3084--3114. PMLR,
  2019.

\bibitem[Vahdat et~al.(2021)Vahdat, Kreis, and Kautz]{vahdat2021score}
Vahdat, A., Kreis, K., and Kautz, J.
\newblock Score-based generative modeling in latent space.
\newblock \emph{arXiv preprint arXiv:2106.05931}, 2021.

\bibitem[Vershynin(2018)]{vershynin2018high}
Vershynin, R.
\newblock \emph{High-dimensional probability: An introduction with applications
  in data science}, volume~47.
\newblock Cambridge university press, 2018.

\bibitem[Watson et~al.(2021)Watson, Ho, Norouzi, and Chan]{watson2021learning}
Watson, D., Ho, J., Norouzi, M., and Chan, W.
\newblock Learning to efficiently sample from diffusion probabilistic models,
  2021.

\bibitem[Whang et~al.(2021)Whang, Delbracio, Talebi, Saharia, Dimakis, and
  Milanfar]{whang2021deblurring}
Whang, J., Delbracio, M., Talebi, H., Saharia, C., Dimakis, A.~G., and
  Milanfar, P.
\newblock Deblurring via stochastic refinement, 2021.

\bibitem[Xiao et~al.(2021)Xiao, Kreis, and Vahdat]{xiao2021tackling}
Xiao, Z., Kreis, K., and Vahdat, A.
\newblock Tackling the generative learning trilemma with denoising diffusion
  gans, 2021.

\bibitem[Xu et~al.(2017)Xu, Chen, Zou, and Gu]{xu2017global}
Xu, P., Chen, J., Zou, D., and Gu, Q.
\newblock Global convergence of langevin dynamics based algorithms for
  nonconvex optimization, 2017.

\bibitem[Yang et~al.(2019)Yang, Zhang, Tian, Wang, Xue, and Liao]{yang2019deep}
Yang, W., Zhang, X., Tian, Y., Wang, W., Xue, J.-H., and Liao, Q.
\newblock Deep learning for single image super-resolution: A brief review.
\newblock \emph{IEEE Transactions on Multimedia}, 21\penalty0 (12):\penalty0
  3106--3121, 2019.

\bibitem[Yu et~al.(2019)Yu, Lin, Yang, Shen, Lu, and Huang]{Yu_2019}
Yu, J., Lin, Z., Yang, J., Shen, X., Lu, X., and Huang, T.
\newblock Free-form image inpainting with gated convolution.
\newblock \emph{2019 IEEE/CVF International Conference on Computer Vision
  (ICCV)}, Oct 2019.
\newblock \doi{10.1109/iccv.2019.00457}.
\newblock URL \url{http://dx.doi.org/10.1109/ICCV.2019.00457}.

\end{thebibliography}
\bibliographystyle{icml2022}

\newpage
\appendix
\onecolumn

\newglossaryentry{l2_norm}
{
    name=$||x||$,
    description={$l_2$ norm of vector $x$}
}

\newglossaryentry{g_synthesis}{
    name=$g^{(d)}$,
    description={Synthesis of scalar function $g$ with itself, $d$ times}
}

\newglossaryentry{markov_density}{
    name=$v\mathbb P^t_X$,
    description={Probability Density of a Markov Chain at time $t$}
}

\newglossaryentry{markov_path_density}{
    name=$\mathbb P^{x, t}_X$,
    description={Probability Density of a continuous Markov Chain that started from point $x$ on the space of all continuous paths on $[0, t]$}
}

\printglossary[title=Notation, toctitle=List of terms]
\section{Formal statement}

In the theorem, we make use of a $d$-layer ReLU network $G$,
\begin{gather}
    G(z) = \relu\left(W^{(d)} \left(\cdots\relu\left(W^{(1)}z \right) \cdots\right) \right)\notag
\end{gather}
where each $w_i \in \mathbb{R}^{n_i \times n_{i-1}}$. Further, we make use of a matrix $A$ of dimension $n_d\times k$. We use the same assumptions on $G$ and $A$ as in \citet{paul_hand_v1}. Our setting has a minor difference compared to \citet{paul_hand_v1}: to remove unnecessary scalings, we scale the distribution of the weights by a factor of $2$ at every layer.

\begin{definition}
We say that the matrix $W \in \R^{n \times k}$ satisfies the \emph{Weight Distribution Condition} with constant $\eps$ if for all nonzero $x,y \in \R^k$, 
\begin{align}
\Bigl \| \sum_{i=1}^n \indwx \indwy \cdot w_i w_i^t  - \Qxy \Bigr \| \leq \eps,  \text{ with } \Qxy = \frac{\pi - \theta_0}{2 \pi} I_k + \frac{\sin \theta_0}{2\pi}  \Mxyhat, \label{WDC}
\end{align}
where $w_i \in \R^k$ is the $i$th row of $W$; $\Mxyhat \in \R^{k \times k}$ is the matrix\footnote{A formula for $\Mxyhat$ is as follows.  If $\theta_0 = \angle(\xhat, \yhat) \in (0, \pi)$ and $R$ is a rotation matrix such that $\xhat$ and $\yhat$ map to $e_1$ and $\cos \theta_0 \cdot e_1 + \sin \theta_0 \cdot e_2$ respectively, then $\Mxyhat = R^t \begin{pmatrix} \cos \theta_0 & \sin \theta_0 & 0 \\ \sin \theta_0 & - \cos \theta_0 & 0 \\ 0 & 0 & 0_{k-2} \end{pmatrix} R$, where $0_{k-2}$ is a $k-2 \times k-2$ matrix of zeros.  If $\theta_0 = 0$ or $\pi$, then $\Mxyhat = \xhat \xhat^t$ or $- \xhat \xhat^t$, respectively.} such that $\xhat \mapsto \yhat$, $\yhat \mapsto \xhat$, and $z \mapsto 0$ for all $z \in \Span(\{x,y\})^\perp$;  $\xhat = x/\|x\|_2$  and $\yhat = y /\|y\|_2$;  $\theta_0 = \angle(x, y)$; and $1_S$ is the indicator function on $S$. 
\end{definition}
\begin{definition}
We say that the compression matrix $A \in \R^{m \times n}$ satisfies the \emph{Range Restricted Isometry Condition (RRIC)} with respect to $G$ with constant $\eps$ if for all $x_1, x_2, x_3, x_4 \in \R^k$,
\begin{align}
\Bigl| \Bigl \langle A \bigl( G(x_1) - G(x_2) \bigr), A \bigl( G(x_3) - G(x_4) \bigr) \Bigr\rangle  &- \Bigl \langle  G(x_1) - G(x_2) ,  G(x_3) - G(x_4)  \Bigr \rangle \Bigr| \notag \\&\leq \eps \| G(x_1) - G(x_2) \|_2  \| G(x_3) - G(x_4) \|_2.
\end{align}
\end{definition}
We assume that each matrix $W^{(i)}$ in the network $G$ satisfies WDC and that the matrix $A$ satisfies RRIC. Such assumptions hold for random matrices:
\begin{theorem}
Let $\epsilon>0$. Suppose that each entry of each weight matrix $W^{(i)} \in \mathbb{R}^{n_i \times n_{i-1}}$ is drawn i.i.d. $N(0,1)$, and suppose that each entry of $A\in \mathbb{R}^{m\times k}$ is drawn i.i.d. Further, suppose that for all $i=1,\dots,d$, $n_i/n_{i-1} \ge C\epsilon^{-2}\log(1/\epsilon)$ and suppose that $m \ge C \epsilon^{-1} \log(1/\epsilon) dn \log(\prod_{i=1}^d n_i)$, where $n=n_0$ and $C>0$ is a universal constant. Then, with probability $1-e^{-cn}$, WDC is satisfied for all matrices $W^{(i)}$ and RRIC is satisfied for $A$.
\end{theorem}

We prove the following theorem, which assumes that WDC and RRIC are satisfied:
\begin{theorem}
Consider the Markov Chain defined by the following Langevin Dynamics:
\begin{gather}
    z_{t+1} = 
    z_t - \eta \nabla f(z_t) + \sqrt{2\eta \beta^{-1}}u
\end{gather}
where $u \sim N(0,I_n)$ is a zero-mean, unit variance Gaussian vector, $G(z)$ is a fully-connected $d$-layer ReLU neural network,
\begin{gather}
    G(z) = \relu\left(W^{(d)} \left(\cdots\relu\left(W^{(1)}z \right) \cdots\right) \right)\notag
\end{gather}
and $f(z)$ is the loss function:
\[
    f(z) = \beta\|AG(z) - y\|_2^2
\]
where $A \in \R^{n_d\times k}$, and $y=AG(z^*)$, for some unknown vector $z^* \in \mathbb{R}^n$.

Define $\mu(z) \propto e^{-f(z)}$ and denote by $Z_t$ the distribution over $z_t$, i.e. the $t$'th step of the dynamics. Then, there exist constants $C_1,C_2,C_3,C_4$ that depend only on $d$ such that the following holds: For any $\epsilon >0$ and for $t \ge C_1\log(1/\epsilon)/\epsilon^2$,
\begin{align*}
\mathcal W(Z_t, \mu)
:= \inf_{\text{$Q\in\{$couplings of $Z_t,\mu\}$}} \mathbb{E}_{(z_t,z)\sim Q}\|z_t-z\|\\
\le (\epsilon + e^{-C_2 n}) \|z^*\|,
\end{align*}
provided that $C_3 \epsilon^2 \le  \eta \le 1000C_3\epsilon^2$, that $\beta = C_4n$, that $\|z_0\|\le 1000\|z^*\|$, that $W^{(i)}$ and $A$ satisfy conditions WDC and RRIC with constant $\mathrm{poly}(\epsilon)$ and that $d\ge 2$. (above, 1000 can be replaced by any other constant)
\label{main_theorem_formal}
\end{theorem}

\section{Proof}

\subsection{Overview}

We start by some preliminaries and definitions in Sections~\ref{sec:pr-prel}, \ref{sec:pr-langevin} and \ref{sec:pr-setting}. Then, we analyze the loss function in Section~\ref{sec:pr-loss}. Next, we show that with high probability, the norm of the iterates will neither be very small nor very large, in Section~\ref{sec:escape} and Section~\ref{sec:norm-bounded}. Afterwards, we define a modified loss function, that is changed around the origin, in Section~\ref{sec:modified-loss}. This change is necessary because the original loss function is not smooth next to the origin, and it is significantly simpler to analyze smooth losses. Since the norm of the iterate is not small with high probability, this change will not be apparent in the dynamics. Later, we would like to argue that the iterates will converge to some region around the optimum. In Section~\ref{sec:potential-sufficient} we argue that in order to show that the iterates converge to some region, it is sufficient to construct a potential function whose Laplacian is negative outside this region. Then, in Section~\ref{sec:def-potential} we define such potential, concluding that the iterates converge to some region around the optimum. In that region, the function is convex. 
In Section~\ref{sec:epsilon-dynamics} we show that if the function is convex, then we can couple the continuous and discrete iterations such that they get closer and closer. In order to apply this argument, we have to guarantee that the iterates do not leave the convexity region. Indeed, in Section~\ref{sec:staying} we show that if the iterates are already in a convexity region, they will stay there, enabling them to get closer, until they are $\epsilon$-apart.

\subsection{Preliminaries on polar coordinates}\label{sec:pr-prel}

We start with some preliminaries. Denote by $\vec{\cdot}$ any vector from $\mathbb{R}^n$. Denote $ e_1 = (1,0,\dots,0)$. Given any $\x$, denote by $r=r(\x)=\|\x\|_2$, and denote by 
\[
\theta=\theta(\x) = \arccos \l(\frac{\langle x, e_1\rangle}{\|\x\|\|e_1\|}\r) = \arccos\l(\frac{\langle \x, e_1 \rangle}{\|x\|}\r)
\]
the angle $\theta \in [0,\pi]$ between $\x$ and $e_1$. Denote
\[
\vr = \vec{r}(\x) = \nabla_{\x} r(\x) = \frac{\x}{\|\x\|}
\]
and
\[
\vtheta = \vtheta(\x) = \frac{\nabla_{\x} \theta(\x)}{\|\nabla_{\x}\theta(\x)\|} = \frac{\cos(\theta(\x))}{\sin(\theta(\x))} \frac{\x}{\|\x\|} - \frac{e_1}{\sin(\theta(\x))}\enspace.
\]
We will use $\theta,r,\vtheta,\vr$ without writing $\x$ when $\x$ is clear from context.
We have the following properties:
\begin{lemma}\label{lem:polar}
Let $\x \in \mathbb{R}^n$, denote by $V^\perp$ the vector space that is the orthogonal complement to $\mathrm{span}(\vtheta(\x),\vr(\x))$ as a subspace of $\mathbb{R}^n$ and let $\vpsi_1,\dots,\vpsi_{n-2}$ denote an orthonormal basis of $V^\perp$. Then:
\begin{enumerate}
    \item The set $\{\vr(\x),\vtheta(\x),\vpsi_1,\dots,\vpsi_{n-2}\}$ forms an orthonormal basis to $\mathbb{R}^n$. In particular, $\|\vr(\x)\|=\|\vtheta(\x)\|=1$ and $\langle \vr(\x),\vtheta(\x)\rangle=0$.
    \item Let $f \colon [0,\infty)\times [0,\pi] \to \mathbb{R}$ be a $C^2$ function. Denote  
    \[
    f_r = \frac{\partial f}{\partial r}, \quad
    f_{\theta} = \frac{\partial f}{\partial \theta}, \quad 
    f_{rr} = \frac{\partial^2 f}{\partial r^2}, \quad
    f_{\theta\theta} = \frac{\partial^2 f}{\partial \theta^2}, \quad
    f_{r\theta} = f_{\theta r} = \frac{\partial^2 f}{\partial r\partial \theta}.
    \]
    Denote $r = r(\x)$, $\theta = \theta(\x)$, $\vr = \vr(\x)$, $\vtheta = \vtheta(\x)$, $f_r=f_r(r,\theta)$, $f_\theta(r,\theta)$ etc. Then,
    \[
    \nabla_{\x} f(r(\x),\theta(\x)) = f_r \vr
    + \frac{f_\theta}{r} \vtheta.
    \]
    Further, the Hessian of $f(r(\x),\theta(\x))$ with respect to $\x$ equals
    \[
    \nabla^2_{\x} f(r(\x),\theta(\x))
    = f_{rr} \vr \vr^\top
    + \l(\frac{f_r}{r} + \frac{f_{\theta\theta}}{r^2} \r) \vtheta \vtheta^\top + \l(\frac{f_{r\theta}}{r} - \frac{f_\theta}{r^2} \r) \l(\vr\vtheta^\top + \vtheta\vr^\top \r) + \l(\frac{f_r}{r} + \frac{f_\theta}{r^2 \tan(\theta)} \r) \l(\sum_{i=1}^{n-2}\vpsi_i\vpsi_i^\top \r)\enspace.
    \]
    \item It holds that
    \[
    \triangle f = \sum_{i=1}^n \frac{d^2 f}{d x_i^2} = f_{rr} + \l(\frac{f_r}{r} + \frac{f_{\theta\theta}}{r^2} \r) + (n-2) \l(\frac{f_r}{r} + \frac{f_\theta}{r^2 \tan(\theta)} \r)
    = f_{rr} + (n-1) \frac{f_r}{r} + \frac{f_{\theta\theta}}{r^2} + (n-2) \frac{f_\theta}{r^2\tan \theta}
    \enspace.
    \]
    \item Assume that for all $\x$,
    \[
    \max\l(
    \l| f_{rr} \r|, \l|\frac{f_r}{r} + \frac{f_{\theta\theta}}{r^2} \r|,
    \l|\frac{f_{r\theta}}{r} - \frac{f_\theta}{r^2}\r|
    ,
    \l|\frac{f_r}{r} + \frac{f_\theta}{r^2 \tan(\theta)} \r|\r)  \le s/2\enspace.
    \]
    Then, $f$ is $s$-smooth.
\end{enumerate}
\end{lemma}
\begin{proof}
The first two items are folklore, and follow from a simple application of the chain rule. The third item follows from the fact that for any orthonormal basis $B = \{\v u_1,\dots \v u_n\}$ of $\mathbb{R}^n$, $\triangle f$ equals the trace of the Hessian of $f$, computed with respect to the basis $B$. In particular, the entries of the Hessian with respect to the basis $\{\vr,\vtheta, \v \psi_1,\dots,\v \psi_{n-2}\}$ are computed in item 2 and the trace equals the formula in item 3, as required. For the forth item, it holds that $f$ is $s$-smooth if the spectral norm of $\nabla^2 f$ is bounded by $s$, while the Hessian $\nabla^2 f$ can be computed with respect to any orthonormal basis. We will write the Hessian with respect to the basis defined in item 1, and the Hessian's coefficients are computed in item 2. Since the Hessian is symmetric, its spectral norm is bounded by the $\infty$-norm, which is the maximum over the rows of the sum of absolute values, namely, $\|A\|_\infty = \max_i \sum_j |A_{ij}|$. The infinity norm of the Hessian is bounded by $s$, using the formula computed in item 2 and using the assumption of item 4.
\end{proof}

\subsection{Definitions of Langevin dynamics}\label{sec:pr-langevin}

Assuming some potential function $H \colon \mathbb{R}^n \to \mathbb{R}$ and parameters $\eta,\beta$. The Langevin dynamics can be defined by
\[
\x_t = \x_{t-1} - \eta \nabla H(\x) + \vec{z}_t
\]
where $\vec{z}_t\sim N(\vec 0, \sigma^2 I_n)$, $\sigma^2 := 2\eta/\beta$. It is known that in the limit where $\eta \to 0$ (and under some regularity assumptions) the distribution of $\x_t$ as $t\to \infty$ converges to
\[
\mu_{\beta H}(\x) := \frac{e^{-\beta H(\x)}}{\int_{\mathbb{R}^n} e^{-\beta H(\v y)} d\v y}\enspace.
\]
Denote the distribution of $\x_t$ by $\mu_{\beta H, t}$.

\subsection{Setting}\label{sec:pr-setting}

We will use the same network as suggested by \citet{huang2021provably} (multiplying the weights by 2 for convenience): given an input $\x$, the network is given by
\[
G(\x) = \relu(2W^d(\cdots(\relu(2W^1(\x)))\cdots)).
\]
A compressive map $A$ is applied on the outcome, for obtaining an output of $AG(\x)$. The goal is to recover some unknown $G(\x^*)$, given the measurement $AG(\x^*)$. We assume that each $W^i$ satisfies the Weight Distribution Condition and that $A$ satisfies Range Restricted Isometry Condition \cite{huang2021provably}, both with parameter $\delta$.\footnote{The weight Distribution Condition holds with high probability for isotropic Gaussian matrices with constant expansion: namely, when the output dimension of each layer is at least a constant times larger than the input dimension. The required expansion constant depends on $\delta$. Further, the Range Restricted Isometry Condition holds with high probability for matrices $A$ with a constant output dimension, where the constant depends on $\delta$  \cite{huang2021provably,daskalakis2020constant}.} As shown by \cite{huang2021provably}, this implies that
\[
\forall \x,\y \colon
|\angle(\relu(2W^i(\x)),\relu(2W^i(\y))) - g(\angle(\x,\y))|\le  f(\delta),
\]
where $f(\delta)\to 0$ as $\delta\to 0$, $\angle(\x,\y) \in [0,\pi]$ is the angle between $\x$ and $\y$ and
\[
g(\theta) = \arccos\l(\frac{(\pi-\theta)\cos\theta+\sin \theta}{\pi} \r)\enspace.
\]
Further, 
\[
|\|\relu(2W^i(\x))\| - \|\x\||\le f(\delta)\|\x\|
\]
and
\[
\forall \x \colon |\|A\x\| - \|\x\||\le f(\delta)\|\x\|.
\]
The loss function is 
\[
\tilde{L}(\x) 
= \|AG(\x) - AG(\x^*)\|^2/2.
\]
Let us compute the loss function assuming that $\delta= 0$. There, $f(\delta)=0$ and further
\[
\|AG(\x) - AG(\x^*)\|^2/2 = \|G(\x) - G(\x^*)\|^2/2.
\]
Additionally, $\delta=0$ implies that $\|G\x^*\| = \|\x^*\|$, $\|G\x\|=\|\x\|$ for all $\x$ and $\theta_d(\x):= \angle(G(\x),G(\x^*)) = g^{\circ d}(\angle(\x,\x^*))$ where $g^{\circ d}$ is a composition of $g$ for $d$ iterations. Assuming that $\x^* = (1,0,\dots,0)$, and denoting by $\theta(\x) = \angle(\x,\x^*)$ we have that
\[
\|G(\x)-G(\x^*)\|^2 = 
(\|\x\|\cos \theta_d(\x)-1)^2 + (\|\x\|\sin \theta_d(\x))^2
= \|\x\|^2 - 2 \|x\|\cos \theta_d(\x)+1.
\]
Denote by $L(\x)$ the value of the loss function $\tilde{L}(\x)$ when $\delta=0$. As computed above,
\[
L(\x) = \|\x\|^2/2 - \|x\|\cos \theta_d(\x)+1/2.
\]
\citet{huang2021provably} have shown that the gradients of $L$ are close to the gradients of $\tilde{L}$, under the above assumptions on the weights, in the following sense:
\begin{equation} \label{eq:tilde-delta}
\forall \x\colon \|\nabla L(\x) - \nabla \tilde{L}(\x)\| \le (\|\x\|+1) f(\delta,d),
\end{equation}
for some $f(\delta,d)$ that decays to zero as $\delta \to 0$ while keeping $d$ fixed. 

It is sufficient to assume that $\x^*=e_1$, since Langevin dynamics is indifferent to scaling and rotations. Yet, once we consider $\x^*$ such that $\|\x^*\|\ne 1$, the error has to be multiplied by $\|\x^*\|$. 

\paragraph{Notation.}
When using $O(\cdot)$-notation, we will ignore constants that depend on the depth $d$ of the network. Given some parameter, e.g. $l>0$, we denote by $C(l)$ a constant that may depend only on $l$ (and perhaps also on the depth $d$), but not on the other parameters, in particular, not on $n$. We will use $C,c$ to denote positive constants that depend only on $d$ (and perhaps on other parameters that are explicitly defined as constants).

\subsection{Properties of the loss function}\label{sec:pr-loss}

Define by $\theta'_d = \frac{dg^{\circ d}(\theta)}{d\theta}$. And similarly define $\theta''_d$ as the second derivative. The loss function $L(\x)$ can be computed as a function of $r=r(\x)$ and $\theta=\theta(\x)$, by the formula 
\[
L(r,\theta)=r^2/2 - r\cos\left(g^{\circ d}(\theta)\right) + 1/2
= r^2/2 - r\cos\theta_d + 1/2,
\] 
and the corresponding derivatives of $L$ as a function of $r$ and $\theta$ equal:
\begin{itemize}
    \item $L_r = r - \cos \theta_d$
    \item $L_\theta = r\sin \theta_d \theta'_d$
    \item $L_{rr} = 1$
    \item $L_{\theta\theta} = r\cos \theta_d (\theta'_d)^2 + r\sin \theta_d\theta''_d$
    \item $L_{r\theta} = \sin\theta_d\theta'_d$
\end{itemize}
Consequently, we have by Lemma~\ref{lem:polar}:
\[
\nabla L(\x) = (r - \cos \theta_d)\vr + \sin \theta_d \theta'_d \vtheta,
\]
and
\begin{gather}
\nabla^2 L = 
\vec r \vec r^\top
+ \frac{r - \cos \theta_d + \cos \theta_d (\theta'_d)^2 + \sin \theta_d \theta''_d}{r} \vec{\theta}\vec{\theta}^\top
+ \sum_{i=1}^{n-2}\frac{(r-\cos \theta_d)\sin \theta
	+  \sin \theta_d \theta'_d \cos \theta}{r\sin \theta} \vec{\psi}_i \vec{\psi}_i^\top \enspace. 
	\label{eq:hessian_3d}
\end{gather}

Before proceeding, let us analyze the function $g$ and consequently, $\theta'_d$ and $\theta''_d$:
\begin{lemma} \label{lem:g}
Let $g'$ and $g''$ denote the first and second derivatives of $g$. Then, $g$ is decreasing, $g'(\theta) \in [0,1]$ and $g''(\theta)\le 0$. Consequently, $\theta'_d \in [0,1]$ while $\theta''_d \le 0$.
\end{lemma}
\begin{proof}
The properties of $g$ follow directly by computing the derivatives of $g$ and they were analyzed by \citet{huang2021provably}. The derivative of $\theta_d$ can be computed using the composition rule:
\begin{equation}\label{eq:theta-prime}
\theta'_d = \frac{dg^{\circ d}(\theta)}{d\theta}
= \prod_{i=0}^{d-1} g'(g^{\circ i}(\theta)) \in [0,1].
\end{equation}
Notice that $\theta''_d$ is the derivative of \eqref{eq:theta-prime} and it is non-positive as $g''(\theta)\le 0$.
\end{proof}

We have the following:
\begin{lemma}\label{lem:smooth-lip}
Let $0<r<R$. Then the loss $L$ is $C(r,R)$-Lipschitz in $K=\{\|x\| \colon r\le \|\x\|\le R\}$ and $C(r)$-smooth in $K' = \{\|x\| \colon r\le \|\x\|\}$.
\end{lemma}
\begin{proof}
First of all, we prove that the function is Lipschitz and smooth in $K$, and then we extend the smoothness to $K'$. It is sufficient to use Lemma~\ref{lem:polar}, and argue that the coefficients in the expansion of $\nabla L$ and $\nabla^2 L$ are bounded in absolute value in $K$. First, it is easy to verify that the two derivatives of $g$ are finite, which implies that $\theta'_d,\theta''_d$ are bounded. Further, $r$ is bounded from below by assumption, hence, the only coefficient that could possibly go to infinity is
\[
\frac{(r-\cos \theta_d)\sin \theta
	+  \sin \theta_d \theta'_d \cos \theta}{r\sin \theta}
	= \frac{(r-\cos \theta_d)}{r}
	+ \frac{\sin \theta_d \theta'_d \cos \theta}{r\sin \theta}.
\]
In fact, the quantity that can possibly go to infinity in $K$ is 
\begin{equation}\label{eq:tobound-secondder}
\frac{\sin \theta_d \theta'_d \cos \theta}{r\sin \theta}.
\end{equation}
since it has $\sin \theta$ in its denominator and $\sin\theta$ can be zero. Yet we would like to argue that when the denominator goes to zero, the numerator goes to $0$ as well and the ratio does not go to infinity. Notice that the numerator contains the term $\theta_d = g^{\circ d}(\theta)$. Since $g(\theta)\le \theta$ \cite{huang2021provably} we derive that $\theta_d = g^{\circ d}(\theta) \le \theta$. If $\theta\le \pi/2$ then we have $0 \le \sin \theta_d \le \sin \theta$ which implies that the ratio in \eqref{eq:tobound-secondder} is bounded. Otherwise, $\theta \ge \pi/2$ and the denominator can go to $0$ only when $\theta \to \pi$. We would like to argue that the numerator also goes to $0$ in this case. Indeed, 
\[
\theta'_d = \frac{dg^{\circ d}(\theta)}{d\theta}
= \prod_{i=0}^{d-1} g'(g^{\circ i}(\theta)) \le g'(\theta),
\]
since $0\le g'\le 1$ \cite{huang2021provably}, where $g'(\theta) = dg(\theta)/d\theta$. Further, $g'(\theta)\to 0$ as $\theta\to \pi$. Hence, by L'Hopital's rule, 
\[
\limsup_{\theta\to\pi} \frac{\theta'_d}{\sin \theta}
= \limsup_{\theta\to \pi}
\frac{\theta''_d}{\cos \theta} < \infty,
\]
using the fact that $\theta''_d$ as argued above.
By continuity, $\theta'_d/\sin\theta$ is uniformly bounded in $[0,\pi]$, which implies that \eqref{eq:tobound-secondder} is uniformly bounded, as required.

Notice that the smoothness holds also over all of $K'$, since, from the form of the second derivative and the arguments above, it is clear that these do not go to $\infty$ as $r\to \infty$.
\end{proof}

Further, we use the following lemma from \citet{huang2021provably}:
\begin{lemma}
There exist only three points where the gradient of $L$ possibly equals $0$: at the optimum $\x^*$, at $\x = -\x^*\cos g^{\circ d}(\pi)$, and at $0$.
\end{lemma}
We note the at $0$ there is a local max, at $\x^*$ a local min and at $-\x^*\cos g^{\circ d}(\pi)$ a saddle point, that is a minimum with respect to $r$ and a maximum with respect to $\theta$.

We add the following lemma:
\begin{lemma}\label{lem:grad-large}
Let $l>0$. There exists a constant $C(l)>0$ (independent of $n$) such that if $\x$ is at least $l$-far apart from any point where the gradient vanishes (in $l_2$-norm), then $\|\nabla L(\x)\|\ge C(l)$.
\end{lemma}
\begin{proof}
Notice first that as $\|\x\|\to \infty$ then the gradient-norm goes to infinity. Hence, it is sufficient to assume that $\|\x\|\le M$ for some sufficiently large $M$. Secondly, notice that both the distance of $\x$ from any stationary point, and the norm of its gradient, are only functions of $r(\x)=\|\x\|$ and $\theta(\x)$. Let $K$ be the set of pairs $(r,\theta)$ such that (1) $r \le M$ and (2) $(r,\theta)$ signify a point of distance at least $l$ from any stationary point. This set is compact, hence the continuous function $(r,\theta) \to \|\nabla L(r,\theta)\|$ has a minimum in $K$, which is non-zero since we assumed that $K$ does not contain any stationary point.
\end{proof}

Further, we have the following lemma:
\begin{lemma}\label{lem:strongly-around-opt}
There exists some constant $l>0$ such the function is $0.9$-strongly convex in a ball of radius $\ell$ around $\x^*$.
\end{lemma}
\begin{proof}
First of all, we will prove that the Hessian is PSD at $\x^*$. For this purpose, it is sufficient to prove that all the coefficients in \eqref{eq:hessian_3d} are positive at $\x^*$, since the basis $\vr,\vtheta,\v\psi_1,\cdots,\v\psi_{n-1}$ is orthonormal, as stated in Lemma~\ref{lem:polar}. The coefficient that multiplies $\vr\vr^\top$ is $1>0$. The second coefficient is
\[
\frac{r-\cos\theta_d+\cos\theta_d(\theta'_d)^2 +\sin\theta_d\theta''_d}{r}.
\]
We have that $r(\x^*)=1$, $\theta(\x^*) = 0$, and $\theta_d(\x^*) = g^{\circ d}(\theta(\x^*)) = g^{\circ d}(0) = 0$ since $g(0)=0$.
Further, $\theta'_d(\x^*) = 1$ as $\frac{dg(\theta)}{d\theta}\big|_{\theta=0}=1$.
Hence,
\[
\frac{r-\cos \theta_d + \cos \theta_d (\theta'_d)^2 + \sin \theta_d \theta''_d}{r} = \frac{1 - 1 + 1 + 0}{1}
= 1.
\]
For the last coefficient in \eqref{eq:hessian_3d}, we have
\begin{equation}\label{eq:third-term-hessian}
\frac{(r-\cos \theta_d)\sin \theta
	+  \sin \theta_d \theta'_d \cos \theta}{r\sin \theta}
	= \frac{(r-\cos \theta_d)}{r}
	+ \frac{\sin \theta_d \theta'_d \cos \theta}{r\sin \theta}.
\end{equation}
The first term is $0$, while the second term is undefined, yet, can be computed using the limit $\theta\to 0$. In particular, using the calculations above and L-Hopital's rule,
\[
\lim_{\theta\to 0} \frac{\sin \theta_d}{\sin \theta}
= \lim_{\theta\to 0}\frac{\cos \theta_d\theta'_d}{\cos \theta} = 1.
\]
In particular, the second term in \eqref{eq:third-term-hessian} equals $1$, hence $\nabla^2 L(\x^*) = I_n$.
The minimal eigenvalue of the Hessian at $\x$, which is the minimal of the three coefficients in \eqref{eq:hessian_3d}, is a function only of $r(\x)$ and $\theta(\x)$. By continuity, there exists a neighborhood $U \subseteq [0,\infty) \times [0,2\pi]$ or pairs $(r,\theta)$, that contains the point $(r,\theta) = (1,0)$, such that $\nabla^2(\x) \succeq 0.9I_n$, for any $\x$ such that $(r(\x),\theta(\x))\in U$. This proves the result.
\end{proof}

Lastly, let us analyze $\triangle L(\x)$. This will be useful later in the proof.
\begin{lemma}\label{lem:laplacian}
\[
\triangle L\le \begin{cases}
    2 + (n-2)(r-\cos \theta_d)/r & \theta \ge \pi/2 \\
    n & \theta \le \pi/2
\end{cases}
\]
\end{lemma}
\begin{proof}
Using Lemma~\ref{lem:polar}, we have that,
\begin{equation}\label{eq:expand-lap}
\triangle L
= 1 + \frac{r - \cos \theta_d + \cos \theta_d (\theta'_d)^2 + \sin \theta_d \theta''_d}{r}
+ (n-2)\frac{(r-\cos \theta_d)\sin \theta
	+ \sin \theta_d \theta'_d \cos \theta}{r\sin \theta}\enspace.
\end{equation}
First,
\begin{equation}\label{eq:bnd-first-term-lap}
\frac{r - \cos \theta_d + \cos \theta_d (\theta'_d)^2 + \sin \theta_d \theta''_d}{r}
\le 1,
\end{equation}
since $-\cos \theta_d + \cos \theta_d (\theta'_d)^2 \le 0$, and $\sin \theta_d \theta''_d \le 0$ (as follows from the fact that $g'(\theta)\in [0,1]$ hence $\theta'_d = \frac{dg^{\circ d}(\theta)}{d\theta} \in [0,1]$, and $g''(\theta)\le 0$ hence $\theta''_d \le 0$). 
To bound the last term in \eqref{eq:expand-lap}, first assume that $\theta \ge \pi/2$. Then,
\[
\frac{(r-\cos \theta_d)\sin \theta
	+ \sin \theta_d \theta'_d \cos \theta}{r\sin \theta}
	\le \frac{r-\cos \theta_d}{r}
\]
since $\sin \theta_d \ge 0$, $\theta'_d \ge 0$ and $\cos \theta \le 0$.
We conclude that
\[
\Delta L \le 2 + (n-2)\frac{r-\cos \theta_d}{r}.
\]
Next, for $\theta \le \pi/2$, we have that
\[
\frac{(r-\cos \theta_d)\sin \theta
	+ \sin \theta_d \theta'_d \cos \theta}{r\sin \theta}
	\le \frac{r-\cos \theta_d + \cos \theta}{r}
	\le 1,
\]
using $0 \le \sin \theta_d \le \sin \theta$, $\theta'_d \in [0,1]$ and $0 \le \cos \theta \le \cos \theta_d$. This concludes the proof, in combination with \eqref{eq:expand-lap} and \eqref{eq:bnd-first-term-lap}.
\end{proof}

\subsection{Escaping from the origin}\label{sec:escape}

Since the loss function is not well behaved around the origin, we want to show that the dynamics do not approach the origin with high probability, as stated below:

\begin{lemma}\label{lem:escape}
Fix $t \ge 3/\eta$, define $A = \cos g^{\circ d}(\pi)$, and assume that we run the Langevin dynamics according to $\tilde{L}$. Define $\beta = 2\eta/\sigma^2$, as in Section~\ref{sec:pr-langevin}. Then, for any $a >0$,
\[
\Pr\l[\|\x_t\| < 0.9A - a\r]
\le e^{-\beta a^2/4}.
\]
\end{lemma}

The remainder of this subsection is devoted for the proof of this Lemma. Let us write the update rule:

\begin{align*}
\x_{t+1} 
&= \x_t - \eta \nabla \tilde L(\x) + \vec{z}_{t + 1}
= \x_t - \eta \nabla L(\x) - \eta (\nabla \tilde L(\x) - \nabla L(\x)) + \vec{z}_{t+1} \\
&= \x_t - 
\eta (r-\cos \theta_d) \vec r
- \eta \sin \theta_d \theta'_d \vec \theta + \vec z_{t + 1} - \eta (\nabla \tilde L(\x) - \nabla L(\x))
\end{align*}
where $r, \vec r, \vec \theta$ etc. refer to $\x_t$ (as defined in Section~\ref{sec:pr-prel}).
We will define the following intermediate random variables, that help us transferring from $\x_{t}$ to $\x_{t+1}$:
\[
\x_t' 
= \x_t - \eta(r-\cos \theta_d)\vec{r}
= r \vec{r} - \eta(r - \cos \theta_d) \vec{r}
= (r - \eta r + \eta \cos \theta_d) \vec{r},
\]
\[
\x_t'' = \x_t'- \eta \sin \theta_d \theta'_d \vec \theta,
\]
\[
\x_t''' = \x_t'' + \langle \vec z_{t+1}, \vr(\x_t'')\rangle \vr(\x_t''),
\]
\[
\x_t^{(4)} = \x_t''' + \v z_{t+1} - \langle \vec z_{t+1}, \vr(\x_t'')\rangle \vr(\x_t'')
\]
and notice that
\[
\x_{t+1} = \x_t^{(4)} - \eta (\nabla \tilde L(\x) - \nabla L(\x)).
\]
We will lower bound $\|\x_{t+1}\|$ as a function of $x_{t}$ and of $\v z_{t+1}$. First, we will lower bound the norms of these intermediate variables.
Notice that 
\[
\|\x_t'\| = |r - \eta r + \eta \cos \theta_d| = (1 - \eta) \|\x_t\| + \eta \cos \theta_d \ge (1-\eta)\|\x_t\| + \eta A,
\]
where we use the fact by monotonicity of $g(\theta)$ (Lemma~\ref{lem:g}), 
\[
\theta_d = g^{\circ d}(\theta) \le g^{\circ d}(\pi) \Rightarrow \cos \theta_d \ge \cos g^{\circ d}(\pi) := A.
\]

Further, notice that $\x'_t$ is a multiple of $\vec{r}$, and that $\vec{r}$ and $\vec{\theta}$ are orthogonal unit vectors by Lemma~\ref{lem:polar}, hence,
\[
\|\x''_t\| = \sqrt{\|\x'_t\|^2 + \|\eta \sin \theta_d \theta'_d\|^2}
\ge \|\x'_t\|.
\]

Define $z_{t+1} = \langle \vec z_{t+1}, \vr(\x_t'')\rangle$.
Notice that
\[
\|\x_t'''\| = \|\x_t'' + z_{t+1} \vr(\x_t'')\| = \| \vr(\x_t'') (\|\x_t''\| + z_{t+1})\|
= | \|\x_t''\| + z_{t+1}|
\ge \|\x_t''\| + z_{t+1}\enspace,
\]
using the fact that by Lemma~\ref{lem:polar}, $\|\vr(\x_t'')\|=1$. 
Recall that
\[
\x^{(4)}_{t} = \x_t''' + \vec z_{t+1} - \langle \vec z_{t+1}, \vr(\x_t'')\rangle \vr(\x_t''),
\]
and notice that $\x_t'''$ is a multiple of $\vr(\x''_t)$ while $\vec z_{t+1} - \langle \vec z_{t+1}, \vr(\x_t'')\rangle \vr(\x_t'')$ is perpendicular to $\vr(\x''_t)$ (namely, their inner product is $0$). Hence,
\[
\|\x_t^{(4)}\| = \sqrt{\|\x_t'''\|^2 + \|\vec z_{t+1} - \langle \vec z_{t+1}, \vr(\x_t'')\rangle \vr(\x_t'')\|^2} \ge \|\x_t'''\|.
\]
Lastly, by the triangle inequality,
\[
\|\x_{t+1}\| = \|\x_t^{(4)} - \eta (\nabla \tilde L(\x) - \nabla L(\x))\|
\ge \|\x_t^{(4)}\| - \eta \|\nabla \tilde L(\x) - \nabla L(\x)\| \ge \|\x_t^{(4)}\| - \eta (\|\x_t\|+1)f(\delta,d),
\]
where the last inequality follows from \eqref{eq:tilde-delta} and $f(\delta,d) \to 0$ as $\delta \to 0$. In particular,
\[
\|\x_{t+1}\| \ge \|\x_t^{(4)}\| - \eta c (\|\x_t\|+1),
\]
where $c>0$ can be chosen arbitrarily small, since $\delta$ can be chosen arbitrarily small (as assumed in this lemma). Combining all the above, we have
\begin{align}
\|\x_{t+1}\| 
&\ge \|\x_t^{(4)}\| - \eta c (\|\x_t\|+1) 
\ge \|\x_t'''\| - \eta c (\|\x_t\|+1) 
\ge \|\x_t''\| + z_{t+1} - \eta c (\|\x_t\|+1)\notag\\
&\ge \|\x_t'\| + z_{t+1} - \eta c (\|\x_t\|+1)
\ge (1-\eta) \|\x_t\| + \eta A + z_{t+1} - \eta c (\|\x_t\|+1)\notag\\
&= (1 - \eta - \eta c) \|\x_t\| + \eta (A - c) + z_{t+1}, \label{eq:xt-recursive-norm}
\end{align}
where $z_{t+1} = \langle \v z_{t+1}, \vr(\x_t'')\rangle$. Since $\v z_{t+1} \sim N(\v 0, \sigma^2 I)$ and since $\|\vr(\x_t'')\|=1$ (see Section~\ref{sec:pr-prel}), it holds that $z_t \sim N(0,\sigma^2)$. Further, since $\v z_1,\v z_2, \dots$ are i.i.d., then $z_1,z_2,\dots$ are i.i.d. By expanding the recursive inequality in Section~\ref{eq:xt-recursive-norm}, we derive that
\begin{equation}\label{eq:lb-xt}
\|\x_{t}\| \ge (1-\eta-\eta c)^t\|\x_0\| + \sum_{i=0}^{t-1} (1 - \eta - \eta c)^i \eta (A-c) + \sum_{i=0}^{t-1} (1-\eta - \eta c)^i z_i.
\end{equation}
Let us lower bound the three terms above. The first term will be bounded by $0$. The second term equals
\begin{align*}
&\sum_{i=0}^{t-1} (1-\eta-\eta c)^i \eta (A-c)
= \frac{1-(1-\eta-\eta c)^t}{1-(1-\eta - \eta c)} \eta (A-c) \\
&\ge (1-(1-\eta - \eta c)^t)(A-c)
\ge (1-(1-\eta)^t)(A-c)
\ge (1 - e^{-\eta t})(A-c)
\ge 0.95(A-c) \ge 0.9A.
\end{align*}
Here, we used that $(1-\eta)^t \le e^{-\eta t} \le e^{-3} \le 0.05$, since $1-x\le e^{-x}$ for all $x \in \mathbb{R}$ and due to the assumption that $t\ge 3/\eta$, and further, we used the fact that $c>0$ can be chosen arbitrarily small to bound $0.95(A-c) \ge 0.9A$. It remains to bound the third term in the expansion of \eqref{eq:lb-xt}, which is a Gaussian random variable, with zero mean and its variance can be computed as:
\begin{align*}
&\Var\left(\sum_{i=0}^{t-1} (1-\eta-\eta c)^i z_t\right)
= \sum_{i=0}^{t-1} \Var((1-\eta - \eta c)^i z_t)
= \sum_{i=0}^{t-1} (1-\eta - \eta c)^{2i} \sigma^2\\
&\le \sum_{i=0}^{t-1} (1-\eta)^{2i} \sigma^2
\le \sum_{i=0}^{\infty} (1-\eta)^{2i} \sigma^2
= \frac{\sigma^2}{1-(1-\eta)^2}
= \frac{\sigma^2}{2\eta-\eta^2}
\le \frac{\sigma^2}{\eta} 
= \frac{2}{\beta},
\end{align*}
recalling that we defined $\beta = 2\eta/\sigma^2$. Denote by $z$ the random variable corresponding to the third term of \eqref{eq:lb-xt}, then we have just shown that $\Var(z) \le 2/\beta$ and that $\|\x_t\| \ge 0.9A - z$. In order to conclude the proof, it is sufficient to bound $\Pr[z>a]$ for any $a >0$. From standard concentration inequalities for Gaussians, we know that for any $a$,
\[
\Pr[z > a]
\le e^{-a^2/2\Var(z)}
\le e^{-a^2\beta/4},
\]
as required.

\subsection{The norm is bounded from above}\label{sec:norm-bounded}

Here we prove the following proposition:
\begin{proposition}\label{prop:above-bnd}
For any $t \ge 0$,
\[
\Pr[\|\x_t\|\ge (1-\eta/2)^t\|\x_0\|+C+C\sqrt{n/\beta}] \le e^{-n/C}
\]
for some universal $C>0$.
\end{proposition}

We write the gradient of the loss as
\[
\nabla \tilde{L}(\x) = \nabla L(\x) + (\nabla \tilde{L}(x) - \nabla L(\x))
= 
(\|\x\|-\cos \theta_d(\x)) \vec r(\x)
+ \sin \theta_d(\x) \theta'_d(\x) \vec \theta(\x) + (\nabla \tilde{L}(x) - \nabla L(\x))
\]
and the Langevin step is
\[
\x_t = \x_{t-1} - \eta \nabla \tilde L(\x) + \v z_t
\]
where $\v z_t \sim N(\v 0, \sigma^2 I_n)$. We have
\[
\v x_t
= (1-\eta) \x_{t-1} + \eta \v A_t + \eta \v B_t + \v z_t,
\]
where
\[
\v A_t := \cos \theta_d \vec{r} + \sin \theta_d \theta'_d \vec{\theta}, \quad 
\v B_t = \nabla \tilde{L}(x) - \nabla L(\x).
\]
Notice that 
\[\|\v A_t\|^2 
\le \cos^2 \theta_d + \sin^2\theta_d(\theta'_d)^2 \le \cos^2\theta_d + \sin^2 \theta_d \le 1\]
where we used that $\theta'_d \le 1$ (this follows from Lemma~\ref{lem:g}).
Further, by \eqref{eq:tilde-delta},
\[
\|\v B_t\| \le (\|\x_{t-1}\| + 1)f(\delta, d),
\]
where $f(\delta,d)\to 0$ as $\delta \to 0$. In particular, since we assume that $\delta$ can be chosen arbitrarily small, we can assume that $f(\delta,d)\le c$ for some arbitrarily small constant $c>0$.

Expanding on the definition of $\x_t$, we have
\[
\x_t = (1-\eta)^t \x_0 + \sum_{i=1}^t (1-\eta)^{t-i}\v z_i + \eta \sum_{i=1}^t (\v A_i + \v B_i) (1-\eta)^{t-i}.
\]
Decompose $\x_t = \v y_t + \v w_t$ as follows:
\[
\v y_t = \sum_{i=1}^t (1-\eta)^{t-i}\v z_i, \qquad
\v w_t = (1-\eta)^t \x_0 + \eta \sum_{i=1}^t (\v A_i + \v B_i) (1-\eta)^{t-i}.
\]
Notice that
\[
\v w_t = (1-\eta) \v w_{t-1} + \eta (\v A_t + \v B_t).
\]
We have that
\begin{align*}
&\|\v w_t\| \le 
(1-\eta) \|\v w_{t-1}\| + \eta (\|\v A_t\| + \|\v B_t\|)
\le (1-\eta) \|\v w_{t-1}\| + \eta + \eta c (\|\v x_{t-1}\| + 1) \\
&\le (1-\eta) \|\v w_{t-1}\| + \eta + \eta c (\|\v y_{t-1}\| + \|\v w_{t-1}\| + 1)
\le (1- \eta + \eta c) \|\v w_{t-1}\| + \eta c \|\v y_{t-1}\| + \eta(1+c),
\end{align*}
and $\|\v w_0\| = \|\x_0\|$. By expanding on this, we have that
\[
\|\v w_t\|
\le (1-\eta + \eta c)^t \|\x_0\| + \eta c \sum_{i=1}^{t-1} (1-\eta + \eta c)^{t-1-i} \|\v y_i\| + \eta(1+c) \sum_{i=1}^t (1-\eta+\eta c)^{t-i}.
\]
Hence,
\[
\|\x_t\| \le \|\v y_t\| + \|\v w_t\|
\le (1-\eta + \eta c)^t \|\x_0\| + \eta c \sum_{i=1}^{t} (1-\eta + \eta c)^{t-i} \|\v y_i\| + \|\v y_t\| + \eta(1+c) \sum_{i=1}^t (1-\eta+\eta c)^{t-i}.
\]
Assuming that $c \le 1/2$, we have
\[
\eta (1+c) \sum_{i=1}^t (1-\eta+\eta c)^{t-i}
\le 1.5\eta \sum_{i=0}^\infty (1-\eta/2)^i
= \frac{1.5\eta}{1-(1-\eta/2)} = 3.
\]
Assuming again that $c \le 1/2$, we have that
\[
\|\x_t \| \le 
(1-\eta/2)^t \|\x_0\|
+ \sum_{i=1}^{t} (1-\eta/2)^{t-i} \|\v y_i\| + \|\v y_t\| + 3.
\]
Let us bound the term that corresponds to the $\v y_i$, and notice that these are isotropic Gaussians with variance bounded as follows:
\begin{align*}
\Var(\v y_t)
= \Var\l(\sum_{i=1}^t (1-\eta)^{t-i}\v z_i\r)
= \sum_{i=1}^t \Var\l((1-\eta)^{t-i}\v z_i\r)
= \sum_{i=1}^t (1-\eta)^{2t-2i}\sigma^2 I
\preceq \sigma^2 \sum_{i=0}^\infty (1-\eta)^{2i} I\\
= \frac{\sigma^2}{1-(1-\eta)^2} I
= \frac{\sigma^2}{2\eta - \eta^2} I
\preceq \frac{\sigma^2I}{\eta} = 2\beta I,
\end{align*}
using the fact that $\eta \le 1$ and recalling the definition of $\beta$ from Section~\ref{sec:pr-langevin}. 
We will use the following definition of a sub-Gaussian random variable:
\begin{definition}
A random variable $X$ is $L$-subGaussian if $\E[\exp((X-\E x)/(2L))] \le 2$.
\end{definition}
We have the following properties of a subGaussian random variable \cite{vershynin2018high}:
\begin{lemma} \label{lem:subgConcentration}
Let $\v X \sim N(\v 0, \sigma^2 I)$. Then, $\|\v X\|$ is an $L$-subGaussian for some universal constant $L>0$.
\end{lemma}
\begin{lemma}
If $X$ is an $L$-subGaussian random variable than for any $t > 0$, 
\[
\Pr[X \ge t]
\le \exp(-t^2/2CL)
\]
for some universal constant $C>0$.
\end{lemma}
\begin{lemma}
If $X_1,\dots,X_n$ are $L$-subGaussian random variables, then $\sum_i \lambda_i X_i$ is $L\sqrt{\sum_i \lambda_i^2}$-subGaussian, hence it is $L\sum_i |\lambda_i|$-subGaussian.
\end{lemma}

Since $\v y_t$ is an isotropic random variable with variance bounded by $2\beta$, we derive that $\|\v y_t\|$ is $C\beta$ subGaussian for some $C>0$. Further, we derive that
\[
\eta \sum_{i=1}^t (1-\eta/2)^{t-i} \|\v y_i\| + \|\v y_t\|
\]
is a subGaussian with parameter bounded by
\[
C\beta\sum_{i=1}^t (1-\eta/2)^{t-i} + C\beta
\le C\beta\eta \sum_{i=1}^\infty (1-\eta/2)^{i} +C\beta
= \frac{C\beta\eta}{\eta/2} + C\beta
= 3C\beta.
\]
Further, let us compute:
\begin{align*}
\E\l[\eta \sum_{i=1}^t (1-\eta/2)^{t-i} \|\v y_i\|+\|\v y_t\|\r]
\le \eta \sum_{i=1}^t (1-\eta/2)^{t-i} \sqrt{\E[\|\v y_i\|^2]} + \sqrt{\E[\|\v y_t\|^2]} \le \eta \sum_{i=1}^t (1-\eta/2)^{t-i} \sqrt{2\beta n} + \sqrt{2\beta n}\\
\le \frac{\eta}{\eta/2} \sqrt{2 \beta n} + \sqrt{2 \beta n}
\le 3\sqrt{2 \beta n}.
\end{align*}
From Lemma~\ref{lem:subgConcentration} we derive that for any $h \ge 0$
\[
\Pr\l[\sum_{i=1}^t (1-\eta/2)^{t-i} \|\v y_i\|+\|\v y_t\| \ge 3\sqrt{2 \beta n} + h \r]
\le \exp(-h^2/C'\beta),
\]
for some universal constant $C'>0$.
This implies that
\[
\Pr\l[\|\v \x_t\| \ge 3 + (1-\eta/2)^t \|\x_0\| + 3 \sqrt{2\beta n} + h\r]
\le \exp(-h^2/(C'\beta)).
\]
In particular, if we substitute $h = \sqrt{n/\beta}$, we get that
\[
\Pr\l[\|\v \x_t\| \ge 3 + (1-\eta/2)^t \|\x_0\| + 3 \sqrt{2\beta n}\r] \le 
\exp(-n/C'')
\]
for some universal constant $C''>0$.

\subsection{Defining a smooth loss function} \label{sec:modified-loss}

One problem that arises with $L$ is that it is not smooth around the origin. As we have shown, $\x$ does not approach the origin with high probability. Hence, it is sufficient to assume that the loss function is different around the origin. In particular, the dynamics will not reach a ball of radius $r_0 := \cos(g^{\circ d}(\pi))/2$ around the origin, w.h.p. We define a modified loss, $\Lm$, that is different in this ball. 
First, we define an auxiliary function, that is parameterized by $0\le a<b$ and is a function of $r \ge 0$:
\[
h^{a,b}(r) = \begin{cases}
    0 & r \le a \\
    2(r-a)^2/(b-a)^2 & a \le r \le (a+b)/2 \\
    1 - 2(b-r)^2/(b-a)^2 & (a+b)/2 \le r \le b\\
    1 & r \ge b
\end{cases}\enspace.
\]
Notice that this function transitions smoothly from $0$ to $1$ in the interval $[a,b]$, it has a continuous first derivative and has a second derivative almost everywhere, with
\[
\frac{dh^{ab}(r)}{dr} = h^{a,b}_r(r) = \begin{cases}
    0 & r \le a \\
    4(r-a)/(b-a)^2 & a \le r \le (a+b)/2 \\
    4(b-r)/(b-a)^2 & (a+b)/2 \le r \le b\\
    0 & r \ge b
\end{cases}
\]
and
\[
\frac{d^2h^{ab}(r)}{dr^2}=
h^{a,b}_{rr}(r) = \begin{cases}
    0 & r < a \\
    4/(b-a)^2 & a < r < (a+b)/2 \\
    -4/(b-a)^2 & (a+b)/2 < r < b\\
    0 & r > b
\end{cases} \enspace.
\]
We will define the following smoothed loss function function:
\[
\Lm(r,\theta) = L(r,\theta) h^{r_0/3,2r_0/3}(r) + \xi(1- h^{0,r_0}(r)),
\]
for some parameter $\xi>0$ to be determined. Denote $h^{r_0/3,2r_0/3} = h^1$ and $h^{0,r_0}=h^2$ for convenience.
The derivatives of $\Lm$ can be computed as follows:
\begin{itemize}
    \item $\Lm_r = L_r h^1 + Lh^1_r - \xi h^2_r$.
    \item $\Lm_\theta = L_\theta h^1$
    \item $\Lm_{rr} = L_{rr}h^1 + 2 L_r h_r^1 + h^1_{rr} -\xi h^2_{rr}$
    \item $\Lm_{\theta\theta} = L_{\theta\theta}h^1$
    \item $\Lm_{r\theta} = L_{r\theta}h^1 + L_\theta h^1_r$
\end{itemize}

We conclude the following properties that the modified loss function satisfies everywhere, based on the properties computed above and the properties of $L$:
\begin{lemma} \label{lem:modified-L}
Assume that $\xi$ is a large enough universal constant. Then, the modified loss satisfies:
\begin{itemize}
    \item $\Lm$ is $O(1)$ smooth everywhere. Further, $\triangle \Lm \le O(n)$.
    \item The critical points of $\Lm$ (those with zero derivative) are $0,\x^*$ and $-\cos g^{\circ d}(\pi)\x^*$. For any $l>0$ there exists a constant $c(l)>0$ such that any point whose distance from the critical points is at least $l$ satisfies $\|\nabla \Lm(\x)\| \ge c(l)$.
    \item At $\x \in \{\x \colon \|\x\| \le r_0/3\}$, $\triangle \Lm(\x) \le -\Omega(n)$.
\end{itemize}
\end{lemma}
\begin{proof}
The smoothness in the ball $\{ \x \colon \|\x\| \ge 2r_0/3\}$ follows from Lemma~\ref{lem:smooth-lip}, which argues that $L$ is smooth in this region, due to the fact that $\Lm=L$ in this region. In the region $\{\x\colon\|\x\|\in\{r_0/3,2r_0/3\}\}$ smoothness of $\Lm$ follows from the expression for the second derivative of $\Lm$, from Lemma~\ref{lem:polar} and from the fact $L, h^1,h^2$ are smooth with bounded derivatives in this region. For $\{\x\colon\|\x\|\in\{0,r_0/3\}\}$ smoothness of $\Lm$ follows from the smoothness of $h^2$. Further, $|\triangle L| \le O(n)$ since any function $f$ on $n$ variables that is $s$-smooth satisfies $\triangle f \le sn$.

Next, we argue about the critical points of $L$. First, look at the region defined by $\|\x\|\ge r_0$, where, $\hat{L}=L$. In this region, the critical points of $L$ are  $\x^*$ and $-\cos (g^{\circ d}(\pi)) \x^*$ and these are also the critical points of $\hat{L}$ in this region. Next, we study the region $\|\x\| \in [r_0/3,r_0]$. In this region, recall that
\[
\Lm_r = L_r h^1 + Lh^1_r - \xi h^2_r.
\]
Now, the first two terms are bounded by a constant, using the calculations of the derivatives of $L$ and of $h^1$. And the last term (which is being subtracted from the first two terms) is lower bounded by a constant times $\xi$. We can make the whole derivative negative by taking $\xi$ to be a sufficiently large constant. In particular, in this region, the derivative with respect to $r$ is nonzero, hence, by Lemma~\ref{lem:polar}, the gradient of $\Lm$ is nonzero. Lastly, for the region $r \in [0,r_0/3]$: Here, $\Lm = \xi(1-h^2)$. By the derivative computation above, the only critical point is $\v 0$. In particular, this concludes that the critical points of $\Lm$ are $\x^*, - \cos(\theta^{\circ d}(\pi))\x^*$ and $\v 0$. Now, from continuity, for any $l>0$ there exists $c(l,n)>0$ such that any point $\x$ whose distance from any critical point is at least $l$, satisfies that its gradient norm is at least $c(l,n)$. Yet, notice that this constant can be taken independent of $n$. This is due to the fact that $\Lm(\x)$ is only a function of $r(\x)$ and $\theta(\x)$, hence $\|\nabla \Lm(\x)\|$ is as well, and there is no dimension dependence.

For the last item, notice that in the region $\|\x\|\le r_0/3$, $\Lm(\x) = \xi(1-h^2(\x))$. By the computation of the second derivative, and by Lemma~\ref{lem:polar}, it follows that $\triangle \Lm(\x) \le -\Omega(\xi n) \le -\Omega(n)$ as required.
\end{proof}

\subsection{Convergence assuming a potential function} \label{sec:potential-sufficient}

In this section, we want to argue that certain potential functions decrease as a consequence of applying a Langevin step. We will use this in the future to prove that the iterations converge to a certain region where this potential is small.

Assume that there is a potential function $V \colon \mathbb{R}^d \to \mathbb{R}$. Further assume the Laplacian is defined as
\[
\mathcal LV(\x) 
=\triangle V(\x) - \beta \langle \nabla H(\x), \nabla V(\x)\rangle,
\]
where $H(\x)$ is the function that defines the Langevin dynamics as in Section~\ref{sec:pr-langevin}. We would like to show that if $\mathcal{L}V(\x)$ is negative around $\x_{t-1}$ then $\E[V(\x_t)\mid \x_{t-1}] < V(\x_{t-1})$.

\begin{lemma}\label{lem:pot-one-step}
Assume that the functions $H,V$ are $O(1)$ smooth in $\mathbb{R}^n$, and that $\sigma^2 = 2\eta/\beta \le O(1/n)$.
Assume that $\|\nabla H(\x_{t-1})\|,\|\nabla V(\x_{t-1})\|\le O(1)$.
Let $-\kappa$ denote the maximum of $\mathcal LV$ in the ball of radius $r:=2\sqrt{n}\sigma= 2 \sqrt{2n\eta/\beta}$ around $\x_{t-1}$, and assume that $\mathcal{L}(\x)$ is bounded by $-\kappa + M$ in $\mathbb{R}^n$.
Then,
\[
\E[V(\x_t) - V(\x_{t-1})\mid \x_{t-1}]
\le -\mu\kappa/\beta + e^{-cn} M/\beta + O(\eta \sqrt{\eta n/\beta}), 
\]
where $c>0$ is a universal constant.
\end{lemma}
To prove the above, we use the following stochastic process:
\[
\v y_0 = \x_{t-1} ;
\quad
d\v y_s = - \eta \nabla H(\y_0)ds + \sqrt{2\eta/\beta} d\v B_s\enspace.
\]

Notice that $\v y_1 \sim \x_t$ conditioned on $\x_{t-1}$. Let us assume that $\x_{t-1}$ is fixed for the calculations ahead.
We would like to compute $\E V(\v y_1)$. To do this, we can use It's formula, to derive that
\[
\E[V(\v y_1) - V(\v y_0)]
= \int_{0}^1 \E[-\eta \langle \nabla H(\v y_0), \nabla V(\v y_s) \rangle + \triangle V(\v y_s) \eta/\beta] ds.
\]
For a fixed $s$, the term under expectation equals
\[
\frac{\eta}{\beta} 
\mathcal LV(\v y_s) + \eta \langle \nabla H(\v y_0) - \nabla H(\v y_s), \nabla V(\v y_s) \rangle.
\]
Let us bound both terms in expectation. For the first term, we use the fact that since $\v y_s \sim N(\v y_0, \sigma^2 s I_n)$, $\Pr[\|\v y_s - \v y_0\| \ge 2\sqrt{\sigma s n}] \le e^{-cn}$. If the above does not hold, we Laplacian of $\v y_s$ is at most $-\kappa$, and otherwise it is at most $-\kappa + M$, as assumed above. Hence,
\[
\E[\mathcal LV(\v y_s)] \le -\kappa(1-e^{-cn}) + e^{-cn} (-\kappa + M) \le -\kappa + Me^{-cn}.
\]
For the second term, we have, for some constant $C$,
\begin{align*}
&\E \langle \nabla H(\v y_0) - \nabla H(\v y_s), \nabla V(\v y_s) \rangle
\le \E \|\nabla H(\v y_0) - \nabla H(\v y_s)\| \|\nabla V(\v y_s)\|\\
&\le \E [\|\nabla H(\v y_0) - \nabla H(\v y_s)\| (\| \nabla V(\v y_s) - \nabla V(\v y_0)\| + \|\nabla V(\v y_0)\|)]
\le O(1) \E [\|\v y_0-\v y_s\| (\|\v y_s-\v y_0\| + O(1))]\\
&\le O(1)\E \|\v y_0-\v y_s\|^2 + O(1) \E \|\v y_0 - \v y_s\|
\le O(\sqrt{\sigma^2 s n}) = O(\sqrt{\eta n/\beta}).
\end{align*}

This completes the proof of the lemma above. As a consequence, we bound the number of times that it takes for the function to get to a region with positive value of $\mathcal LV$:
\begin{lemma}\label{lem:using-pot}
Let $V \ge 0$ be an $O(1)$-smooth potential function, assume that $H$ is $O(1)$ smooth, let $\kappa >0$, define
\[
K = \{ \x \colon \exists \y,\ \|\y-\x\|\le 2\sqrt{n}\sigma,\
\mathcal{L}V(\y) > -\kappa\}.
\]
Let $C_1 > 0$ and define
\[
B = \{ \x \colon \max(\|\nabla H(\x)\|,\|\nabla V(\x)\|)>C_1\}.
\]
Assume that $\x_t$ is according to the Langevin dynamics with potential function $H$ (see Section~\ref{sec:pr-langevin}). Let $\tau>0$ be the first $t$ such that $\x_t \in K \cup B$. Let $M$ denote the maximum of $\mathcal{L}V$ over all $\mathbb{R}^n$. If $ e^{-cn}M/\beta + O(\eta \sqrt{\eta n/\beta}) < \mu\kappa/2\beta$, then,
\[
\E[\tau\mid \x_0] \le \frac{V(\x_0)}{\mu\kappa/\beta -(e^{-cn}M/\beta + O(\eta \sqrt{\eta n/\beta})} \le 
\frac{2V(\x_0)}{\mu\kappa/\beta}.
\]
\end{lemma}
\begin{proof}
Denote $\Delta = \mu\kappa/\beta -(e^{-cn}M/\beta + O(\eta \sqrt{\eta n/\beta})$
If $\x_t\notin K\cup B$, we can apply Lemma~\ref{lem:pot-one-step} to argue that $\E[V(\x_{t+1})\mid \x_t] \le V(\x_t) - \Delta$. Since $V$ cannot decrease below $0$, the expected number of iterations that this happens is at most $V(\x_0)/\Delta$ as required.
\end{proof}

\subsection{Defining a potential function} \label{sec:def-potential}

We would like to apply Lemma~\ref{lem:using-pot} for the dynamics 
defined by the loss function $\Lm$. Notice that this function identifies with $L$ except for some ball around $0$. Define the potential function
\[
V(\x) = \Lm(\x) - \lambda \cos(\theta)  h^{r_0,3r_0/2}(r)\mathds 1(\theta \ge \pi/2).
\]

We prove the following:
\begin{lemma}\label{lem:potential-prop}
Assume that $n$ is at least a sufficiently large constant. There exists some $\lambda = \Theta(1)$ such that the following holds.
Let $l>0$ be a constant. Then, there exist a constants $C,c>0$ (depending possibly on $l$) such that for any $\beta \ge Cn$ and any $\x$ that satisfies $\|\x-\x^*\|\ge l$, we have that $\mathcal{L}V(\x)\le - cn$. Further, $\mathcal LV\le O(n)$ everywhere.
\end{lemma}
\begin{proof}
For convenience, denote $h = h^{r_0, 3r_0/2}$.
First of all, we explain how to set $\lambda$. For that purpose, recall that the Laplacian involves an inner product between the gradient of $\Lm$ and that of $V$, and we would like to make sure that this inner product is always non-positive (as it appears with a negative sign). Notice that
\[
\langle \nabla \Lm(\x), \nabla V(\x)\rangle 
= \|\nabla \Lm(\x)\|^2 + \lambda \langle \nabla \Lm(\x), \nabla -\cos \theta(\x)h^{r_0,3r_0/2}(r(\x)) \mathds{1}(\theta(\x)\ge \pi/2) \rangle.
\]
While the first term is always non-negative, we would like to make sure that the second term is not very negative. For that purpose, let us compute the gradient of the second term of the loss function, and notice that it is nonzero only if $\theta \ge \pi/2$ and $r \ge 3r_0/2$, and assume that we are in this region for convenience, and in particular, the indicator function $\mathds{1}(\theta \ge \pi/2)$ can be replaced with $1$. In order to compute the gradient, it is sufficient to compute the derivatives with respect to $r$ and $\theta$, as follows from Lemma~\ref{lem:polar}. We have the the derivative with respect to $r$ equals
\[
(-\cos(\theta) h(r))_r
= -\cos(\theta)  h_r(r),
\]
and 
\[
(-\cos(\theta) h(r))_\theta = \sin \theta h(r).
\]
Hence, the gradient equals
\[
\nabla (-\cos(\theta) h(r))
= -\cos(\theta)  h_r(r) \vec{r}
+\frac{\sin \theta h(r)}{r} \vtheta.
\]
The inner product with the gradient of $\Lm$ equals
\[
\langle (r -\cos \theta_d)\vr + \sin \theta_d \theta'_d \vtheta, -\cos(\theta)  h_r(r) \vec{r}
+\frac{\sin \theta h(r)}{r} \vtheta
\rangle
= (r -\cos \theta_d)(-\cos(\theta)  h_r(r)) + (\sin \theta_d \theta'_d)\frac{\sin \theta h(r)}{r},
\]
which follows since $\vr,\vtheta$ are orthonormal vectors, from Lemma~\ref{lem:polar}. The second term in the right hand side is nonnegative, since all the involved terms are positive, including $\theta'_d$ which is the derivative of $g^{\circ d}(\theta)$ that is nonnegative since $g$ is increasing. Hence, we only have to take care of the first part. Notice that $h_r(r) = 0$ for any $r \ge 3r_0/2$, hence, we have to care only for $r \in [r_0,3r_0/2]$. In this region, the norm of the gradient of $\Lm$ is at least some constant, since $\Lm=L$ in this region and due to Lemma~\ref{lem:grad-large}. Further, $h_r$ is always bounded from above in an absolute value, as is $\cos \theta$ and $r$ is bounded in this region, hence, we can set $\lambda$ to be a constant such that for all $r \in [r_0,3r_0/2]$,
\[
\lambda\l\langle (r -\cos \theta_d)\vr + \sin \theta_d \theta'_d \vtheta, -\cos(\theta)  h_r(r) \vec{r}
+\frac{\sin \theta h(r)}{r} \vtheta
\r\rangle
\ge (r-\cos \theta)(-\cos \theta h_r(r)) \ge -\|\nabla \Lm\|^2/2.
\]
This concludes that 
\begin{equation}\label{eq:inner-dominates}
\langle \nabla V(\x),  \nabla \Lm(\x)\rangle 
\ge \|\nabla \Lm\|^2/2,
\end{equation}
for all $\x$, for a sufficiently small (but constant) $\lambda$.

Now, it is sufficient to prove that $\triangle V \le -\Omega(n)$ in some balls of constant radius $l'$ around $\v 0$ and the saddle point $-\x^* \cos g^{\circ d}(\pi)$. Indeed, assume that this is the case and let us conclude the proof. First of all, in the above two neighborhoods, we have that
\[
\mathcal LV \le \triangle V - \frac{\beta}{2} \|\nabla \Lm\|
\le \triangle V \le -\Omega(n)
\]
using \eqref{eq:inner-dominates}. 
Next, we would analyze $\mathcal LV$ outside these regions and outside a ball of radius $l$ around $\x^*$. Using Lemma~\ref{lem:modified-L}, $\|\nabla \Lm\| \ge \Omega(1)$ in these regions and $\triangle V = \triangle \Lm - \triangle \cos \theta h(r) \le O(n)$, using  Lemma~\ref{lem:modified-L} to bound $\triangle \Lm$ and using the fact that $\cos \theta$ and $h(r)$ has bounded first and second derivatives to bound $\triangle \cos \theta h(r)$. In particular, using \eqref{eq:inner-dominates} we derive that
\[
\mathcal LV = \triangle V - \frac{\beta}{2} \|\nabla \Lm\|^2\le O(n) - \Omega(\beta) \le -\Omega(\beta)\le -\Omega(n),
\]
if $\beta \ge \Omega(n)$ for a sufficiently large constant. It remains to show that $\triangle V \le -\Omega(n)$ in two balls of radius $l'>0$ around $0$ and around the saddle point. First of all, around $0$ we have that $V = \Lm$, and using Lemma~\ref{lem:modified-L} we have that $\triangle V \le -\Omega(n)$. Secondly, let us look at the region around $-\x^* \cos g^{\circ d}(\pi)$. In that region $h = 1$, hence, using Lemma~\ref{lem:polar} we can compute that
\[
\triangle(-h(r)\cos(\theta)) = \triangle(-\cos\theta)
= \frac{(n-1)\cos\theta}{r^2} \le -\Omega(n),
\]
around $-\x^* \cos g^{\circ d}(\pi)$, since $\theta(-\x^* \cos g^{\circ d}(\pi)) = \theta(-\x^*)=\pi$ and $\cos\pi = -1$.
Next, from Lemma~\ref{lem:laplacian}, we have that around this point,
\[
\triangle \Lm = \triangle L 
\le 2 + \frac{(n-2)(r-\cos \theta_d)}{r}.
\]
At the saddle point this equals $0$. Yet, this can be positive if $r \ge \cos \theta$, however, it is bounded by $c(l')\cdot n$ in a ball of radius $l'>0$ around the saddle point. By continuity, we can take $c(l')$ to zero as $l'\to 0$. Hence, if $l'$ is taken as a sufficiently small constant, then  $|\triangle \Lm| \le |\triangle \lambda \cos \theta|/2$. In particular, we have that in a ball of radius $l'$ around the saddle point,
\[
\triangle V = \triangle L + \triangle \cos \theta
\le \frac{1}{2}\triangle \cos \theta\le -\Omega(n),
\]
as required.

Lastly, inside a ball of radius $l$ around $\x^*$, we have that
\[
\mathcal LV \le \triangle V = \triangle \Lm = \triangle L \le O(n),
\]
where we used the computed bound on $\triangle L$ and the fact that $V$ identifies with $L$ around $\x^*$.
\end{proof}

For conclusion, let us bound the time that it takes the algorithm to get into the convexity region:
\begin{lemma}\label{lem:getting-to-conv}
Let $l>0$ and assume that $n$ is a sufficiently large constant. Assume that we run the dynamics according to the loss function $L$ and let $\tau$ denote the first iteration such that $\|\x_t-\x^*\|\le l$. Then, $\E\tau \le O(1/\eta)$. Further, $\Pr[\tau \ge \Omega(\log(1/\epsilon)/\epsilon)] \le \epsilon + e^{-cn}$.
\end{lemma}
\begin{proof}

First, we will argue that if we run the dynamics according to $L$, the hitting hitting time is bounded by $O(1/\eta)$ in expectation. In order to prove that, we use the potential function $V$ in combination with Lemma~\ref{lem:using-pot}. Recall that the set $K$ defined in Lemma~\ref{lem:using-pot} corresponds to the set of all points $\x$ where $\mathcal L V$ is smaller than $-\kappa$, in a neighborhood of radius $2\sqrt{n}\sigma$ around $\x$. First of all, notice that $2\sqrt{n}\sigma = 2\sqrt{2n\eta/\beta} \le 2\sqrt{2n/\beta}$ and we will choose $\beta$ sufficiently large such that this is smaller than $l/2$. Further, from Lemma~\ref{lem:potential-prop}, we know that we can choose $\kappa = cn$ such that 
\[
\{ \x \colon \mathcal LV(\x) \ge -cn\} \subset \{\x \colon \|\x-\x^*\| \le l/2\}
\]
($c>0$ is a universal constant). This implies that the set $K$ from Lemma~\ref{lem:using-pot} is contained in a ball of radius $l$ around $\x^*$. By the same lemma, the hitting time to $K\cup B$, is bounded by
\begin{equation}\label{eq:bnd-hitting}
\frac{V(\x_0)}{\kappa\eta/\beta-(e^{-cn}M/\beta + O(\eta \sqrt{\eta n/\beta})},
\end{equation}
where $M$ is a total bound on $\mathcal L V$ and it is $O(n)$ using Lemma~\ref{lem:potential-prop}.
Since $\beta = \Theta(n) = \Theta(\kappa)$, we have that $\kappa\eta/\beta = \Theta(\eta)$. Further, $e^{-cn}M/\beta = O(e^{-cn}n/\beta) = O(e^{-cn})$, since $n = \Theta(\beta)$. Lastly, $\eta \sqrt{\eta n/\beta} = \Theta(\eta^{3/2})$ and this can be smaller than $\kappa\eta/\beta = \Theta(\eta)$ if $\eta$ is sufficiently small. Hence, the denominator in \eqref{eq:bnd-hitting} is $\Omega(\kappa\eta/\beta)\ge \Omega(1)$. Finally, we want to bound the numerator. We have that if $\x_0$ is bounded, then $\|V(\x_0)\|\le O(1)$. This derives that the number of iterations required to hit either $K$ or $B$ is bounded by $O(1/\eta)$. Recall the definition of $B$ from Lemma~\ref{lem:using-pot}, and notice that it contains only points where wither $\nabla \Lm$ or $\nabla V$ are larger than some constant $C_1>0$ that we can select. Hence, $B$ only contains points of large norm. In particular, the probability to hit $B$ is very small, from Proposition~\ref{prop:above-bnd}, hence, with high probability we first hit $K$. In particular, the hitting time is $O(1/\eta)$ with high probability.

Recall that this assumed that we run the dynamics according to $L'$, yet the lemma is about running it according to $L$. Yet using Lemma~\ref{lem:escape} we know that if we run the dynamics according to $L$, then with high probability, at iterations $t=3/\eta, \cdots, 3/\eta + O(1/\eta)$, the norm is at least $r_0=\cos g^{\circ d}(\pi)/2$. In this region, $L$ and $L'$ are the same and running the dynamics according to $L$ is the same as running according to $L'$. In particular, the expected time to hit $K$ after iteration $3/\eta$ is $O(1/\eta)$.

Notice that the above argument can fail with some probability, if at some point the norm of $\x_t$ is either very small or very large. Yet, if this holds, then by Lemma~\ref{lem:escape} and Proposition~\ref{prop:above-bnd}, after a small number of iterations the norm will be of the right order and again, one have a large chance of hitting $K$. Overall, a simple calculation shows that the expected number of iterations to hit $K$ is $O(1/\eta)$ as required.

Lastly, we prove the high probability bound on $\tau$. By iterating: fix $C>0$ some appropriate constant, then for any $t$,
\[
\Pr[\tau > C\log(1/\epsilon)/\eta]
= \prod_{i=1}^{\log (1/\epsilon)} \Pr[\tau > C i/\eta \mid \tau > C(i-1)/\eta] \le \prod_{i=1}^{\log (1/\epsilon)} 0.1 \le \epsilon,
\]
where we use Markov's inequality to bound $\Pr[\tau > C i/\eta \mid \tau > C(i-1)/\eta]$.\footnote{There is one detail that should be taken care of: conditioned on $\tau > C(i-1)/\eta$, $\|\x_t\|$ may be large. Yet, since $\|\x_t\|\le O(1)$ w.pr. $e^{-cn}$ and since we can assume that $\epsilon \ge e^{-cn}$ (as an error of $e^{-cn}$ is already present in the theorem statement), we will not encounter a large $\|\x_t\|$, even conditioned on a large $\tau$.}
\end{proof}

\subsection{Continuous gets closer to discrete}\label{sec:epsilon-dynamics}

Assume that a function $H$ is $M$-smooth and $\mu$-strongly convex. We want to compare the langevin iteration 
\[
\x_t = \x_{t-1} -\eta \nabla H(\x_{t-1}) + \v z_t,
\]
where $\v z\sim N(\v 0, \sigma^2 I)$, to the continuous iteration defined by
\[
d\y_t = -\eta \nabla H(\y_t)dt + \sigma dB_t.
\]
Note that $\x_t$ runs in discrete times $t=0,1,2,\dots$ while the continuous runs continuous time $t\ge 0$. We want to show the following:
\begin{lemma}\label{lem:cont-to-disc}
Assume that we run the discrete and continuous time dynamics, $\x_t$ and $\y_t$, with respect to some function $H$, that is $\Omega(1)$ smooth and $O(1)$ strongly convex, assume that for all $t$, $\E\|\nabla H(\x_t)\| \le O(1)$ and that $\|\x_0-\y_0\|\le O(1)$. Then, for $T \ge \Omega(\log(1/\eta)/\eta)$, one has a coupling between $\x_t$ and $\y_t$ such that
\[
\E\|\x_t - \y_t\|\le O(\sqrt\eta).
\]
\end{lemma}
\begin{proof}
To prove the lemma, let us first present $\x_t$ as continuous dynamics over $t \ge 0$:
\[
d\x_t = -\eta \nabla H(\x_{\lfloor t\rfloor})dt + \sigma dB_t,
\]
and note that the difference between $\x_t$ and $\y_t$ is that the gradient with respect to $\x_t$ is taken according to $\x_{\lfloor t\rfloor}$ and not to $\x_t$. This produces exactly the same distribution over $\x_0,\x_1,\dots$ as the discrete dynamics. Let us couple $\x_t$ with $\y_t$, while using the same Gaussian noise $dB_t$. Then, if we take $\x_t-\y_t$ the noise cancels, and we have
\[
d(\x_t-\y_t)
= -\eta (\nabla H(\x_{\lfloor t\rfloor}) - \nabla H(\y_t))dt\enspace.
\]
Applying Ito's lemma, and assuming that the function is $c$-strongly convex and $C$-smooth, one has
\begin{align*}
&d\|\x_t-\y_t\|^2/2 = \langle \x_t-\y_t, d(\x_t-\y_t)\rangle
= -\eta \l\langle \x_t-\y_t, \nabla H(\x_{\lfloor t\rfloor}) - \nabla H(\y_t)\r\rangle dt\\
&= -\eta\l\langle \x_t-\y_t, \nabla H(\x_t) - \nabla H(\y_t)\r\rangle dt
+ \eta \l\langle \x_t-\y_t, \nabla H(\x_t) - \nabla H(\x_{\lfloor t\rfloor})\r\rangle dt\\
& \le -\eta c (\|\x_t-\y_t\|^2
+ \eta \|\x_t-\y_t\|\|\nabla H(\x_t) - \nabla H(\x_{\lfloor t\rfloor})\|) dt\\
&\le -\eta c (\|\x_t-\y_t\|^2+\eta C\|\x_t-\y_t\| \|\x_t - \x_{\lfloor t\rfloor}\|) dt\enspace.
\end{align*}
Using Ito's formula again, one has
\[
d\|\x_t - \y_t\|
= d \sqrt{\|\x_t - \y_t\|^2}
= \frac{d \|\x_t - \y_t\|^2}{2 \|\x_t - \y_t\|}
\le (\eta c \|\x_t - \y_t\| + \eta C \|\x_t - \x_{\lfloor t \rfloor}\|)dt.
\]
Integrating, one has
\[
\|\x_T-\y_T\| = e^{-\eta c T} \|\x_0-\y_0\| + \int_{0}^T \eta C \|\x_t - \x_{\lfloor t \rfloor}\| e^{-\eta c (T-t)}dt\enspace.
\]
We would like to take an expectation, for that purpose, let us estimate $\|\x_t - \x_{\lfloor t \rfloor}\|$. 
Denote
$s = t - \lfloor t \rfloor$. Then, $\x_t - \x_{\lfloor t \rfloor} \sim N(\nabla H(\x_{\lfloor t \rfloor})s\eta, s\sigma^2I_n)$, in particular, \[\E[\|\x_t - \x_{\lfloor t \rfloor}\|^2\mid \x_{\lfloor t \rfloor}] = \|\nabla H(\x_{\lfloor t \rfloor})s\eta\|^2 + sn\sigma^2\]
which implies, by Jensen, that
\[
\E\l[\|\x_t - \x_{\lfloor t \rfloor}\|\mid \x_{\lfloor t \rfloor}\r] \le \|\nabla H(\x_{\lfloor t \rfloor})s\eta\| + \sqrt{ns}\sigma
\le \eta\|\nabla H(\x_{\lfloor t \rfloor})\| + \sigma\sqrt{n}.
\]
Taking an outer expectation and using the bound on the gradient and that $\beta = \Theta(n)$, one has that
\[
\E\l[\|\x_t - \x_{\lfloor t \rfloor}\|\r]
\le O(\eta + \sqrt{\eta n/\beta}) \le O(\sqrt\eta).
\]
Substituting this above, one has
\[
\E[\|\x_T-\y_T\|] \le 
e^{-\eta cT} \|\x_0-\y_0\| + O(1) \int_0^T \eta^{3/2} e^{-\eta c(T-t)} dt \le e^{-\eta cT}\|\x_0-\y_0\| + O(\sqrt{\eta}).
\]
The result follows by substituting $T \ge \Omega(\log(1/\eta)/\eta)$.
\end{proof}

\subsection{Staying in the convexity region}\label{sec:staying}

We use the following known property for gradient descent:
\begin{lemma}\label{lem:contraction}
Let $f \colon K \to \mathbb{R}$, where $K \subset \mathbb{R}^n$ is a convex set. Assume that $f$ is $s$-smooth and $\mu$-strongly convex, let $\eta \le 2/(s+\mu)$. Let $\x,\v y\in K$ 
then,
\[
\| \x - \eta \nabla f(\x) - (\v y - \eta \nabla f(\v y))\| \le \l(1-\frac{\eta s\mu}{s+\mu} \r)\|\x-\v y\|.
\]
In particular, if $K$ is a ball around the minima $\x^*$ of $f$, then
\[
\|\x - \eta \nabla f(\x) - \x^*\| \le \l(1-\frac{\eta s\mu}{s+\mu} \r)\|\x-\x^*\|.
\]
\end{lemma}

We would like to show that the Langevin rarely escapes some ball around $\x^*$. We have the following proposition:
\begin{lemma}\label{lem:staying-in-conv}
Let $f$ be function that has local minima at $\x^*$ and which is $\mu$-strongly convex and $s$-smooth in a ball of radius $R$ around $\x^*$. Assume that we run Langevin dynamics, starting at $\v x_0$, following
\[
\v x_{t} = \x_{t-1} - \eta \nabla f(x) + \v z_t,
\]
where $\v z_t \sim N(0,\sigma^2)$, and 
\[
\eta \le \frac{2}{s+\mu}.
\]
Further, assume the $\sigma \sqrt{n} \le R/4$ and that $\|\x_0-\x^*\|\le R/2$
Then, for any $T > 0$,
\[
\Pr[\forall i=1,\dots,T\colon \|\x_i - \x^*\| \le R]
\ge 1- e^{-cn},
\]
where $c>0$ is a small universal constant.
\end{lemma}

To prove this lemma, we would like to couple $\x_0,\dots,\x_T$ with auxiliary variables $\v y_0,\dots,\v y_T$ such that for all $t\le T$, if $\|\v y_0\|,\dots, \|\v y_{t-1}\|\le R$ then $\|\x_t\| \le \|\v y_t\|$. In particular, if $\|\v y_0\|,\dots,\|\v y_T\|\le R$ then $\|\v x_0\|,\dots,\|\x_T\|\le R$. It will be convenient to bound the $\v y_t$ variables.

Define $p = 1-\frac{\eta s\mu}{s+\mu}$, then we define
\[
\v y_0 = \x_0;\quad
\v y_t = p \v y_{t-1} + \v w_t,
\]
where $\v w_t \sim N(0,\sigma^2)$.
Let us show how to couple $\v x_t$ and $\v y_t$ and show by induction the required property. For $t=0$ this holds by definition. Assume that this holds for all $i < t$ and prove for $t$. Let us assume that $\|\v y_0\|,\dots,\|\v y_{t-1}\|\le R$ otherwise the proof follows. By assumption we have that $\|\v x_{t-1}\|\le \|\v y_t\|\le R$. Denote by $\v x_{t-1}' = \v x_{t-1} - \eta \nabla f(\x_{t-1})$. By Lemma~\ref{lem:contraction}, we have that
\[
\|\x_{t-1}'\| \le p \|\x_{t-1}\| \le p\|\v y_{t-1}\|.
\]
Let us now couple the noise $\v z_t$ added to $\x_{t-1}$ in the recursive formula, with the noise $\v w_t$ added to $\v y_{t-1}$. Denote $\tilde \x'_{t-1} = \frac{\x'_{t-1}}{\|\x'_{t-1}\|}$ and $\tilde\v y_{t-1} = \frac{\v y_{t-1}}{\|\v y_{t-1}\|}$, $z_t = \langle \v z_t, \tilde\x_{t-1}\rangle$, $w_t = \langle \v w_t, \tilde\v y_{t-1}\rangle$. Notice that $z_t, w_t \sim N(0,\sigma^2)$ and we have that:
\begin{equation}\label{eq:xt-staying}
\|\x_{t}\|^2 = \|\x'_{t-1} + \v z_t\|^2
= \|\x'_{t+1} + z_t \tilde \x_{t-1}' + (\v z_t - z_t \tilde \x'_{t-1})\|^2
= \|\x'_{t+1} + z_t \tilde \x'_{t-1}\|^2 +\|\v z_t - z_t \tilde \x'_{t-1})\|^2,
\end{equation}
where the last equality is due to the fact that $\x'_{t+1} + z_t \tilde \x_{t-1}$ is a multiple of $\|\x'_{t-1}\|$ while the second term is perpendicular to this vector. Similarly, we have that
\begin{equation}\label{eq:yt-staying}
\|\v y_t\|^2
= \|p \v y_{t-1} + \v w_t\|^2
= \|p \v y_{t-1} + w_t \tilde\v y_{t-1} + (\v w_t-w_t \tilde\v y_{t-1})\|^2
= \|p \v y_{t-1} + w_t \tilde\v y_{t-1}\|^2+ \|\v w_t-w_t \tilde\v y_{t-1}\|^2.
\end{equation}
We would couple $\v z_t$ and $\v w_t$ such that the first term in \eqref{eq:xt-staying} is bounded by the first term in \eqref{eq:yt-staying} and similarly for the second term. For the second term, notice that both $\v z_t - z_t \tilde \x'_{t-1}$ and $\v w_t - w_t \tilde \v y_{t-1}$ are Gaussian variables with Isotropic covariance $\sigma^2 I_{n-1}$ over a subspace of dimension $n-1$, so we can couple them such that their absolute value is identical. Now, we argue for the first terms. Analyzing the first term in \eqref{eq:xt-staying}, we have that,
\[
\x'_{t-1} + z_t \tilde \x'_{t-1} = (\|\x'_{t-1}\| + z_t)\tilde\x'_{t-1},
\]
hence, the corresponding first term equals
\[
\|\x'_{t-1} + z_t \tilde \x'_{t-1}\|^2 = (\|\x'_{t-1}\| + z_t)^2\|\tilde\x'_{t-1}\|^2
= (\|\x'_{t-1}\| + z_t)^2.
\]
For the term corresponding \eqref{eq:yt-staying}, we have that 
\[
p \v y_{t-1} + w_t \tilde \v y_{t-1}
= (\|p \v y_{t-1}\| + w_t)\tilde \v y_{y-1},
\]
hence, the corresponding first term equals
\[
\|p \v y_{t-1} + w_t \tilde \v y_{t-1}\|^2
= (\|p \v y_{t-1}\| + w_t)^2\|\tilde \v y_{y-1}\|^2
= (\|p \v y_{t-1}\| + w_t)^2.
\]
Hence, our goal is to couple $z_t$ with $w_t$ such that
\[
(\|\x'_{t-1}\| + z_t)^2
\le (\|p \v y_{t-1}\| + w_t)^2,
\]
or, equivalently,
\[
|\|\x'_{t-1}\| + z_t|
\le |\|p \v y_{t-1}\| + w_t|.
\]
We have already argued that $\|\x'_{t-1}\| \le \|p\v y_{t-1}\|$. Further, notice that $z_t,w_t \sim N(0,\sigma^2)$. So, it is sufficiently to use the following lemma:
\begin{lemma}
Let $a\ge b\ge 0$. Then, we can couple two random variables, $z,w \sim N(0,\sigma^2)$ such that $|a+z| \ge |b+w|$.
\end{lemma}
\begin{proof}
First of all, notice that two real-valued random variables, $X,Y$, can be coupled such that $X\ge Y$ whenever $\Pr[X\le t] \le \Pr[Y \le t]$ for any $t \in \mathbb{R}$. This is a standard argument, and the proof is by first $b \in [0,1]$ and then setting $X=x,Y=y$ for the values $x,y$ such that $\Pr[X\le x] = b$ and $\Pr[Y\le y] = b$. By the assumption that $\Pr[X\le t] \le \Pr[Y \le t]$ it holds that $x\ge y$ as required.

So, it suffices to show that $\Pr[|a+z|\le t] \le \Pr[|b+w| \le t]$, for any $t \ge 0$. Indeed, if $\phi$ is the density of a random variable $N(0,\sigma^2)$, we have that
\begin{align*}
&\Pr[|a+z|\le t] - \Pr[|b+w|\le t]\\
&= 1 - \Pr[|a+z|> t] - (1 - \Pr[|b+w|> t])\\
&= \Pr[|b+w|> t] - \Pr[|a+z|> t]\\
&= \Pr[b+w> t] + \Pr[b+w < -t] - \Pr[a+z> t] - \Pr[a+z < -t]\\
&= \Pr[w>t-b] + \Pr[w< -t-b] - \Pr[z>t-a] - \Pr[z<-t-a]\\
&= \Pr[w>t-b] + \Pr[w> t+b] - \Pr[z>t-a] - \Pr[z>t+a]\\
&= \int_{t-b}^\infty \phi(u)du + \int_{t+b}^\infty \phi(u) du - \int_{t-a}^\infty \phi(u)du - \int_{t+a}^\infty \phi(u)du\\
&= \int_{t+b}^{t+a} \phi(u)du - \int_{t-a}^{t-b}\phi(u) du\\
&= \int_{b+t}^{a+t} \phi(u)du - \int_{b-t}^{a-t}\phi(u) du\\
&= \int_{b}^a (\phi(u+t) - \phi(u-t)) du\\
&= \int_{b}^a (\phi(|u+t|) - \phi(|u-t|)) du
\end{align*}
using the fact that $\phi$ is symmetric around the origin. Recall that $a,b,t\ge 0$, hence $|u+t| \ge |u-t|$ which implies that $\phi(|u+t|)\le \phi(|u-t|)$ as $\phi$ is decreasing. We derive that the desired quantity is negative and this is what we wanted to prove.
\end{proof}

This concludes the inductive proof that if $\|\v y_1\|,\dots, \|\v y_T\| \le R$ then $\|\x_1\|,\dots,\|\x_T\|\le R$. It suffices to bound the probability that $\|\v y_1\|,\dots,\|\v y_T\|\le R$.
Notice that
\[
\v y_T =
p \v y_{T-1} + \v w_T
= p^T \v y_0 + \sum_{t=1}^T p^{T-t} \v w_t. 
\]
Notice that its mean is $p^T \v y_0$ and its covariance is
\[
\sum_{t=1}^T Cov(p^{T-t} \v w_t)
= \sum_{t=1}^T p^{2(T-t)} \sigma^2 I_n
= \sigma^2 I_n \sum_{i=0}^{T-1} p^{2i}
\le \sigma^2 I_n \sum_{i=0}^{\infty} p^{2i}
= \frac{\sigma^2}{1-p^2} I_n,
\]
where the inequality corresponds to the constant that multiplies the identity matrix. Using this inequality, we can derive that
\begin{align*}
\Pr[\forall i, \|\v y_i\| \le R]
\le \sum_{i=1}^T \Pr[\|\v y_i\| \le R]
\le \sum_{i=1}^T \Pr[\|\v y_i - \E \v y_i\| \le R - \|\E \v y_i\|]\\
\le \sum_{i=1}^T \Pr[\|\v y_i - \E \v y_i\| \le R - p^t\|\v y_0\|]
\le \sum_{i=1}^T \Pr[\|\v y_i - \E \v y_i\| \le R - \|\v y_0\|]\\
\le \sum_{i=1}^T \Pr[\|\v y_i - \E \v y_i\| \le R/2].
\end{align*}
We can use the fact that for a random variable $\v X \sim N(\v 0, \sigma^2 I_n)$, it holds that $\Pr[\|\v X\|>2 \sigma \sqrt{n}] \le e^{-c n}$ for some universal constant $c>0$. In particular, applying $\v X = \v y_i- \E \v y_i$, we derive that
\[
\Pr[\forall i, \|\v y_i\| \le R]
\le T e^{-cn}, 
\]
using that $R/2 \ge 2 \sigma \sqrt{n}$.
This concludes the proof.

\subsection{Culminating the Proof}

We start by arguing about the dynamics according to $L$ and then we argue for $\tilde L$.
First, assume that $t = \Theta(\log(1/\epsilon)/\epsilon^2)$. This assumption will be removed later.
Let $\tau$ be the minimal $t$ such that $\|\x_t - \x^*\|\le l$ for some appropriately chosen constant $l>0$.
First, notice that by Lemma~\ref{lem:getting-to-conv}, 
with probability $1-\epsilon$ we have $\tau \le O(\log(1/\epsilon)/\epsilon)$. By Lemma~\ref{lem:strongly-around-opt} there is some radius around $\x^*$ where the function is $\Omega(1)$-strongly convex, and assume that this radius is $2l$.
By Lemma~\ref{lem:staying-in-conv}, with high probability, the dynamics stay within the ball of radius $2l$ for additional $O(\log(1/\epsilon)/\epsilon^2)$ iterations.

In order to bound the Wasserstein distance between $\x_T$ and $\mu$,
we would like to couple the discrete dynamics $\x_t$ to the continuous dynamics, defined by
\[
d\y_t = -\eta \nabla L(\y_t)dt + \sqrt{2\eta/\beta} dB_t; \y_0 \sim \mu, \mu(\y) = \frac{e^{-\beta L(\y)}}{\int e^{-\beta L(\v z)}d\v z} ;
\]
Notice that $\y_t \sim \mu$ for all $t$. The coupling is done as follows: the chains are run independently until time $\tau$. Since $\y_\tau \sim \mu$ independently of $\x_\tau$, and since, assuming that $\beta \ge \Omega(n)$, $\mu$ has mass $1-e^{-cn}$ in the ball of radius $\ell$ around $\x^*$ (as can be computed using a simple integral), one has that with probability $1-e^{-cn}$, $\x_t$ and $\y_t$ are in this ball. From that point onward, using Lemma~\ref{lem:cont-to-disc} we can couple $\x_t$ and $\y_t$ such that after additional $O(\log(1/\epsilon)/\eta)$ iterations, $\E\|\x_t-\y_t\| \le O(\epsilon)$. By taking into account the failure probability to stay and remain in the convexity region, we derive that after $T = O(\log(1/\epsilon)/\eta)=O(\log(1/\epsilon)/\epsilon^2)$ iterations, $\E\|\x_t-\y_t\|\le O(\epsilon) + T 2^{-cn}$. We can assume that $\epsilon \ge e^{-cn/10}$ without hurting the guarantee, hence $T \le e^{cn/2}$, and the error is bounded by $e^{-cn/2}+O(\epsilon)$.

Next, we want to argue about the dynamics according to $\tilde L$ rather than $L$. Due to the added noise in each step, if the parameter $\delta$ in \eqref{eq:tilde-delta} is sufficiently small, then the KL divergence between the execution with $L$ vs. $\tilde L$ is small.

Lastly, we argue about what happens when $T > \Omega(\log(1/\epsilon)/\epsilon^2)$: in this case, we disregard the initial iterations, and restart the above argument, replacing $t=0$ with $t=T-\Theta(\log(1/\epsilon)/\epsilon^2)$ and replacing $T$ with $\Theta(\log(1/\epsilon)/\epsilon^2)$.

\section{Additional Experiments}
\begin{figure*}[!ht]
\captionsetup[subfigure]{labelformat=empty}
\captionsetup{justification=centering}
\begin{center}
\includegraphics[width=0.6\textwidth]{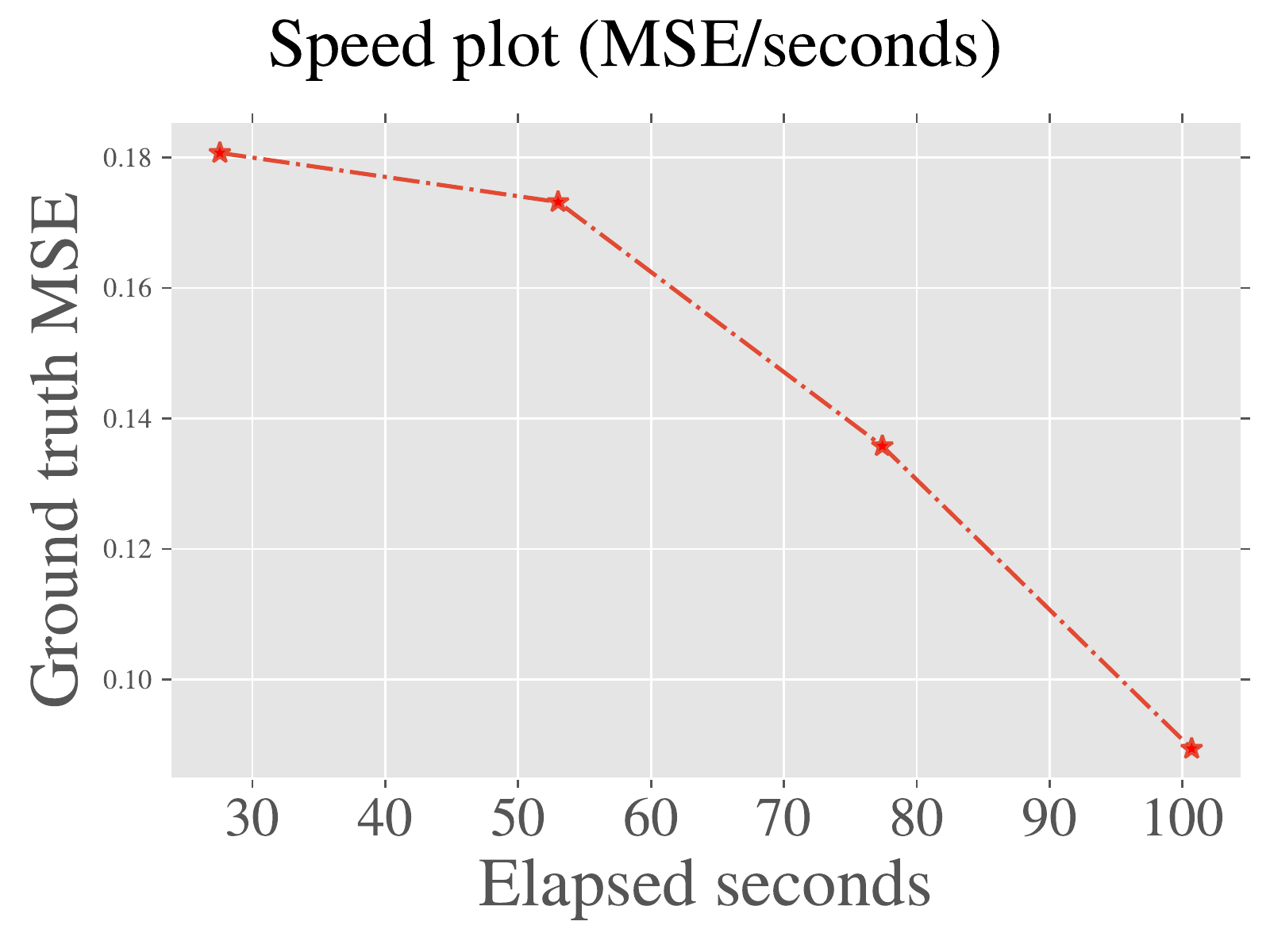}
\caption{Speed plot that demonstrates how the loss is changing over time. Each inversion takes about $1-2$ minutes on single V100 GPU.}
\label{fig:speed_plots}
\end{center}
\end{figure*}

\glsaddallunused
\end{document}